\documentclass{article}

    \usepackage[preprint,nonatbib]{neurips_2020}

\usepackage[utf8]{inputenc} 
\usepackage[T1]{fontenc}    
\usepackage{hyperref}       
\usepackage{url}            
\usepackage{booktabs}       
\usepackage{amsfonts}       
\usepackage{nicefrac}       
\usepackage{microtype}      

\usepackage{algorithm}
\usepackage{algorithmic}

\usepackage{graphicx}
\usepackage{subfigure}
\usepackage{booktabs} 
\usepackage{amsmath}
\usepackage{amsthm}
\usepackage{float}
\usepackage{bm}
\usepackage{amssymb}
\usepackage{bbm}
\usepackage{color}
\usepackage{mathtools}
\usepackage[super]{nth}
\usepackage{etoolbox}
\usepackage{adjustbox}
\usepackage{enumitem}

\DeclareMathOperator{\poly}{poly}

\DeclareMathOperator{\rank}{rank}

\DeclareMathOperator{\out}{out}

\DeclareMathOperator*{\argmax}{arg\,max}
\DeclareMathOperator*{\argmin}{arg\,min}

\renewcommand{\algorithmicrequire}{\textbf{Input:}}

\newcommand{\wh}{\widehat}
\newcommand{\wt}{\widetilde}

\newcommand{\eps}{\epsilon}
\newcommand{\N}{\mathcal{N}}
\newcommand{\R}{\mathbb{R}}
\newcommand{\bR}{\mathbf{R}}
\renewcommand{\P}{\mathbb{P}}

\renewcommand{\i}{\mathbf{i}}
\renewcommand{\varepsilon}{\epsilon}
\renewcommand{\tilde}{\wt}
\renewcommand{\hat}{\wh}
\renewcommand{\eps}{\epsilon}

\newtheorem{thm}{Theorem}
\newtheorem{asmp}{Assumption}

\newtheorem{theorem}{Theorem}[section]
\newtheorem{lemma}[theorem]{Lemma}
\newtheorem{definition}[theorem]{Definition}

\newtheorem{proposition}[theorem]{Proposition}
\newtheorem{corollary}[theorem]{Corollary}

\newtheorem{fact}[theorem]{Fact}
\newtheorem{remark}[theorem]{Remark}

\DeclarePairedDelimiterX{\bracket}[3]{#1}{#2}{#3}
\newcommand{\round}[1]{\bracket*{(}{)}{#1}}
\newcommand{\curly}[1]{\bracket*{\lbrace}{\rbrace}{#1}}
\newcommand{\squarebrack}[1]{\bracket*{\lbrack}{\rbrack}{#1}}
\makeatletter
\newcommand{\vast}{\bBigg@{4}}
\newcommand{\Vast}{\bBigg@{5}}
\makeatother

\newcommand{\sumi}[2]{\sum\limits_{i=#1}^{#2}}
\newcommand{\sumj}[2]{\sum\limits_{j=#1}^{#2}}

\DeclarePairedDelimiterXPP{\nrm}[2]{}{\lVert}{\rVert}{_{#1}}{\ifblank{#2}{\:\cdot\:}{#2}}
\newcommand{\norm}[2]{\nrm*{#1}{#2}}

\newcommand{\enorm}[1]{\norm{2}{#1}}
\newcommand{\esqnorm}[1]{\enorm{#1}^2}

\newcommand{\normF}[1]{\norm{\mathrm{F}}{#1}}

\newcommand{\abs}[1]{\bracket*{\lvert}{\rvert}{#1}}
\newcommand{\inner}[1]{\bracket*{\langle}{\rangle}{#1}}
\newcommand{\ceil}[1]{\bracket*{\lceil}{\rceil}{#1}}

\newcommand{\zero}{\mathbf{0}}
\newcommand{\one}{\mathbf{1}}

\newcommand{\cvec}{\mathbf{c}}

\newcommand{\e}{\mathbf{e}}

\newcommand{\m}{\mathbf{m}}

\newcommand{\p}{\mathbf{p}}

\newcommand{\s}{\mathbf{s}}

\newcommand{\uvec}{\mathbf{u}}
\newcommand{\vvec}{\mathbf{v}}
\newcommand{\w}{\mathbf{w}}
\newcommand{\x}{\mathbf{x}}

\newcommand{\y}{\mathbf{y}}
\newcommand{\z}{\mathbf{z}}

\newcommand{\A}{\mathbf{A}}
\newcommand{\B}{\mathbf{B}}
\newcommand{\D}{\mathbf{D}}

\newcommand{\I}{\mathbf{I}}

\newcommand{\M}{\mathbf{M}}

\newcommand{\Pmat}{\mathbf{P}}

\newcommand{\U}{\mathbf{U}}
\newcommand{\V}{\mathbf{V}}
\newcommand{\W}{\mathbf{W}}
\newcommand{\X}{\mathbf{X}}

\newcommand{\Z}{\mathbf{Z}}

\newcommand{\SIGMA}{\mathbf{\Sigma}}

\newcommand{\Ccal}{\mathcal{C}}
\newcommand{\Dcal}{\mathcal{D}}
\newcommand{\Ecal}{\mathcal{E}}

\newcommand{\Gcal}{\mathcal{G}}
\newcommand{\Hcal}{\mathcal{H}}
\newcommand{\Ical}{\mathcal{I}}

\newcommand{\Ncal}{\mathcal{N}}
\newcommand{\Ocal}{\mathcal{O}}
\newcommand{\Pcal}{\mathcal{P}}

\newcommand{\Scal}{\mathcal{S}}
\newcommand{\Tcal}{\mathcal{T}}
\newcommand{\Ucal}{\mathcal{U}}

\newcommand{\muvec}{\boldsymbol{\mu}}

\newcommand{\BigO}[1]{\mathcal{O}\round{#1}}

\providecommand\given{}
\newcommand\SetSymbol[1][]{%
	\nonscript\:#1\vert
	\allowbreak
	\nonscript\:
	\mathopen{}}
\DeclarePairedDelimiterX\Set[1]\{\}{%
	\renewcommand\given{\SetSymbol[\delimsize]}
	#1
}
\newcommand{\set}[1]{\Set*{#1}}

\newcommand{\Rd}[1]{\mathbb{R}^{#1}}
\newcommand{\Natural}{\mathbb{N}}

\DeclareMathOperator*{\myE}{\mathbb{E}}
\newcommand{\E}[1]{\myE\squarebrack{#1}}

\newcommand{\Exp}[2]{\myE_{#1}\squarebrack{#2}}
\newcommand{\pseudoExp}[2]{\tilde{\myE}_{#1}\squarebrack{#2}}
\newcommand{\Prob}[1]{\mathbb{P}\squarebrack{#1}}
\newcommand{\Probability}[2]{\mathbb{P}_{#1}\squarebrack{#2}}
\newcommand{\Var}[1]{\mathrm{Var}\squarebrack{#1}}

\newcommand{\inv}[1]{\frac{1}{#1}}
\newcommand{\indicator}[2]{\mathbbm{1}_{#1}\set{#2}}
\newcommand{\Tr}[1]{\mathop{\mathrm{Tr}}\squarebrack{#1}}

\definecolor{b2}{RGB}{51,153,255}
\definecolor{mygreen}{RGB}{80,180,0}
\newcommand{\Weihao}[1]{{\color{cyan}[Weihao: #1]}}
\newcommand{\Raghav}[1]{{\color{blue}[Raghav: #1]}}

\title{Robust Meta-learning for Mixed Linear Regression with Small Batches}

\author{
Weihao Kong\thanks{\texttt{kweihao@gmail.com}. University of Washington}
\And
Raghav Somani\thanks{\texttt{raghavs@cs.washington.edu}. University of Washington}
\And
Sham Kakade\thanks{\texttt{sham@cs.washington.edu}.  University of Washington}
\And
Sewoong Oh\thanks{\texttt{sewoong@cs.washington.edu}. University of Washington}
}

\begin{document}

\maketitle

\begin{abstract} 
A common challenge faced in practical supervised learning, such as medical image processing and robotic interactions, is that there are plenty of tasks but each task cannot afford to collect enough labeled examples to be learned in isolation. 
However, by exploiting the  similarities across those tasks, one can hope to overcome such data scarcity.  
  Under a canonical scenario where each task is drawn from a mixture of $k$ linear regressions, we study a fundamental question: can abundant small-data tasks compensate for the lack of big-data tasks? 
Existing second moment based approaches of \cite{2020arXiv200208936K} 
 show that such a trade-off is efficiently achievable, 
 with the help of medium-sized tasks with $\Omega(k^{1/2})$ examples each.
However, this algorithm is brittle in two important  scenarios. 
The predictions can be arbitrarily bad 
$(i)$ even with only a few outliers in the dataset; or $(ii)$ 
even if the medium-sized tasks are slightly smaller with $o(k^{1/2})$ examples each.  
We introduce a  spectral approach that is simultaneously robust under both scenarios. 
To this end, we first design a novel outlier-robust principal component analysis algorithm that achieves an optimal  accuracy.
This is followed by a 
sum-of-squares algorithm to exploit the information from  higher order moments.
Together, this approach is robust against outliers and achieves a graceful statistical trade-off; 
the lack of $\Omega(k^{1/2})$-size tasks can be compensated for with smaller tasks, which can now be as small as ${\cal O}(\log k)$. 

\end{abstract}

\section{Introduction}
\label{sec:intro} 




%


Modern machine learning tasks and corresponding training datasets exhibit a long-tailed behavior \cite{wang2017learning}, where a large number of tasks do not have enough training examples to be trained to the desired accuracy.  
Collecting high-quality labeled data can be  time consuming or require expertise.  Consequently, in domains such as annotating medical images or processing  robotic interactions, 
there might be a large number of  related but distinct tasks, but each task is associated with only a small batch of training data. However, one can hope to meta-train  across those tasks, exploiting their similarities, and collaboratively achieve accuracy far greater than what can be achieved for each task in isolation \cite{FAL17,RL16,KZS15,OLL18,triantafillou2019meta,rusu2018meta}. 
This is the goal of meta-learning \cite{Sch87,TP12}. 

Meta-learning is especially challenging under two practically important settings: $(i)$ 
a few-shot learning scenario where each task is associated with an extremely small dataset; and $(ii)$ an adversarial scenario  where a fraction of those datasets are corrupted. 
We design a novel meta-learning approach that is  robust to such data scarcity and adversarial corruption,  
under a canonical scenario where the tasks are linear regressions in $d$-dimensions and 
the model parameters are drawn from a discrete distribution of a support size $k$.

First, consider a case where we have an {\em uncorrupted} dataset from a collection of $n$ tasks, each with $t$ training examples.    
If each task has a large enough training data with  $t=\Omega(d)$ examples, it can be accurately learned in isolation. 
This is illustrated by solid circles in  Fig.~\ref{fig:longtail}. 
On the opposite extreme, where each task has only a {\em single} example (i.e.~$t=1$), significant efforts have been made to make training statistically efficient \cite{CL13, yi2016solving, sedghi2016provable,zhong2016mixed, li2018learning, chen2019learning, shen2019iterative}. However, even the best known result of \cite{chen2019learning} still requires  exponentially many such tasks: $n=\Omega(de^{\sqrt{k}})$ (details in related work in \S\ref{sec:related}). 
This is illustrated by solid squares. 
Perhaps surprisingly,  this can be significantly reduced to  quasi-polynomial $n=\Omega(k^{\Theta(\log k)})$ sample complexity and quasi-polynomial run-time, 
with a slightly larger dataset that is only logarithmic in the problem parameters. 
This result is summarized in the following, with the algorithm and proof presented in \S\ref{sec:proof_corollary} of the supplementary material.

\begin{corollary}[of our  results with no corruption, informal]
\label{coro:first}
Given a collection of $n$ tasks each associated with $t = \tilde\Omega(1)$ labeled examples,  if the  effective sample size  
$nt = \tilde\Omega(dk^2+k^{\Theta(\log k)})$, then Algorithm~\ref{alg:sketch-nocorruptions}  estimates the meta-parameters up to any desired accuracy of ${\cal O}(1)$ with high probability in time $poly(d, k^{(\log k)^2})$, under certain assumptions on the meta-parameters. 
\end{corollary} 

This is a special case of  a more general class of  algorithms we design, tailored for the following practical scenario; 
 the collection of tasks in hand are heterogeneous, each with varying sizes of datasets (illustrated by the blue bar graphs below). 
Inspired by the seminal work of  \cite{vempala2004spectral}, 
we exploit such heterogeneity by 
separating the roles of  
{\em light} tasks that have smaller datasets and {\em heavy} tasks that have  larger datasets. As we will show, the size of the heavy tasks determines the order of the higher order moments we can reliably exploit. Concretely, we first  use  light tasks to {\em estimate a subspace}, and then {\em cluster} heavy tasks in that projected space. 
The first such attempt  was taken in \cite{2020arXiv200208936K}, where a linkage-based clustering was proposed.  
However, as this clustering method relies on the second moment statistics, 
it strictly requires heavy tasks with $\Omega(k^{1/2})$ examples (left panel). In the absence of such heavy tasks, the abundant light tasks are wasted as no existing algorithm can harness their structural similarities. Such second moment barriers are common in even  simpler problems, e.g.~\cite{kong2018estimating, kong2019sublinear}.

\begin{figure}[H]
    \centering
    \includegraphics[width=0.33\textwidth,page=1]{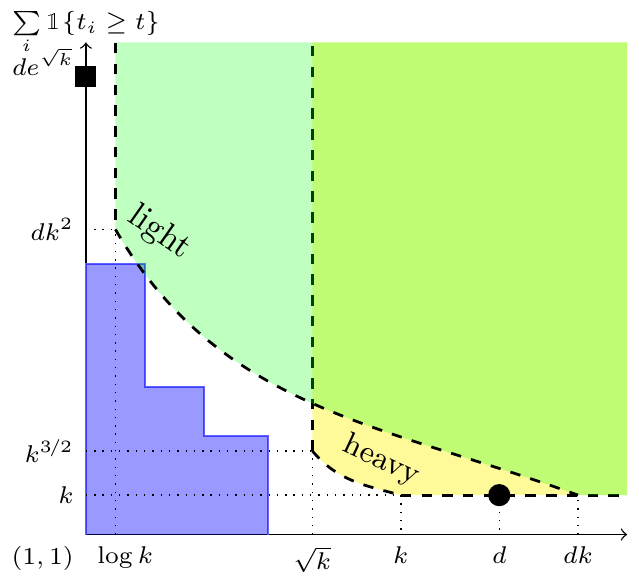}
    \put(-116,122){sufficient condition of \cite{2020arXiv200208936K}}
    \put(-81,-6){batch size $t$}
    \put(-155,20){
    \rotatebox{90}{number of tasks with}}
    \put(-145,34){
    \rotatebox{90}{batch size $\geq t$}} 
    \hspace{0.4cm}
    \includegraphics[width=0.33\textwidth,page=2]{images/figure_robust_sos.pdf}
    \put(-126,122){sufficient condition of Corollary~\ref{coro:small-t}}
    \put(-81,-6){batch size $t$}
    \caption{
    The blue bar graph summarizes the collection of tasks in hand, showing the cumulative count of tasks with more than $t$ examples.
     Typically, this does not include extremely large data tasks (circle) and extremely large number of small data tasks (square), where classical approaches succeed.   
      When any point in the light (green) region and any point in the heavy (yellow) region are both covered by the blue graph, the corresponding algorithm succeeds.  
      On the left, the collection in blue cannot be learned by any existing methods including   \cite{2020arXiv200208936K}.  
     Our approach in Corollary~\ref{coro:small-t} significantly extends the heavy region all the way down to $\log k$, leading to a successful meta-learning in this example. }  
    \label{fig:longtail}
    \vspace{-0.3cm}
\end{figure}
We  exploit higher order statistics to break this barrier, 
using computationally tractable tools from 
sum-of-squares methods \cite{2017arXiv171107465K}. 
This gives a class of algorithms parameterized by an integer $m$ (for $m$-th order moment) to be chosen by the analyst tailored to the size of the heavy tasks in hand $t_H=\Omega(k^{1/m})$. 
This allows for a graceful trade-off between $t_H$ and with  the required number of heavy tasks $n_H$. 
We summarize the result below, with a proof in \S\ref{sec:proof_corollary}, and  illustrate  it in Fig.~\ref{fig:longtail} (right). 
The choice of $m=\Theta(\log  k)$ gives the minimum required batch size, as we highlighted in  Corollary~\ref{coro:first}. 
\begin{corollary}[of our results with no corruption, informal]
    \label{coro:small-t}
    For any integer $m$, given two collections of  tasks, first collection  of light tasks with 
$t_L = \tilde\Omega(1)$, $t_Ln_L = \tilde\Omega(dk^2)$, 
and the second collection of heavy tasks with
$t_H = \tilde\Omega(m\,k^{1/m})$, $t_Hn_H = \tilde\Omega(k^{{\Theta}(m)})$,   
the guarantees of Corollary~\ref{coro:first} hold. 
\end{corollary}









Next, consider an adversarial scenario. Outliers are common in meta-learning as diverse sources contribute to the collection. 
Existing approaches are brittle to a few such outliers. 
A  fundamental question of interest is, what can be meta-learned from past experience that is only partially trusted? 
Following robust learning literature \cite{2017arXiv171107465K,diakonikolas2018list}, we assume a general adversary who can adaptively corrupt any $\alpha$ fraction of the tasks, formally defined in Assumption~\ref{asmp:adv}.
We make both subspace estimation and clustering steps robust against adversarial corruption. 
The sum-of-squares approach is inherently robust, when used within an iterative clustering \cite{2017arXiv171107465K}. 
However, existing robust subspace estimation  approaches are suboptimal,  requiring $\tilde\Ocal(d^2)$ samples \cite{diakonikolas2017beingrobust}.  
 To this end, we introduce a novel algorithm, and prove its optimality in both accuracy and dependence in the dimension $d$.   
 This resolves an open question posed  in~\cite{shen2019iterative} on whether it is possible to robustly learn the subspace with $\tilde\Ocal(d)$ samples.

 This achieves a similar sample complexity as the uncorrupted case in Corollary~\ref{coro:small-t}, while tolerating as much corruption as information theoretically possible: $\alpha = \tilde{\mathcal O}(\eps/k)$. 
Such condition is  necessary as otherwise the adversary can focus its attack on one of the mixtures, and incur $\Omega(\eps)$ error in estimating the parameter of that component. 

\begin{corollary}[of Theorem~\ref{thm:main},  informal]
    \label{coro:adv}
For any  $\eps\in(0,{1}/{k^3})$ and  $m\in\Natural$, given two collections of tasks, the first  with $t_L = \tilde\Omega(1),n_Lt_L = \tilde\Omega\round{{dk}{\eps^{-2}}}$, 
and the second  with
$t_H = \tilde\Omega\big(mk^{1/m}\big),n_Ht_H = \tilde\Omega\big(k^{\BigO{m}}\big)$, if the fraction of corrupted tasks is $\alpha =  \tilde{\mathcal O}(\eps/k)$, 
Algorithm~\ref{alg:sketch} achieves up to $\eps$ accuracy with high probability in time $poly(d, k^{m^2}, \eps^{-1})$, under certain assumptions.
\end{corollary} 

 We  provide the algorithm (Algorithm~\ref{alg:sketch}) and the analysis (Theorem~\ref{thm:main}) under the adversarial scenario in the main text. When there is no corruption, 
  the algorithm can be made  statistically more efficient with tighter guarantees, which is provided in 
 \S\ref{sec:proof_corollary}.
\subsection{Problem formulation and notations}
\label{sec:problem}

We present the probabilistic perspective on 
few-shot supervised learning following  \cite{grant2018recasting}, but focusing on a simple but canonical case where the tasks are linear regressions. 
A collection of $n$ tasks are independently drawn  according to some prior distribution.  
The $i$-th task is associated with a {\em model parameter} $\phi_i=(\beta_i \in {\mathbb R}^d,\sigma_i\in{\mathbb R}_+)$, and a {\em meta-train dataset} $\{(\x_{i,j},y_{i,j})\in{\mathbb R}^d\times {\mathbb R} \}_{j=1}^{t_i}$ of size $t_i$.  
Each example  
$(\x_{i,j},y_{i,j})\sim \P_{\phi_i}(y|\x) \P(\x)$ is independently drawn from a linear model, such that 
\begin{eqnarray}
     y_{i,j} \;\;=\;\; \beta_i^\top\,\x_{i,j} + \epsilon_{i,j}\;, 
     \label{eq:model}
\end{eqnarray} 
where $\x_{i,j}\sim {\cal N}(\zero,{\mathbf I}_d)$ and $\epsilon_{i,j}\sim  {\cal N}(0,\sigma_i^2)$. If $\x_{i,j}$  is from $\N(\zero,\SIGMA)$,
we assume to have 
enough $\x_{i,j}$'s (not necessarily labeled) for whitening, and $\P(\x)$ can be made  sufficiently close to isotropic. 

The goal of meta-learning is to train a model for a new arriving task $\phi_{\rm new}\sim \P_\theta(\phi)$ from a small size {\em training dataset} $\Dcal =\{(\x^{\rm new}_j,y^{\rm new}_j)\sim \P_{\phi_{\rm new}}(y|\x)\P(\x) \}_{j=1}^\tau$ of size $\tau$.  
This is achieved by 
exploiting some structural similarities to the meta-train dataset, drawn from the same 
prior  distribution $\P_\theta(\phi)$.  
To capture such structural similarities, 
we make a mild assumption  that $\P_\theta(\phi)$ is a finite discrete distribution of a  support size $k$. 
This is also known as mixture of linear experts \cite{CL13}. 
Concretely, $\P_\theta(\phi)$ is fully defined by a {\em meta-parameter} $\theta=(\W\in\R^{d\times k}, \s\in\R^{k}_+, \p\in\R^k_+)$ 
with $k$ candidate model parameters 
$\W=[\w_1,\ldots,\w_k]$ and $k$ candidate noise parameters $\s=[s_1,\ldots,s_k]$. 
The $i$-th task is drawn from  $\phi_i \sim \P_\theta(\phi)$, where first a 
$z_i\sim{\rm multinomial}(\p)$ selects a component that the task belongs to, and training data is independently drawn from 
Eq.~\eqref{eq:model} with $\beta_i=\w_{z_i}$ and $\sigma_i=s_{z_i}$.

Following the definition of  \cite{grant2018recasting},  
the {\em meta-learning problem} refers to solving the following:
\begin{eqnarray}
    \theta^* \;\;\in\;\; \argmax_\theta\; \log\,\P(\theta|\Dcal_{\rm meta-train})\;, 
    \label{eq:metalearning}
\end{eqnarray}
which estimates the 
most likely meta-parameter given 
{\em meta-training dataset} defined as 
${\cal D}_{\rm meta-train} \coloneqq \{\{(\x_{i,j},y_{i,j})\in{\mathbb R}^d\times {\mathbb R} \}_{j=1}^{t_i}\}_{i=1}^n$.  
This is a special case of empirical Bayes methods \cite{carlin2010bayes}. 
Our goal is solve this meta-learning problem   {\em robustly} against an   adversarial corruption  of $\Dcal_{\rm meta-train}$ as formally defined in Assumption~\ref{asmp:adv}. 
Once meta-learning is solved, the model parameter of the newly arriving task can be estimated with a Maximum a Posteriori (MAP) or a Bayes optimal estimator: 
\begin{eqnarray}
    \hat\phi_{\rm MAP} \;\in\; 
    \argmax_\phi\, \log\, \P(\phi|\Dcal, \theta^*) \;,\;\;\text{ and }\;\;
    \hat\phi_{\rm Bayes}\;\in\; \argmin_\phi {\mathbb E}_{\phi'\sim \P(\phi|\Dcal,\theta^*)}[\ell(\phi,\phi')]\;,
    \label{eq:MAPBayes}
\end{eqnarray}
for some choice of a loss $\ell(\cdot)$, which is straightforward.  
This is subsequently used to predict the label of a new data point $\x$ from $\phi_{\rm new}$. 
Concretely, $\hat{y} \in\argmax_y \P_{\hat\phi_{\text{MAP / Bayes}}} (y|\x) $.


\noindent
{\bf Notations.} 
We define $\squarebrack{n}\coloneqq\set{1,\ldots,n},\ \forall\ n\in\Natural$; $\norm{p}{\x} \coloneqq (\sum_{x\in\x}\abs{x}^p)^{1/p}$ as the standard vector $\ell_p$-norm; $\norm{*}{\A} \coloneqq \sum_{i=1}^{\min\set{n,m}}\sigma_i(\A)$, $\normF{\A} \coloneqq (\sum_{i,j=1}^{n,m}A_{i,j}^2)^{1/2}$ as the standard nuclear norm and Frobenius norm of matrix $\A\in\Rd{m\times n}$, where $\sigma_i(\A)$ denotes the $i$-th singular value of $\A$; $\Ncal\round{\muvec,\SIGMA}$ is the multivariate normal distribution with mean $\muvec$ and covariance $\SIGMA$; $\indicator{}{E}$ is  the indicator of an event $E$. 
We define $\rho_i^2 \coloneqq s_{z_i}^2+\esqnorm{\w_{z_i}}$ as the variance of a label $y_{i,j}$ in the $i$-th task, and $\rho^2 \coloneqq \max_{i} \rho_i^2$. We define $p_{\rm min}\coloneqq\min_{j\in[k]}p_j$, and $\Delta\coloneqq \min_{i,j\in\squarebrack{k}, i\neq j}\enorm{\w_i-\w_j}$ and assume $p_{\rm min}, \Delta>0$.

\subsection{Algorithm and intuitions}
\label{sec:spectral}

Following the recipe of spectral algorithms for clustering~\cite{vempala2004spectral} and few-shot learning~\cite{2020arXiv200208936K}, 
we propose the following approach consisting of three steps. 
Clustering step requires {\em heavy tasks}; each task has many labeled examples, but we need a smaller number of such tasks. 
Subspace estimation and  classification steps require {\em light tasks}; each task has a few  labeled examples, but we need a larger number of such tasks. Here, we  provide the intuition  behind each step and the corresponding requirements. 
The details are deferred to \S\ref{sec:detail},  where we emphasize  robustness to corruption of the data.  
The estimated  $\hat\theta=\big(\hat\W,\hat\s,\hat\p\big)$ is subsequently used in prediction, when a new task arrives. 

\begin{algorithm}[ht]
   \caption{}
   \label{alg:sketch}
\medskip
{\bf Meta-learning}
\vspace{-0.1cm}
\begin{enumerate}
    \item {\em Subspace estimation:}~Compute subspace $\hat\U$ which approximates ${\rm span}\curly{\w_1,\ldots, \w_k}$. 
    \vspace{-0.2cm}
    \item {\em Clustering:}~Project the heavy tasks onto the subspace of $\hat\U$, perform $k$ clustering, and estimate $\tilde{\w}_\ell$ for each cluster $\ell\in\squarebrack{k}$.
    \vspace{-0.1cm}
    \item {\em Classification:}~Perform likelihood-based classification of the light tasks using $\{\tilde{\w}_\ell\}_{\ell=1}^k$ estimated from the \textit{Clustering} step; compute  refined estimates $\{\hat{\w}_\ell, \hat{s}_\ell, \hat{p}_\ell\}_{\ell=1}^k$ of $\theta$.
    \vspace{-0.1cm}
\end{enumerate}

{\bf Prediction}
\vspace{-0.1cm}
\begin{enumerate}
    \item[4.] {\em Prediction:}~Perform MAP or Bayes optimal prediction using the estimated meta-parameter.
\end{enumerate}
\end{algorithm}

\noindent 
{\bf Subspace estimation.} 
As 
$\SIGMA\coloneqq{\mathbb E}[y_{i,j}^2\x_{i,j}\x_{i,j}^\top]=c\I+2\sum_{\ell=1}^k p_\ell \w_\ell\w_\ell^\top$ for some constant $c\geq 0$, 
the subspace spanned by the regression vectors, $\mathop{\mathrm{span}}\{\w_1,\ldots,\w_k\}$,
can be efficiently estimated by  Principal Component Analysis (PCA), if we have uncorrupted data. This only requires $\tilde\Omega(d)$ samples. 
With $\alpha$-fraction of the tasks adversarially corrupted, existing approaches of outlier-robust PCA attempt to  simultaneously estimate the principal subspace while {\em filtering out} the outliers \cite{xu2012outlier}. 
This removes many uncorrupted data points, and hence can either only tolerate  up to 
$\alpha=\Ocal(1/k^8)$ fraction of corruption (assuming well-separated  $\w_\ell$'s).   
We introduce a new approach in Algorithm~\ref{alg:filtering} that uses a second filter to recover those erroneously removed data points. 
This improves the tolerance to $\alpha=\Ocal(1/k^4)$ while requiring only $\tilde\Omega(d)$ samples (see  Remark~\ref{rem:projection}). We call this step {\em robust subspace estimation} (Algorithm~\ref{alg:filtering} in \S\ref{sec:subspace}).

\noindent 
{\bf Clustering.} 
Once we have the subspace, 
we  project the estimates of $\beta_i$'s to the $k$-dimensional subspace and cluster those points to find the centers. 
As $k\ll d$ in typical settings, this significantly reduces the sample complexity from $\poly(d)$ to $\poly(k)$.   
Existing meta-learning algorithm of
 \cite{2020arXiv200208936K} proposed a linkage based clustering algorithm. 
This utilizes the bounded property of the second moment only. 
Hence, strictly requires {\em heavy tasks} with $t=\Omega(k^{1/2})$. 
We break this second moment barrier by exploiting the boundedness of higher order moments. The heavy tasks are now allowed to be  much smaller, but at the  cost of requiring a larger number of such tasks and additional computations. 

One challenge is that the (empirical) higher order moments are tensors, and tensor norms are 
not efficiently computable. 
Hence boundedness alone does not give an efficient clustering algorithm. 
We need a stronger condition that 
 the moments are Sum-of-Squares (SOS) bounded, i.e.~there \textit{exist} SOS proofs showing that the moments are bounded~\cite{2017arXiv171107465K, hopkins2018mixture}. 
 This SOS boundedness is now tractable with a convex program, leading to a polynomial time algorithm that is also robust against outliers \cite{2017arXiv171107465K}. 
One caveat is that existing method  in~\cite{2017arXiv171107465K} requires data generated from a Poincar\'e distribution. 
As shown in Remark~\ref{rem:poincare}, the distribution of our estimate $\hat\beta_i=(1/t)\sum_{j=1}^t y_{i,j}\x_{i,j}$ is not Poincar\'e. 
Interestingly, as we prove in Lemma~\ref{lem:emp-sos-proof}, the higher order moments are still SOS  bounded. 
This ensures that we can apply the robust clustering algorithm of \cite{2017arXiv171107465K}. 
We call this step {\em robust clustering} (Algorithm~\ref{alg:recluster} in \S\ref{sec:cluster_proof}).



\noindent 
{\bf Classification and parameter estimation.} 
Given rough estimates  $\tilde\w_\ell$'s as center of those clusters, 
we grow each cluster by classifying remaining light tasks. Classification only requires  $t=\Omega(\log k)$. Once we have sufficiently grown each cluster, we can estimate the parameters to a desired level of accuracy. There are two reasons we need this refinement step. First, in the small corruption regime, where the fraction of corrupted tasks $\alpha$ is much smaller than the desired level of accuracy $\epsilon$, this separation is significantly more sample efficient. The subspace estimation and clustering steps require only $\Ocal(\Delta/\rho)$ accuracy, and the burden of matching the desired $\eps$ level of error is left to the final classification step, which is more sample efficient. 
Next, the classification step ensures an adaptive guarantee.
As parameter estimation is done for each cluster separately, a cluster with small noise $s_i$ can be more accurately  estimated. This  ensures  a more accurate prediction for newly arriving tasks. We call this step {\em  classification and robust parameter estimation} (Algorithm~\ref{alg:classification} in \S\ref{sec:proof-classification}).

\section{Main results}
\label{sec:adv}

 To give a more fine grained analysis, we assume 
 there are two types of light tasks. 
  In meta-learning, subspace estimation uses $\Dcal_{L1}$, clustering uses $\Dcal_{H}$, and classification uses $\Dcal_{L2}$.
  
\begin{asmp}
    \label{asmp}
    The heavy dataset ${\cal D}_{H}$ consists of 
    $n_H$ heavy tasks, each with at least $t_H$ samples. 
    The first light dataset ${\cal D}_{L1}$
    consists of $n_{L1}$ light tasks, 
    each with at least $t_{L1}$ samples. 
    The second light dataset ${\cal D}_
    {L2}$
    consists of 
    $n_{L2}$ tasks,
    each with at least $t_{L2}$ samples. We assume $t_{L1},t_{L2}<d$. 
\end{asmp}

The three batches of meta-train datasets are corrupted by an adversary.  

\begin{asmp} 
\label{asmp:adv}
    From the the datasets ${\cal D}_H$, ${\cal D}_{L1}$, and ${\cal D}_{L2}$, 
    the adversary controls  
    $\alpha_H$, $\alpha_{L1}$, and $\alpha_{L2}$ fractions of the tasks, respectively. 
    In a dataset 
    $\Dcal_H=\{\{ (\x_{i,j},y_{i,j}) \}_{j=1}^{t_H}\}_{i=1}^{n_H}$
    consisting of  $n_H$ tasks and the meta-training data associated with each of those tasks, 
    the adversary is allowed to inspect all the examples, remove those examples associated with a subset of tasks (of size at most $\alpha_H n_H$ tasks), and replace the examples associated with those tasks  with arbitrary points. 
    The corrupted meta-train dataset is then presented to the algorithm. $\Dcal_{L1}$ and $\Dcal_{L2}$ are corrupted  analogously.  
\end{asmp}

\subsection{Meta-learning and prediction}

We characterize the achievable accuracy in estimating the meta-parameters $\theta=(\W,\s,\p)$. 
\begin{thm} 
    \label{thm:main}
For any  $\delta \in (0,1/2)$ and $\epsilon>0$, given three batches of samples under Assumptions~\ref{asmp} and~\ref{asmp:adv}, the meta-learning step of Algorithm~\ref{alg:sketch}
achieves the following accuracy 
for all $i\in\squarebrack{k}$, 
\begin{align*}
\enorm{\hat{\w}_i-\w_i} \le \eps s_i\; ,\;\;  \;
\abs{\hat{s}_i^2-s_i^2}\le {\eps s_i^2}/{\sqrt{t_{L2}}}\; ,\quad\text{ and }\; \;\;
\abs{\hat{p}_i-p_i} \le \eps \sqrt{{t_{L2}}/{d}}\;p_i\;+\; \alpha_{L2}\;, 
\end{align*}
 with probability $1-\delta$, if the numbers of tasks, samples in each task, and the corruption levels satisfy 

\begin{adjustbox}{max width=\textwidth}{
\begin{tabular}[!htbp]{lll}
    $n_{L1} = \tilde\Omega\round{\frac{dk^2}{\tilde\alpha t_{L1}}+\frac{k }{\tilde\alpha^2}}$\;, & $t_{L1} \geq 1 $\;, & $\alpha_{L1} = \BigO{\tilde\alpha}$\;, \\
    $n_{H} = \tilde\Omega\round{
    \frac{(km)^{\Theta(m)}}{p_{\rm min}}+\frac{\rho^4}{\Delta^4 p_{\rm min}t_H}}$\;, & $t_H=\Omega\round{\frac{\rho^2}{\Delta^2}\cdot\frac{m}{p_{\rm min}^{2/m}}}$\;, & $\alpha_H = \tilde\Ocal\round{p_{\rm min}\min\set{1,\sqrt{t_H}\cdot\frac{\Delta^2}{\rho^2}}}$\;,\\
    $n_{L2}= \tilde\Omega\round{\frac{d}{t_{L2} p_{\rm min} \eps^2}}$\;, & $t_{L2} = \Omega\round{\frac{\rho^4}{\Delta^4}\log\frac{k n_{L2}}{\delta}}$\;, & $\alpha_{L2}= \Ocal({{p_{\rm min} \eps}/{\log(1/\eps)}})$\;,
\end{tabular}}
\end{adjustbox}
where $\tilde\alpha\coloneqq\max\set{  \Delta^2\sigma_{\rm min}^2/(\rho^6 k^2), \Delta^6 p_{\rm min}^2 /(k^2\rho^6)}$, $\sigma_{\rm min}$ is the smallest non-zero singular value of $\sum_{j=1}^k p_j \w_j\w_j^\top$, and $m\in\Natural$ is a parameter chosen by the analyst. 
\end{thm}

We refer to \S\ref{sec:problem} for the setup and notations, and  provide key lemmas in  \S\ref{sec:detail} and 
a complete proof in \S\ref{sec:metalearning_proof}. 
We 
discuss each of the conditions in the following remarks 
 assuming $\Delta = \Omega(\rho)$, for simplicity. 

\begin{remark}[Separating two types of light tasks] 
    \label{rem:sep} 
    \textcolor{black}{As $t_{L1}$ can be as small as one, the conditions on $\Dcal_{L2}$ does not cover the conditions for $\Dcal_{L1}$.     The conditions on $\alpha_{L1}$ and $n_{L1}$ can be significantly more strict than what is required for $\Dcal_{L2}$. Hence, we separate the analysis for $\Dcal_{L1}$ and $\Dcal_{L2}$.}
\end{remark} 

\begin{remark}[Dependency in $\Dcal_{L1}$]
Since we are interested the large $d$ small $t_{L1}$ setting, the dominant  term in $n_{L1}$ is ${{dk^2}/{\tilde\alpha t_{L1}}}$. 
The effective sample size $n_{L1}t_{L1}$ scaling as $d$ is information theoretically necessary. 
The $\min\{1/\sigma_{\rm min}^2,1/p_{\rm min}^2\}$
dependence of $n_{L1}t_{L1}$
allows sample efficiency even when $\sigma_{\rm mi n}$ is arbitrarily small, including zero. 
This is a significant improvement over the $\poly(1/\sigma_k)$ sample complexity of typical spectral methods, e.g.~\cite{CL13,yi2016solving}, 
where $\sigma_k$ is the $k$-th singular value of $\sum_{\ell=1}^k p_{\ell} \w_\ell \w_\ell^\top $.
This critically relies on  
an extension of the gap-free spectral bound of \cite{AL16,li2018learning}. 
Our tolerance of $\alpha_{L1}=\Ocal(p_{\rm min}^2/k^2)$  significantly improves upon the state-of-the-art guarantee of $p^4_{\rm min}/k^4$ as detailed in \S\ref{sec:subspace}. Further, we show it is information theoretically optimal. This assumes only  bounded fourth moment, which  makes our analysis more generally applicable. However,  this can be tightened under a stricter conditions of the distribution, as we discuss in \S\ref{sec:discussion}. 


\end{remark}

\begin{remark}[Dependency in $\Dcal_{H}$]
Assuming  $p_{\rm min} = \Omega(1/k)$,  the dominant term in $n_H$ is $\tilde\Omega({(km)^{\Theta(m)}}/{p_{\rm min}})$, which is  $\tilde\Omega({k^{\Theta(m)}})$ and the result is trivial when $m\ge \log(k)$. This implies a $n_H = \tilde\Omega(k^{\Theta(m)})$, $t_H = \Omega(m \cdot k^{2/m})$ trade-off for any integer $m$, breaking the $t_H=\Omega(k^{1/2})$ barrier of~\cite{2020arXiv200208936K}. In fact, for an optimal choice of 
$m = \Theta(\log k)$ to minimize the required examples, it can tolerate as small as  $t_H=\Omega(\log k)$ examples, at the cost of requiring 
$n_H = \tilde\Omega(k^{\Theta(\log k)})$ such heavy tasks. 
We conjecture $t_H = \Omega(\log k)$ is also necessary for any polynomial sample complexity. 
For the case of learning mixtures of isotropic Gaussians,  \cite{regev2017learning} shows that super-polynomially many number of  samples are information theoretically necessary when the centers are $o(\sqrt{\log k})$ apart. This translates to $t = o(\log k)$ in our setting. 
The requirement $\alpha_H=\Ocal(p_{\rm min})$ is  optimal. Otherwise, the adversary can remove an entire cluster. 
\end{remark}

\begin{remark}[Dependency in $\Dcal_{L2}$]
The requirement $n_{L2}\cdot t_{L2} = \tilde\Omega(d/p_{\rm min}\eps^2)$ is optimal in $d, p_{\rm min}$ and $\eps$ due to the lower bound for linear regression. The requirement on $\alpha_{L2}= \Ocal({{p_{\rm min} \eps}}/\log(1/\eps))$ is also necessary upto a log factor, from lower bound on robust linear regression~\cite{diakonikolas2019efficient}.
\end{remark}



At test time, we use the estimated  $\hat\theta=(\hat\W,\hat\s,\hat\p)$ to approximate the prior distribution on a new task.
On a new arriving task with training data $\Dcal=\{(\x_{j}^{\rm new},y_{j}^{\rm new})\}_{j=1}^\tau$,  
we propose the standard MAP or Bayes optimal estimators to make predictions on this new task. 
The following guarantee is a corollary of  Theorem~\ref{thm:main} and \cite[Theorem 2]{2020arXiv200208936K}. 
The term $\sum_{i\in[k]}p_i s_i^2$ is due to the noise in the test data 
$(\x,y)$ and cannot be avoided. We can get arbitrarily close to this fundamental limit with only $\tau=\Omega(\log k)$ samples. 
This is a minimax optimal sample complexity as shown in   \cite{2020arXiv200208936K}. 

\begin{corollary}[Prediction]
    \label{coro:pred}
Under the hypotheses of Theorem~\ref{thm:main}, 
the expected prediction errors of both the MAP  and Bayes optimal estimators $\hat\beta({\cal D})$ defined in Eq.~\eqref{eq:MAPBayes} are bound as
$
{\mathbb E}[(\x^\top \hat{\beta}({\cal D})-y)^2]
\;\le\;  \delta+\round{1+\eps^2}\sum_{i=1}^kp_is_i^2$,
if $\tau=  \Omega((\rho^4/\Delta^4)\log (k/\delta))$ and 
$\eps\le \min\{\Delta/(10\rho), \Delta^2\sqrt{d}/(50\rho^2)\}$, 
where 
the expectation is over the new task with model parameter $\phi^{\rm new}=(\beta^{\rm new},\sigma^{\rm new}) \sim {\mathbb P}_{\theta} $, training data $(\x_j^{\rm new},y_j^{\rm new})\sim {\mathbb P}_{\phi^{\rm new}}$, 
and test data $(\x,y)\sim {\mathbb P}_{\phi^{\rm new}}$.
\end{corollary}

\subsection{Novel robust subspace estimation} 
\label{sec:subspace}

Our main result relies on making {\em each step} of  Algorithm~\ref{alg:sketch} robust, as detailed in  \S\ref{sec:detail}. 
However, as our key innovation is a novel {\em robust subspace estimation} in the first step,   we highlight it in this section.

\begin{algorithm}[ht]
   \caption{Robust subspace estimation}
   \label{alg:filtering}
\begin{algorithmic}[1]
  \STATE {\bfseries Input:} Data $\Dcal_{L1} = \{ \{(\x_{i,j},y_{i,j}) \}_{j=1}^{t_{L1}} \}_{i=1}^{n_{L1}}$, $\alpha\in(0,1/36\rbrack$, $\delta\in(0,0.5)$, $k\in\Natural$, and $\nu\in\R_+$
   \STATE $\hat\beta_{i,j} \leftarrow y_{i,j}\x_{i,j} \;,\;\;\;\;\text{ for all }i\in[n_{L1}], j\in[t_{L1}]$\label{alg-line:debbetahat}
   \STATE $S_0 \leftarrow \{\hat\beta_{i,j}\hat\beta_{i,j}^\top\}_{i\in[n_{L1}],j\in[t_{L1}]}\;,\;$ 
   and $\;S_{\rm max}\leftarrow\emptyset$
   \STATE{\bf for}  $\ell=1,\ldots,\log_{6}{(2/\delta)}$ {\bf do}
   \STATE \hspace{0.5cm} $t\leftarrow 0$ and $S_{-1}\leftarrow \emptyset$ 
   \STATE \hspace{0.5cm} {\bf while} {$t \leq \ceil{9\alpha n}$ and $S_{t} \neq S_{t-1}$} {\bf do}
           \STATE \hspace{1.cm} $t\leftarrow t+1$, and 
        $\;\;S_t \gets $ Double-Filtering$( S_{t-1},k,\alpha,\nu)$ \hfill [See Algorithm~\ref{alg:basic_filter}]
    \STATE \hspace{0.5cm} {\bf if } $|S_{\rm max}|<|S_t|$ {\bf then} $S_{\rm max} \leftarrow S_{t}$ 
    \STATE {\bfseries Output:} $\hat\U\leftarrow k$\_SVD$\big(\, \sum_{\hat\beta_{i,j}\in S_{\rm max}}\hat\beta_{i,j}\hat\beta_{i,j}^\top\big)$
\end{algorithmic}
\end{algorithm}

We aim to estimate the subspace 
 spanned by the true meta-parameters $\{\w_1,\ldots,\w_k\}$. 
As  
$\SIGMA\coloneqq{\mathbb E}[\hat\beta_{i,j}\hat\beta_{i,j}^\top] = \{\sum_{\ell=1}^k  p_\ell(s_\ell^2 + \|\w_\ell \|_2^2)\} \I + 2 \sum_{\ell=1}^k p_\ell \w_\ell \w_\ell^\top$
for $\hat\beta_{i,j}$ in Algorithm~\ref{alg:filtering} line~\ref{alg-line:debbetahat}, 
we can use the $k$ empirical principal components; 
this requires  uncorrupted data.  
To remove the 
corrupted datapoints, we introduce {\em double filtering}. 
We repeat $\log_6(2/\delta)$ times for a high probability result.  


\begin{algorithm}[ht]
   \caption{Double-Filtering
} 
   \label{alg:basic_filter}
\begin{algorithmic}[1]
    \STATE {\bfseries Input:} a set of PSD matrices $S=\set{ \X_{i}\in\R^{d\times d } }_{i\in[n] }$, $k\in\Natural$, $\alpha\in(0,1/36\rbrack$ and $\nu\in\R_+$
    \STATE $\Scal_0\leftarrow [n]$, $\U_0 \leftarrow  k$\_SVD$\big(\,  \sum_{i\in \Scal_0}\X_i\big)$, and $z_i \gets \Tr{\U_0^\top\X_i\U_0}$ for all $i\in \Scal_0$
    \STATE $\Scal_G \gets $ First-Filter$(\set{z_i }_{i\in\Scal_0},\alpha)$ \hfill[Remove the upper and lower $2\alpha$ quantiles] 
    \STATE $\mu^{\Scal_0} \leftarrow (1/n) \sum_{i \in \Scal_0 } z_i \;\;$
    and $\;\;\mu^{\Scal_G} \leftarrow  (1/|\Scal_G|) \sum_{i\in \Scal_G} z_i $
    \STATE {\bf if} {$\mu^{\Scal_0}-\mu^{\Scal_G}\le 
    48(\alpha\mu^{\Scal_G}+\nu\sqrt{k\alpha})$} {\bf then} 
    {\bfseries Output:} $S$ \hfill[Sample mean not large, no need to filter.]
    \STATE {\bf else} \hfill[Run  a second filter if sample mean is corrupted]
        \STATE \hspace{0.5cm} $Z\sim \Ucal\squarebrack{0,1}$, 
        $\;\;\;W\gets Z\max\set{z_i-\mu^{\Scal_G}}_{i\in \Scal_0\setminus \Scal_G }$
        \STATE \hspace{0.5cm} $\Scal_1 \gets \Scal_G\cup\set{  i \in \Scal_0\setminus \Scal_G \given \; {{z_i-\mu^{\Scal_G}}}\leq W}$  \hfill[Add some removed points back.]
        \STATE {\bfseries Output:} $S'=\{\X_i\}_{i\in \Scal_1}$ 
\end{algorithmic}
\end{algorithm}

If the adversarial examples have the outer product $\X_i=\hat\beta_{i',j'}\hat\beta^\top_{i',j'}$'s with small norms, then 
they are challenging to detect. 
However, such  undetectable 
corruptions can only perturb the subspace by little.  
Hence, Algorithm~\ref{alg:basic_filter} focuses on detecting large corruptions. 
Ideally, we want to find a subspace by 
\[ 
\hat\U \; \leftarrow\; \argmax\limits_{\U\in\R^{d\times k}:\U^\top \U=\I_k} \; \underset{\Scal'\subseteq [n]:|\Scal'|\geq(1-\alpha)n}{\rm minimize}\;\; \sum_{i\in \Scal'} \underbrace{{\rm Tr}[ \U^\top \X_i \U]}_{\coloneqq z_i} \;,
\]
for $n=n_{L1}t_{L1}$, which is computationally intractable.  
This relies on the intuition that a good subspace  preserves the second moment, even  when large (potentially corrupted) points are removed. 

%
We propose a {\em filtering} approach in Algorithm~\ref{alg:basic_filter}. 
At each iteration,  we alternate between finding a candidate subspace $\U_0$ and then 
filtering out 
suspected corrupted data points, which have large trace norms in $\U_0$.  
Existing filtering approaches (e.g.~\cite{xu2012outlier}) 
use a {\em single filter} to 
 remove examples with 
 large trace norm (denoted by $z_i$ in Algorithm~\ref{alg:basic_filter}). 
This suffers from removing too many {\em uncorrupted} examples. 
We give a precise comparison in 
Eq.~\eqref{eq:rpca_subopt}. 
We instead use two filters to 
add back some of those mistakenly removed points. 
The First-Filter partitions the input set into a good set $\Scal_G$ and a bad set $\Scal_0\setminus \Scal_G$.
If the bad set  contributed to a significant portion of the projected trace (this can be detected by the shift in the mean of the remaining points $\mu^{\Scal_G}$), a second filter is applied to the bad set, recovering some of the uncorrupted examples. 
 This algorithm and our analysis 
applies more generally to   
any random vector, and may be of independent interest in other applications requiring robust PCA.
Under a mild assumption that $\x_i\sim\Pcal$ has a bounded fourth-moment, we prove the following, with a proof in \S\ref{sec:proof-of-rpca}.


\begin{proposition}[Robust PCA for general PSD matrices] \label{prop:rpca}
Let $S = \set{\x_i \sim \Pcal }_{i=1}^n$  where $\SIGMA\coloneqq\Exp{\x\sim\Pcal}{\x\x^\top}$ is the second moment of $\Pcal$ supported on $\Rd{d}$.
Given $k\in\Natural$, $\delta\in\round{0,0.5}$, and a corrupted dataset $S'$ with $\alpha\in(0,1/36]$ fraction  corrupted arbitrarily, 
if
$\Pcal$ has a bounded support such that  $\|\x\x^\top-\SIGMA\|_2 \leq B$ for $\x\sim\Pcal$ with probability one, and a bounded $4$-th moment such that  
$ \max_{\normF{\A}\leq 1, \rank(\A)\le k} 
\Exp{\x\sim\Pcal}{\big(\left\langle \A,\x\x^\top-\SIGMA\right\rangle \big)^2}
\leq \nu^2 $, and $n = \Omega((dk^2+(B/\nu)\sqrt{k\alpha})\log(d/(\delta\alpha))/\alpha)$, then with probability at least $1-\delta$, 
\begin{align}
\Tr{{\cal P}_k(\SIGMA)} - \Tr{\hat\U^\top \SIGMA \hat\U} &=  \BigO{\alpha\Tr{{\cal P}_k(\SIGMA)} + \nu\sqrt{k\alpha}}\;,\label{eq:rpca_guarantee}\\ 
\text{and }\quad\norm{*}{  \SIGMA -  \hat\U\hat\U^\top \SIGMA  \hat\U\hat \U^\top  } &\le \norm{*}{  \SIGMA -  {\cal P}_k \round{\SIGMA}  }+\BigO{\alpha \norm{*}{ {\cal P}_k \round{\SIGMA} }+\nu \sqrt{k\alpha}}\;. 
\label{eq:rpca_opt} 
\end{align}
where $\hat\U$ is the output of Algorithm~\ref{alg:filtering}, and $\Pcal_k(\cdot)$ is the best rank-$k$ approximation of a matrix in $\ell_2$.
\end{proposition}

The first term in the RHS of  Eq.~\eqref{eq:rpca_opt} is unavoidable, as we are outputting a rank-$k$ subspace.
In the setting of Theorem~\ref{thm:main} we are interested in, 
the last term of $\nu\sqrt{k \alpha}$ dominates the second term. 
We next show that this  cannot be improved upon;  
no algorithm can learn the subspace with an additive error smaller than $\Omega(\nu\sqrt{k\alpha})$ under $\alpha$ fraction of corruption, even with infinite samples. 
In the following minimax lower bound, since the total variation distance $D_{\rm TV}(\Pcal, \Pcal')\le \alpha$, the adversary can corrupted the datapoints from $\Pcal'$ to match the distribution $\Pcal$,  by changing just the $\alpha$ fraction. It is impossible to  tell if the corrupted samples came from $\Pcal$ or $\Pcal'$, resulting in an $\Ocal(\nu\sqrt{k\alpha})$ error. 
\begin{proposition}[Information theoretic lower bound]
    \label{prop:rpca_lb}
    Let $\hat{\U}(\{\x_i\}_{i=1}^n)$ be any subspace estimator 
    that takes $n$ samples from distribution $\Pcal$ as input, and estimates the $k$ principal components of $\SIGMA \coloneqq \Exp{\x\sim \Pcal'}{\x\x^\top}$ from another distribution $\Pcal'$ that is $\alpha$-close in total variation $D_{\rm TV}$. Then, 
 \begin{align*}
 \inf_{\hat{\U}}\max_{ \Pcal'\in \Theta_{\nu, B}} \max_{\Pcal: D_{\rm TV}(\Pcal, \Pcal')\le \alpha} \Exp{\{\x_i\}_{i=1}^n\sim \Pcal^n}{\norm{*}{\SIGMA - \hat\U\hat\U^\top\SIGMA\hat\U\hat\U^\top}- \norm{*}{  \SIGMA -  {\cal P}_k \round{\SIGMA}  }} \;=\;  \Omega(\nu \sqrt{k\alpha}),
 \end{align*}
 for any $k\ge 16, d\ge k^2/\alpha$, and  $B\ge 2d\nu$, 
 where 
 $\Theta_{\nu, B}$ is a set of  all distributions $\Dcal'$ on $\Rd{
 d}$ such that 
 $\max_{\normF{\A}\leq 1} {\mathbb E}_{\x\sim \Dcal'}[\left(\left\langle \A,\x\x^\top-
 \SIGMA 
 \right\rangle\right)^2]
\leq \nu^2$, and 
 $\Probability{\x\sim \Dcal'}{\enorm{\x\x^\top-\E{\x\x^\top}}\le B} = 1$.
\end{proposition}

{\bf Comparisons with  \cite{xu2012outlier}.}
 Outlier-Robust Principal Component Analysis (ORPCA) \cite{xu2012outlier, feng2012robust,yang2015unified} studies a similar problem  
under a Gaussian model. 
For comparison, we can modify the best known ORPCA estimator from \cite{xu2012outlier} to our setting in  Proposition~\ref{prop:rpca}, to get a semi-orthogonal $\hat\U$ achieving
\begin{align}
\norm{*}{  \SIGMA -  \hat\U\hat\U^\top \SIGMA  \hat\U \hat\U^\top  } = \norm{*}{  \SIGMA -  {\cal P}_k \round{\SIGMA}  }+\Ocal \big(\,\alpha^{1/2} \norm{*}{ {\cal P}_k \round{\SIGMA} }+ \nu {k}\alpha^{1/4} \,\big)\;. 
\label{eq:rpca_subopt}
\end{align}
We significantly improve in the dominant third term (see Eq.~\eqref{eq:rpca_opt}).
Simulation results supporting our theoretical prediction are shown  in Figure~\ref{fig:exp-error-comparison}, and we provide detailed analysis and the experimental setup in \S\ref{sec:ORPCA}. 

\begin{figure}[ht]
  \centering
  \includegraphics[width=0.35\linewidth]{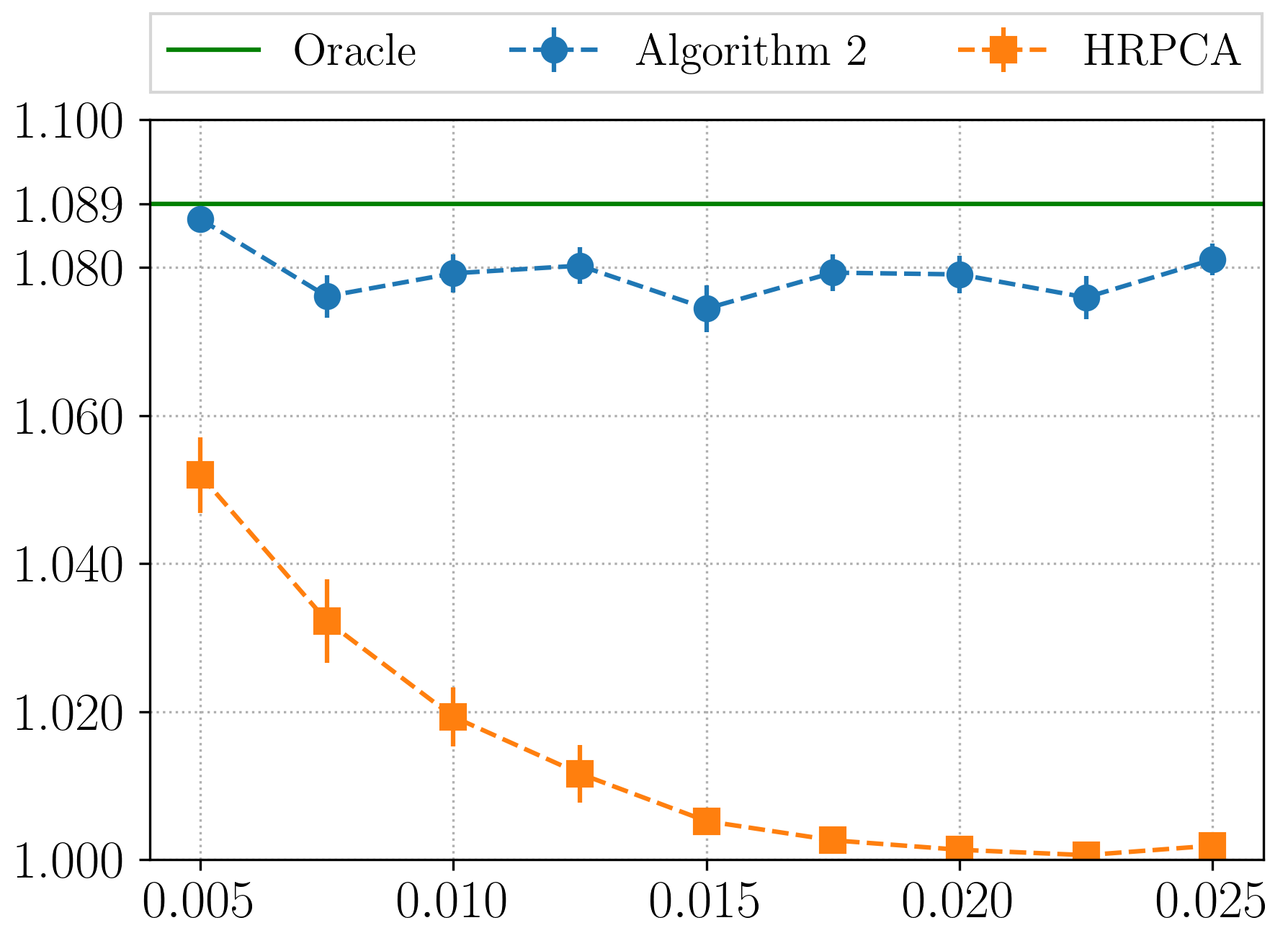}
  \put(-160,25){\rotatebox{90}{$\Tr{\hat\U^\top\SIGMA\hat\U}$}}
  \put(-65,-6){$\alpha$} 
  \hspace{0.5cm}
  \includegraphics[width=0.35\linewidth]{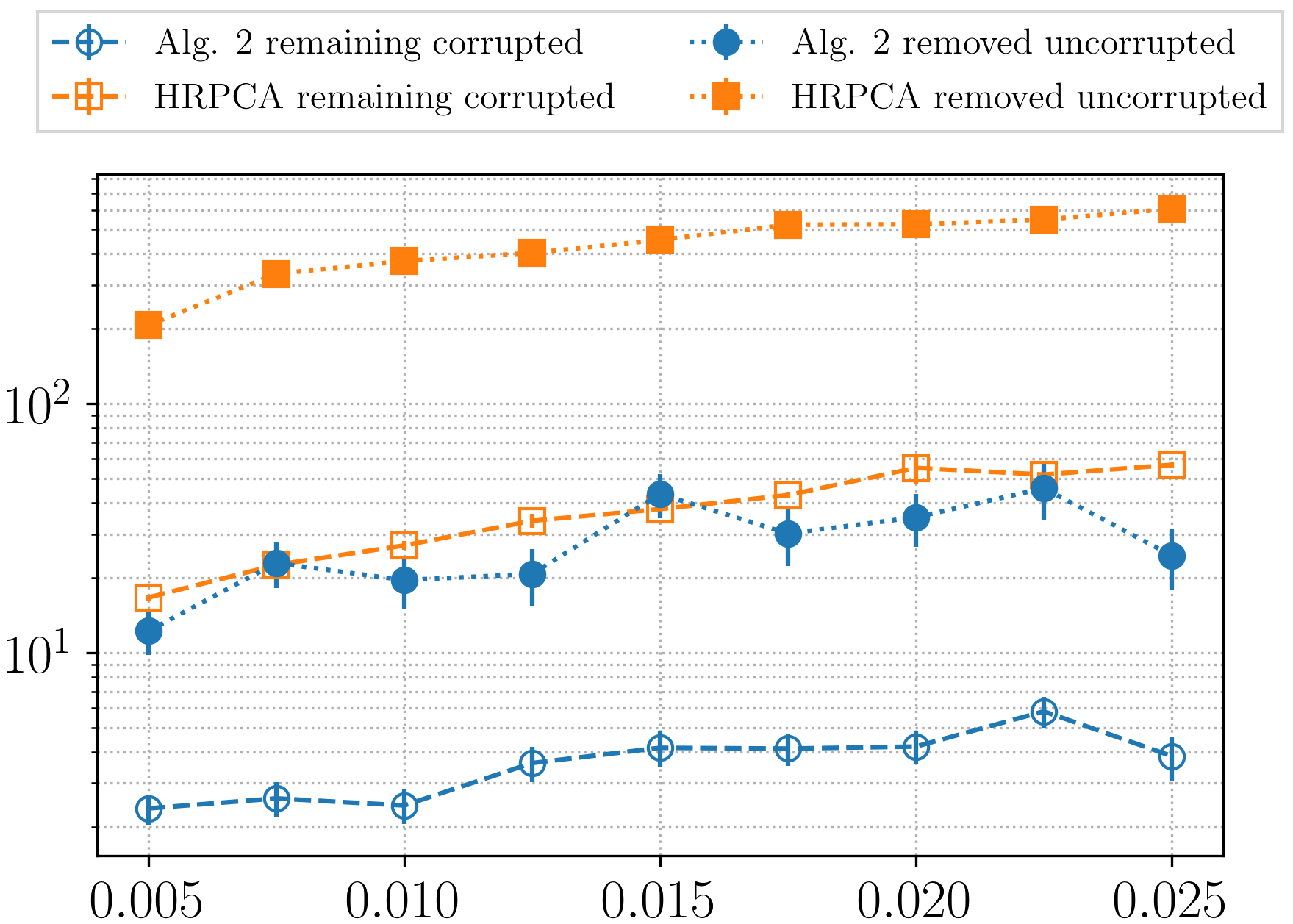}
  \put(-150,8){\rotatebox{90}{number of data points}}
  \put(-70,-6){$\alpha$}
  \caption{ Algorithm~\ref{alg:filtering} performs close to an oracle which knows the corrupted points, improving upon HRPCA of~\cite{xu2012outlier}, by removing more corrupted points and less uncorrupted ones.} 
  \label{fig:exp-error-comparison}
\end{figure}

\textcolor{black}{This is due to our double filtering in Algorithm~\ref{alg:basic_filter}, which guarantees that we remove more corrupted examples than uncorrupted examples. 
On the other hand, ORPCA estimator in \cite{xu2012outlier} uses a single filter that 
removes a single example per iteration. The removed example is a 
corrupted one with probability at least  $\gamma\in(0,1)$. 
This filter runs for 
 $\Theta(n\alpha/\gamma)$ iterations to ensure that the corrupted examples are sufficiently removed (remaining corrupted examples only contribute to $\gamma$-fraction of the second moment).
 However,  this filter also remove roughly a $\Theta(\alpha/\gamma )$-fraction of good examples. Under our setting, this  causes a $\Theta(\alpha/\gamma)$ multiplicative error and $\Theta(k\sqrt{\alpha/\gamma})$ additive deviation by Lemma~\ref{lem:goodevent}, part 3. This achieves 
\[
\Tr{\U^\top \SIGMA \U} \;\; \ge \;\; \round{1-\Theta(\gamma+\alpha/\gamma)}\Tr{{\cal P}_k(\SIGMA)} - \BigO{\nu k\sqrt{\alpha/\gamma}}\;, 
\]
 for any   $\gamma>0$. 
Setting $\gamma=\sqrt{\alpha}$ gives 
Eq.~\eqref{eq:rpca_subopt_appendix}, following a similar line of analysis as in \S\ref{sec:proof-of-rpca}. }

\textcolor{black}{{\bf Comparisons with filters based on the second moment of $z_i$'s.} Popular recent results on robust  estimation are based on filters that rely on the second moment~\cite{diakonikolas2017beingrobust, steinhardt2018resilience}. One might wonder if it is possible to apply these filters to $z_i$'s to remove the corrupted samples. Such approaches fail when $n=\Ocal(d)$, even when $\x_i$ are standard Gaussian; this is immediate from the fact that the empirical second moment of $z_i$'s does not concentrate until we have $n=\Ocal(d^2)$ uncorrupted samples. } 

\textcolor{black}{{\bf Comparisons with robust mean estimation \cite{cheng2019high}.}
Another approach is to 
use the existing off-the-shelf robust mean estimators, such as \cite{cheng2019high}, directly to estimate the second moment matrix $\SIGMA\in\R^{d\times d}$, as remarked in~\cite{shen2019iterative}.  
However, the application of~\cite[Theorem~1.3]{cheng2019high} does not take advantage of the
spiked low-rank structure of $\SIGMA$. This results in a larger sample complexity scaling as $n=\tilde\Omega(d^2/\alpha)$ to achieve  the following bound similar to Eq.~\eqref{eq:rpca_opt}.  
\begin{align}
    \norm{*}{\SIGMA- \hat\U\hat\U^\top \SIGMA \hat\U\hat\U^\top  } &= \norm{*}{\SIGMA-\Pcal_k(\SIGMA)} + \Ocal(\nu\,k \sqrt\alpha \, \|\SIGMA\|_2)
\end{align}
\begin{remark}\label{lem:robust_M}
    Given $\alpha\in\round{0,1/3}$ fraction corrupted tasks, there exists an algorithm~\cite{cheng2019high} that can robustly estimate the matrix $\M = \sumj{1}{k}p_j\w_j\w_j^\top$ with $n = \Omega\round{\frac{d^2}{\alpha}\log \frac{d}{\delta}}$, and time $\tilde\Ocal\round{nd^2/\mathrm{\poly}(\alpha)}$. The algorithm returns $\bar\M$ that satisfies
    \begin{align*}
        \enorm{\bar\M-\M} &\lesssim \rho^2\sqrt{\alpha} \\
    \end{align*}
    with probability at least $1-\delta$ for $\delta\in\round{0,1}$. 
\end{remark}
}
\section{Details of the algorithm and the analyses} 
\label{sec:detail} 

We  explain  and  analyze  each  step  in  Algorithm  \ref{alg:sketch}, which  imply  our  main  result, as shown in \S\ref{sec:metalearning_proof}. 

\subsection{Robust Subspace estimation}

Building upon the tight guarantees of Proposition~\ref{prop:rpca}, 
Algorithm~\ref{alg:filtering} 
achieves the following guarantee, first when $t_{L1}=1$.  
The conditions depend on the ground truths meta-parameters.
With $\tilde\Ocal(d)$ samples, 
we can tolerate up to 
$\Ocal(\eps^2p_{\rm min}^2/k^2)$ corruption, in an ideal case when $\W\in\Rd{d\times k}$ is a semi-orthogonal matrix.  
For the worst case $\W$, it is 
$\Ocal(\eps^6p_{\rm min}^2/k^2)$.


\begin{lemma}[Learning the subspace]
    \label{lem:projection}
    Under Assumptions~\ref{asmp} and~\ref{asmp:adv},
    for any target probability $\delta\in(0,0.5)$, and $\eps>0$,
    if $t_{L1}=1$, 
    \begin{eqnarray}
        n_{L1} 
    \;=\; \tilde\Omega\round{\,{dk^2} \min\set{\frac{\rho^4 k^2}{\eps^2\sigma_{\rm min}^2}, \frac{k^2}{\eps^6p_{\rm min}^2}} \,}\;, \;
    \text{ and }\;\;\;
        \alpha_{L1} \;=\; \Ocal\round{\max\set{\frac{\eps^2\sigma_{\rm min}^2}{\rho^4 k^2}, \frac{\eps^6p_{\rm min}^2}{k^2}}}\;,
    \end{eqnarray}
    then the semi-orthogonal matrix $\hat\U\in\Rd{d\times k}$ from robust subspace estimation in Algorithm~\ref{alg:filtering} achieves 
    \begin{eqnarray} 
     \enorm{\round{\I-\hat\U\hat\U^\top }\w_i} \;\; \leq\;\; \eps\rho \;,    \label{eq:subspace-guarantee}
    \end{eqnarray} 
    with probability at least $1-\delta$, 
    where 
    $\sigma_{\rm min}$ is the minimum singular value of  $\sum_{j=1}^{k}p_j\w_j\w_j^\top$. 
\end{lemma}

We provide a proof in \S~\ref{sec:projection_proof}. 
When $t_{L1}>1$, we get the following sufficient condition. The dominant first term requires  the effective sample size  $n_{L1}\, t_{L1}=\tilde\Omega(d\poly(k))$, 
which is linear in $d$.  

\begin{remark}[Handling $t_{L1}>1$]
Lemma~\ref{lem:projection} can be naturally generalized to the case where $1<t_{L1}<d$, and the requirement on $n_{L1}$ is 
\begin{eqnarray}
    n_{L1} 
    \;=\; \tilde\Omega\round{\,\frac{dk^2}{t_{L1}} \min\set{\frac{\rho^4 k^2}{\eps^2\sigma_{\rm min}^2}, \frac{k^2}{\eps^6p_{\rm min}^2}} + k\min\set{\frac{\rho^8 k^4}{\eps^4\sigma_{\rm min}^4}, \frac{k^4}{\eps^{12}p_{\rm min}^4}}\,}\;, 
\end{eqnarray}\label{rem:projection}
\end{remark}
\noindent
{\bf Time complexity:} ${\cal O}(n_{L1}t_{L1}d)$ for computing $\hat\beta_{i,j}$'s and ${\cal O}(n_{L1}^2t_{L1}^2dk\alpha \log(1/\delta))$ for  the filtering algorithm which uses $k$\_SVD~\cite{AL16}. The running time is from the fact that Double-Filtering (Algorithm~\ref{alg:basic_filter}) executes at most $\BigO{n_{L1}t_{L1}\alpha}$ times, and the running time of each execution is dominated by $k$\_SVD which takes $\BigO{n_{L1}t_{L1}dk}$ time.
\begin{remark}[Gaussianity assumptions]
 Although our robust PCA algorithm~\ref{alg:filtering} only requires bounded fourth moment assumption, our robust subspace estimation succeeds with the fact that
 \[
 {\mathbb E}[y_{i,j}^2\x_{i,j}\x_{i,j}^\top]=c\I+2\sum_{\ell=1}^k p_\ell \w_\ell\w_\ell^\top
 \]
 for some constant $c\geq 0$. This depends on the fourth moment property of Gaussian (Fact~\ref{fact:gauss-4}). \cite{2020arXiv200208936K} adopts a different approach by taking
 \[
 {\mathbb E}[y_{i,j}y_{i,j'}\x_{i,j}\x_{i,j'}^\top]=\sum_{\ell=1}^k p_\ell \w_\ell\w_\ell^\top
 \]
 for $j\ne j'$, which is able to handle general sub-Gaussian $\x$. While it is possible to make the approach in~\cite{2020arXiv200208936K} robust to outliers with robust mean estimation techniques (Remark~\ref{lem:robust_M}), such approaches requires an sub-optimal $\Omega(d^2)$ samples complexity. How to make the approach in~\cite{2020arXiv200208936K} robust with only linear dependency in $d$ remains an open problem, and the key obstacle is that the random matrix $y_{i,j}y_{i,j'}\x_{i,j}\x_{i,j'}^\top$ is not PSD.
\end{remark}

\subsection{Robust Clustering} 

Once we have the subspace, 
we use the 
 Sum-of-Squares (SOS) algorithm 
of 
\cite{2017arXiv171107465K} to cluster the $k$-dimensional points $\hat\U^\top\hat\beta_i$'s where $\hat\beta_i=(1/t_H)\sum_{j\in[t_H]}y_{i,j}\x_{i,j}$. For a value of
$m$ as discussed in Remark~\ref{rem:choice-m}, 
we can exploit the $m$-th order moment to filter out corrupted points and recover the clusters. This allows us to gracefully trade off  $t_H$ and $n_H$, breaking the barrier of $t=\Omega(k^{1/2})$ in \cite{2020arXiv200208936K}. Further, this approach is robust against adversarial corruption up to a $\Ocal(p_{\rm min})$ fraction of the data. 
We 
explicitly write the algorithm in 
\S\ref{sec:cluster_algorithm} and 
provide a proof in \S\ref{sec:cluster_proof}. 

\textcolor{black}{
Previous work~\cite{2020arXiv200208936K} proposed a linkage based clustering algorithm that is able to correctly cluster $\hat\beta_i$'s as long as $t_H=\Omega(\sqrt{k})$ and the second moment  is bounded. However, the algorithm fails when $t_H=o(\sqrt{k})$, and it has been noticed in~\cite{2020arXiv200208936K} that the failure of such kind of algorithms is inherent since it only relies some boundedness condition on the second moments. It is natural to ask whether it is possible to exploit the boundedness of higher order moments (larger than $2$) of the distribution to obtain stronger clustering results. Assuming boundedness of higher order moments is not too restrictive, since typical distributional assumptions, e.g. sub-Gaussianity, often imply boundedness for arbitrarily high order moments.}

\textcolor{black}{It turns out in order to \textit{efficiently} exploit the higher order moments assumptions for clustering, one need slight stronger condition than boundedness, that is the moments are sum-of-squares bounded, meaning there \textit{exist} sum-of-squares proofs showing that the moments are bounded~\cite{2017arXiv171107465K, hopkins2018mixture}. It is also shown in~\cite{2017arXiv171107465K} that a Poincar\'e distribution has sum-of-squares bounded moments, and thus their algorithm can be applied to clustering any Poincar\'e distributions.}

\textcolor{black}{It turns out that in our model, even assuming that $\x_{i,j}$ follows from isotropic Gaussian distribution, the distribution of $\hat\beta_i$ is not Poincar\'e and thus preventing us from applying the result in~\cite{2017arXiv171107465K} directly. Interestingly, as we showed in Lemma~\ref{lem:emp-sos-proof}, though $\hat\beta_i$ is not Poincar\'e, the high order moments of $\hat\beta_i$ is still sum-of-squares bounded under Gaussianity assumption of $\x_{i,j}$, and thus we can apply the result of~\cite{2017arXiv171107465K} to efficiently cluster $\hat\beta_i$'s, with guarantees formalized in Lemma~\ref{lem:cluster}.}




\begin{lemma}[Clustering and initial parameter estimation]
Under Assumptions~\ref{asmp} and~\ref{asmp:adv}, if $\alpha_{H}$ fraction of the tasks are adversarially corrupted in $\Dcal_H$, and given a semi-orthogonal matrix $\hat\U$ satisfying Eq.~\eqref{eq:subspace-guarantee} with $\eps = \BigO{\Delta/\rho}$,  
Algorithm~\ref{alg:recluster} with a choice of $m\in\Natural$, and Algorithm~\ref{alg:estimate_rl2} outputs $\curly{\tilde\w_\ell\in\Rd{d},\tilde{r}_\ell^2}_{j\in[k]}$ satisfying
\begin{align}
\enorm{\tilde\w_\ell-\w_\ell} &\leq \frac{\Delta}{10}\;, \qquad\text{and}\qquad\abs{\tilde{r}^2_\ell - r^2_\ell} \leq r_j^2\frac{\Delta^2}{50\rho^2} \;, \qquad\forall\ \ell \in\squarebrack{k}\; ,
\label{eq:cluster_guarantee}
\end{align}
with probability at least $1-\delta$, where $\tilde{r}_\ell^2$ is the robust estimate of $r_\ell^2 \coloneqq \esqnorm{\tilde\w_\ell-\w_\ell}+s_\ell^2$, if 
\begin{align} 
    \alpha_{H}<\min\set{\frac{p_{\rm min}}{16},\frac{C\Delta^2\sqrt{t_H}p_{\rm min}}{\rho^2\log \round{\nicefrac{\rho^2}{\Delta^2t_H}}}}, \quad
    n_H = \tilde\Omega\round{\frac{\round{km}^{\Theta(m)}}{p_{\rm min}} + \frac{\rho^4}{\Delta^4t_Hp_{\rm min}}}\;,
\end{align}
and $t_H =  \Omega({m\rho^2}/({p_{\rm min}^{2/m}\Delta^2}))$ 
for some universal constant $C>0$.
\label{lem:cluster}
\end{lemma}




We can therefore get the range of $m$, such that the condition on $t_H$ holds. Which is when
\[
\frac{2\log k}{-W_{-1}\round{-\frac{2c\rho^2\log k}{t_H\Delta^2}}} \leq m \leq \frac{2\log k}{-W_{0}\round{-\frac{2c\rho^2\log k}{t_H\Delta^2}}}
\]
for some $c>0$, where $W_0$ and $W_{-1}$ are the Lambert W function, only if $t_H\geq 2ec\rho^2\Delta^{-2}\log k$.

\noindent
{\bf Time complexity:} 
The robust clustering algorithm runs in time $\BigO{(n_H\,k)^{\BigO{m}}\log(1/\delta)}$.

\begin{remark}[Gaussianity assumption]
The only distributional assumption required for the clustering algorithm is that $y_{i,j}\x_{i,j}$ is SOS bounded. It is noted in~\cite{kothari2018robust} that this is a general assumption, and the algorithm can be applied much more broadly than Gaussian.
\end{remark}

\subsection{Classification and robust estimation}

Once the $\w_\ell$'s are estimated from the robust clustering step,   
we can efficiently classify 
any light task with $t_{L2}=\Omega(\log(k n_{L2}))$. 
After classification, we use a robust linear regression method \cite{diakonikolas2019efficient} on each group, which can tolerate up to $\BigO{p_{\rm min}}$ fraction of corruption. 
It is critical that we  separate the role of heavy and light tasks as 
initial estimation (robust clustering) and refinement (classification and robust  estimation). This allows us to have abundant $n_{L2}p_{\rm min}=\tilde\Omega(d/t_{L2})$ light tasks compensate for 
scarce $n_Hp_{\rm min}=\Omega(\poly(k))$ heavy tasks. 
We provide the algorithm and a proof in \S~\ref{sec:proof-classification}.

\begin{lemma}[Refined parameter estimation via classification]\label{lem:refined-est}
Under Assumptions~\ref{asmp} and \ref{asmp:adv} with $\alpha_{L2}$ fraction of the tasks are adversarially corrupted in $\Dcal_{L2}$, and given estimated parameters $\tilde{\w}_i$, $\tilde{r}_i$ satisfying 
Eq.~\eqref{eq:cluster_guarantee}, 
with probability $1-\delta$, for any accuracy $\eps>0$, if 
$t_{L2} = \Omega\round{\rho^4\log (k n_{L2}/\delta)/\Delta^4}$, 
\begin{eqnarray}
\alpha_{L2}\;\;=\;\;\BigO{p_{\rm min} \eps\log^{-1}(1/\eps)}\;, \quad \text{and}\quad n_{L2} \;\; =\;\; \tilde{\Omega}\round{dt_{L2}^{-1} p_{\rm min}^{-1}\eps^{-2}}\;, 
\end{eqnarray}
then Algorithm~\ref{alg:classification}  outputs  estimated parameters $\{\hat\w_i\}_{i=1}^k, \{\hat{s}_i\}_{i=1}^k, \{\hat{p}_i\}_{i=1}^k$ such that
for all $i\in\squarebrack{k}$, 
\begin{eqnarray}
\label{eq:estimation_err}
\enorm{\hat{\w}_i-\w_i} \le \eps s_i\; ,\quad \abs{\hat{s}_i^2-s_i^2} \le {\eps s_i^2}/{\sqrt{t_{L2}}}\; ,\quad\text{and}\quad\abs{\hat{p}_i-p_i} \le  \eps\,p_i\, \sqrt{{t_{L2}}/{d}}+\alpha_{L2}\;.
\end{eqnarray} 
\end{lemma}

\noindent
{\bf Time complexity:} $\BigO{n_{L2}^2t_{L2}^2d}$. 
The running time of Algorithm~\ref{alg:classification} is dominated by the robust linear regression procedure (~\cite[Algorithm~2]{diakonikolas2019efficient}), which takes at most $n_{L2}t_{L2}\alpha$ iterations and $\BigO{n_{L2}t_{L2}d}$ time for SVD per iteration.

\section{Related work}
\label{sec:related} 
 
 {\bf Mixed linear regression.} 
 Previous work on mixed linear regression focus on the setting where each task has only one sample, i.e.~$t_i=1$. 
 As a consequence, all the previous work suffer from either the sample complexity or the running time that scale exponentially in $k$ (specifically at least $\exp(\sqrt{k})$)~\cite{zhong2016mixed, li2018learning, chen2019learning, shen2019iterative}. In other cases, such blow-up in complexity is hidden in the  dependence of the inverse of $k$-th singular value of a moment matrix, which can be arbitrarily large \cite{CL13, yi2016solving, sedghi2016provable}. 

{\bf Multi-task learning.}~\cite{baxter2000model,ando2005framework,rish2008closed,orlitsky2005supervised, argyriou2008convex,harchaoui2012large,amit2007uncovering,pontil2013excess, balcan2015efficient, bullins2019generalize} address a similar problem as our setting, 
but focusing on finding a common low-dimensional linear representation, where all tasks can be accurately solved. Typically, the batch size is fixed and the performance is evaluated on the past tasks used in training.  
 Close to ours are a few concurrent work~\cite{du2020few, tripuraneni2020provable}, but their focus is still on recovering the common subspace, and not the meta-parameters.

 {\bf Robust regression.} There are  several work on robust linear regression and sparse regression problems, \cite{bhatia2015robust,BhatiaJKK17, BDLS17, gao2020robust, prasad2018robust, klivans2018efficient, diakonikolas2019sever, liu2018high, karmalkar2019compressed,dalalyan2019outlier, mukhoty2019globally, charikar2017learning, raghavendra2020list, karmalkar2019list}. 
The recent advances in the list-decodable setting~\cite{charikar2017learning, raghavendra2020list, karmalkar2019list} can potentially be applied to our mixture setting, 
but the sample complexity  is exponentially large. 
 Recently,~\cite{shen2019iterative}  studies the robust mixed linear regression problem. In contrast to our setting which allows random noise on the label and adversarial corruption on both covariate $\x$ and label $y$, their setting assumes no noise on label $y$'s, and the adversary is only allowed to corrupt $\alpha$-fraction the label $y$'s. Although their algorithm requires only $\tilde\Ocal(dk)$ samples, the   running time is $\tilde\Ocal(k^knd)$  
 and also requires a good estimate of the subspace spanned by $\{\w_\ell\}_{\ell=1}^k$. 
 
{\bf Sum-of-squares algorithms.} Our work is inspired by sum-of-squares algorithms that have recently been studied on many learning problems, including linear regression~\cite{karmalkar2019list, raghavendra2020list}, mixture models~\cite{ hopkins2018mixture, kothari2018robust, jia2019robustly, raghavendra2018high, diakonikolas2020robustly}, mean estimation~\cite{hopkins2020mean, cherapanamjeri2020list}, subspace estimation~\cite{bakshi2020list}.
 This provides the key building block of our approach, in breaking the second moment barrier of linkage-based clustering algorithms.

\section{Conclusion}

\label{sec:discussion}
By exploiting similarities on a collection of related but different tasks, meta-learning predicts a newly arriving task with a 
far greater accuracy than what can be achieved in isolation. 
We ask two fundamental questions under a canonical model of $k$-mixed linear regression: ($i$) can we meta-learn from tasks with only a few training examples each?; and $(ii)$ can we meta-learn from tasks when only part of the data can be trusted?
We introduce a novel spectral approach that achieves both simultaneously, significantly improving the required batch size from  $\Omega(k^{1/2})$ to $\Omega(\log k)$ while being robust to adversarial corruption.   
We use a sum-of-squares algorithm 
to exploit the higher order moments and design a novel robust subspace estimation algorithm  that achieves  optimal guarantees.


 {\bf Closing the gap in robust subspace estimation.} 
\cite{xu2012outlier, feng2012robust,yang2015unified,cherapanamjeri2017thresholding} study robust PCA under the Gaussian assumption.
For the reasons explained in \S\ref{sec:subspace}, 
the rate is sub-optimal in $\alpha$ in comparisons to an information theoretic lower bound with a multiplicative factor of  $(1-\Theta(\alpha))$. 
Applying the proposed Algorithm~\ref{alg:filtering}, it is possible to generalize Proposition~\ref{prop:rpca} to this Gaussian setting and achieve an optimal upper bound. We leave this as a future research direction, and 
 provide a sketch of how to adapt the proof of our algorithm to the exponential tail setting in Section~\ref{sec:adaption}. 

\textcolor{black}{
Concretely, our analysis of Algorithm~\ref{alg:filtering} assumes only a bounded 4-th moment of the input vector $\z_i$, of the form $\Prob{\abs{(\vvec^\top\z_i)^2 - \vvec^\top\SIGMA \vvec} \ge  t}\le c\,t^{-2}$. 
Our current proof proceeds by focusing on that $1-\alpha$ probability mass, which falls in the interval $[-\sqrt{1/\alpha}, \sqrt{1/\alpha}]$. This is tight with only the second moment assumption. 
More generally, one can consider a family of distributions satisfying  $\Prob{\abs{(\vvec^\top\z_i)^2 - \vvec^\top\SIGMA \vvec} \ge \mathrm{variance}\cdot t}\le \exp(-t^\gamma)$.
If we have such an exponential concentration, we can instead focus on  the subset of examples with second moment $\abs{(\vvec^\top\z_i)^2 - \vvec^\top\SIGMA \vvec}$ falling in the interval $[-\log^{1/\gamma}(1/\alpha), \log^{1/\gamma}(1/\alpha)]$. 
This bounded distribution has a sub-Gaussian norm $\sqrt{k}\log^{1/\gamma}(1/\alpha)$, and thus we can apply the sub-Gaussian filter (Proposition A.7 of~\cite{diakonikolas2017beingrobust}) to learn $\E{(\vvec^\top\z_i)^2}$ with error
We can obtain an error of $\alpha\sqrt{k}\log^{1/\gamma}(1/\alpha)$. We provide a sketch of how to adapt the proof of our algorithm to the exponential tail setting in Section~\ref{sec:adaption}.}

\textcolor{black}{ {\bf Removing the Gaussianity assumption.} 
Our approach relies on the special structure of the $4$-th moment of $\x_{i,j}$ and the SOS boundedness of higher order moments of $\x_{i,j}$. The approach in~\cite{2020arXiv200208936K} is able to get around the $4$-th moment requirement, and it is an interesting open problem to make the approach robust to outliers and still preserve the $\tilde\Ocal(d)$ sample complexity. while this class of SOS bounded distributions is fairly broad, as noted in~\cite{2017arXiv171107465K}, one could hope to establish sum-of-squares bounds for even broader families. For examples, it remains open that whether sum-of-squares certifies moment tensors for all sub-Gaussian distributions.}

 \begin{ack}
 Sham Kakade acknowledges funding from the Washington Research Foundation for Innovation in Data-intensive Discovery, and the NSF Awards CCF-1637360, CCF-1703574, and CCF-1740551. 
 Sewoong Oh acknowledges funding from Google Faculty Research Award, and NSF awards CCF-1927712, CCF-1705007, IIS-1929955, and CNS-2002664.
 \end{ack}

\bibliographystyle{plain} 
\bibliography{ref}

\newpage
\appendix

\section*{Appendix}

\section{Proof of Corollary~\ref{coro:small-t}, Corollary~\ref{coro:adv}, Corollary~\ref{coro:first}}
\label{sec:proof_corollary}

\begin{algorithm}[ht]
   \caption{Meta-learning without adversarial corruptions}
   \label{alg:sketch-nocorruptions}
\medskip
{\bf Meta-learning}
\vspace{-0.1cm}
\begin{enumerate}
    \item {\em Subspace estimation:}~Compute subspace $\hat\U$ with~\cite[Algorithm~2]{2020arXiv200208936K} which approximates ${\rm span}\curly{\w_1,\ldots, \w_k}$.
    \vspace{-0.2cm}
    \item {\em Clustering:}~Project the heavy tasks onto the subspace of $\hat\U$, perform $k$ clustering with Algorithm~\ref{alg:recluster}, and estimate $\tilde{\w}_\ell$. Also estimate $\tilde{r}_\ell^2$ using Algorithm~\ref{alg:estimate_rl2} for each cluster $\ell\in\squarebrack{k}$.
    \vspace{-0.1cm}
    \item {\em Classification:}~Perform likelihood-based classification of the light tasks using $\{\tilde{\w}_\ell,\tilde{r}_\ell^2\}_{\ell=1}^k$ estimated from the \textit{Clustering} step; compute refined estimates $\{\hat{\w}_\ell, \hat{s}_\ell, \hat{p}_\ell\}_{\ell=1}^k$ of $\theta$ using~\cite[Algorithm~4]{2020arXiv200208936K}.
    \vspace{-0.1cm}
\end{enumerate}

{\bf Prediction}
\vspace{-0.1cm}
\begin{enumerate}
    \item[4.] {\em Prediction:}~Perform MAP or Bayes optimal prediction using the estimated meta-parameter.
\end{enumerate}
\end{algorithm}

We assume that the meta-parameter satisfies $\Delta = \Theta(1)$, $\rho = \Theta(1)$, $p_{\rm min}=\Theta(1/k)$ for
Corollary~\ref{coro:small-t}, Corollary~\ref{coro:adv}, Corollary~\ref{coro:first}.
\begin{proof}[Proof of Corollary~\ref{coro:first}]
Recall that for this corollary, we assume $\Delta = \Theta(1)$, $\rho = \Theta(1)$, $p_{\rm min}=\Theta(1/k)$ and there
is no adversarial corruption. Thus we execute Algorithm~\ref{alg:sketch-nocorruptions} in this setting. We invoke~\cite[Lemma~5.1]{2020arXiv200208936K} and get that given $t=\Omega(1)$ and $tn = \tilde\Omega(kn)$, the estimated subspace $\hat\U$ satisfies that for all $i\in [k]$,
\begin{eqnarray} 
 \enorm{\round{\I-\hat\U\hat\U^\top }\w_\ell} \;\; \leq\;\; \Delta/(10)\quad\forall\ \ell\in\squarebrack{k}\;,
\end{eqnarray} 
Then, we can invoke Lemma~\ref{lem:cluster} with $m = \Theta(\log k)$, $t = \Theta(\log k)$, $n = k^{\Theta(\log k)}$ and get that the estimated parameters $\tilde\w_\ell$, $\tilde{r}_\ell^2$ satisfy
\begin{align*}
\enorm{\tilde\w_\ell-\w_\ell} &\leq \frac{\Delta}{10}\;, \qquad\text{and}\qquad\abs{\tilde{r}^2_\ell - r^2_\ell} \leq r_\ell^2\frac{\Delta^2}{50\rho^2} \;, \qquad\forall\ \ell\in\squarebrack{k}\; ,
\end{align*}
Finally, given $t = \tilde\Omega(1)$ and $n = \Omega(dk)$, the output of the classification step satisfies 
\begin{eqnarray}
\enorm{\hat{\w}_\ell-\w_\ell} = \BigO{1}\; ,\quad \abs{\hat{s}_\ell^2-s_\ell^2} = \BigO{s_\ell^2}\; ,\quad\text{and}\quad \abs{\hat{p}_\ell-p_\ell} =\BigO{p_\ell}\quad\forall\ \ell\in\squarebrack{k}.
\end{eqnarray}
To conclude, with $t = \tilde\Omega(1), n = \tilde\Omega(dk^2+k^{\Theta(\log k)})$, Algorithm~\ref{alg:sketch-nocorruptions} can estimate model parameters $\theta$ with arbitrary small constant error.
\end{proof}

\begin{proof}[Proof of Corollary~\ref{coro:small-t}]
The proof is the same as Corollary~\ref{coro:small-t}.
\end{proof}
\begin{proof}[Proof of Corollary~\ref{coro:adv}]
With the assumptions that $\Delta = \Omega(1)$, $\rho = \Omega(1)$, $p_{\rm min}=\Omega(1/k)$, $t_{L1} = t_{L2} = \tilde\Omega(1)$, $t_H = \Omega({m}{k^{1/m}})$ Theorem~\ref{thm:main} can be simplified to that 
for all $i\in\squarebrack{k}$, with probability $1-\delta$
\begin{align*}
\enorm{\hat{\w}_i-\w_i} \le \eps s_i\; ,\;\;  \;
\abs{\hat{s}_i^2-s_i^2}\le {\eps s_i^2}\; ,\quad\text{ and }\; \;\;
\abs{\hat{p}_i-p_i} \le \eps p_i\;+\; \alpha_{L2}\;, 
\end{align*}
as long as,

\begin{tabular}[!htbp]{ll}
    $n_{L1} = \tilde\Omega\round{\frac{dk^2}{\tilde\alpha}}\;,$ & $\alpha_{L1} = \BigO{\tilde\alpha}$\;, \\
    $n_{H} = \tilde\Omega\round{
    (km)^{\Theta(m)}}$\;, & $\alpha_H = \tilde\Ocal\round{1/k}$\;,\\
    $n_{L2}= \tilde\Omega\round{\frac{dk}{\eps^2}}$\;, & $\alpha_{L2}= \tilde\Ocal( \eps/k)$\;,
\end{tabular}

where $\tilde\alpha\coloneqq 1/k^4$. Using the assumption that $\eps\le 1/k^3$. this implies that as long as

\begin{tabular}[!htbp]{ll}
    $n_{L1} = \tilde\Omega\round{\frac{dk}{\eps^2}}\;,$ & $\alpha_{L1} = \BigO{\eps/k}$\;, \\
    $n_{H} = \tilde\Omega\round{
    (km)^{\Theta(m)}}$\;, & $\alpha_H = \tilde\Ocal\round{1/k}$\;,\\
    $n_{L2}= \tilde\Omega\round{\frac{dk}{\eps^2}}$\;, & $\alpha_{L2}= \tilde\Ocal( \eps/k)$\;,
\end{tabular}

Our algorithm can estimate the parameters up to error 
\begin{align*}
\enorm{\hat{\w}_i-\w_i} \le \eps s_i\; ,\;\;  \;
\abs{\hat{s}_i^2-s_i^2}\le {\eps s_i^2}\; ,\quad\text{ and }\; \;\;
\abs{\hat{p}_i-p_i} \le \eps/k\;, 
\end{align*}
for all $i\in\squarebrack{k}$.
\end{proof}
\section{Proof of meta-learning in Theorem~\ref{thm:main}}
\label{sec:metalearning_proof}

Applying Lemma~\ref{lem:projection} with $\eps=\Delta/(10\rho)$, 
we get a semi-orthogonal matrix $\hat\U$ satisfying 
\begin{eqnarray}
    \Big\| \round{\I-\hat\U\hat\U^\top}\w_i\Big\|_2 \;\leq\; \Delta/10 \;,
\end{eqnarray}
with $\tilde\alpha\coloneqq\max\set{  \Delta^2\sigma_{\rm min}^2/(\rho^6 k^2), \Delta^6 p_{\rm min}^2 /(k^2\rho^6)}$ if  
\begin{eqnarray}
    n_{L1} \;=\; \tilde\Omega\round{\frac{dk^2}{\tilde\alpha t_{L1}}+\frac{k }{\tilde\alpha^2}}\;,\; 
    \text{ and } \;\alpha_{L1}\;=\; \Ocal(\tilde\alpha)\;. 
\end{eqnarray}
Since we have sufficiently accurate estimate of the $k$-dimensional subspace spanned by the columns of $\W$, we can cluster the tasks more efficiently in this lower dimensional space using robust a clustering algorithm.
We use a SOS algorithm in Algorithm~\ref{alg:recluster} 
with a choice of $m\in\Natural$
such that 
$t_H=\Omega\round{m\rho^2/(p_{\rm min}^{2/m}\Delta^2)}$, 
to get 
\begin{eqnarray}
    \enorm{\tilde\w_j-\w_j} \;\leq\; \Delta/10\;,\text{ and } 
    \abs{\tilde r_j^2 - r_j^2 }\; \leq r_j^2\Delta^2/(50 \rho^2)\; ,
\end{eqnarray}
for all $j\in\squarebrack{k}$, using Lemma~\ref{lem:cluster}, which requires 
\begin{eqnarray}
    n_H = \tilde\Omega\round{
    \frac{(km)^{\Theta(m)}}{p_{\rm min}}+\frac{\rho^4}{\Delta^4 p_{\rm min}t_H}} \;,\;\text{ and } \alpha_H\;=\; \tilde\Ocal\round{p_{\rm min}\cdot\min\set{1,\frac{\Delta^2\sqrt{t_H}}{\rho^2}}}  \;.
\end{eqnarray}
It follows from Lemma~\ref{lem:refined-est} that for any desired accuracy $\eps\geq (\alpha_{L2}/p_{\rm min}) \log(p_{\rm min}/\alpha_{L2}) $ if 
\begin{eqnarray}
    n_{L2}\;=\;  \tilde\Omega\round{\frac{d}{t_{L2} p_{\rm min} \eps^2}}\;, \qquad\text{and}\qquad t_{L2} = \Omega\round{\log (k n_{L2}/\delta)/\Delta^4}
\end{eqnarray}
the output of our algorithm achieves 
\begin{eqnarray}
    \enorm{\hat{\w}_i-\w_i} &\le& \eps s_i\; , \\ 
\abs{\hat{s}_i^2-s_i^2}&\le&  \frac{\eps}{\sqrt{t_{L2}}}s_i^2\; ,\quad\text{and}\\
\abs{\hat{p}_i-p_i} &\le&  \eps \sqrt{t_{L2}/d}\,p_i + \alpha_{L2}.
\end{eqnarray}
for all $i\in\squarebrack{k}$, as long as $\alpha_{L2} = \BigO{p_{\rm min}\eps/\log\inv\eps}$.
\section{Proof of robust subspace estimation analysis in Lemma~\ref{lem:projection}}
\label{sec:projection_proof}


We first prove for the simple setting where $t_{L1} = 1$ and resolving a discrepancy in the independence when $t_{L1} > 1$ at the end of the section. Further, for 
notational convenience, we use $i$ in place of $(i,j)$, and $n$ for $n_{L1}t_{L1}$. 

First we compute the expectation of $\hat\beta_i\hat\beta_i^\top$. Using Lemma~\ref{lem:bound_M}, we have 
\begin{equation*}
\M := \E{\hat\beta_i\hat\beta_i^\top}  = 2\sumj{1}{k}p_j\w_j\w_j^\top + \round{\sumj{1}{k}p_j\round{\|\w_j\|^2+s_j^2}}\I_d,
\end{equation*}
and we define $\bar{\rho}^2 = \round{\sumj{1}{k}p_j(\|\w_j\|^2+s_j^2)}$.

Since our goal is to recover the top $k$ eigenspace of $\M$, we would like to apply Proposition~\ref{prop:rpca} to $\x_i = \hat{\beta}_i$, however $\hat{\beta}_i$ does not satisfy the spectral norm bound requirement on $\norm{2}{\x_i\x_i^\top - \SIGMA}$. The following proposition shows that we can resolve this issue through conditioning on the event that $\hat{\beta_i}$ is bounded.
\begin{proposition}\label{prop:pre-subspace}
For any $0 < \delta \le 1/2$, define event 
\[
\Ecal \coloneqq \set{\esqnorm{\hat\beta_i} \leq \rho\sqrt{d}\log(nd/\delta)\;, \ \forall\ i\in\squarebrack{n}}\;. 
\]

Then conditioned on event $\Ecal$, the distribution of $\hat{\beta}_i$ satisfies the prerequisite of Proposition~\ref{prop:rpca} with
\begin{enumerate}
    \item $\begin{aligned}
    \enorm{\hat\beta_i\hat\beta_i^\top - \E{\hat\beta_i\hat\beta_i^\top\middle|\Ecal}} \;\; \leq \;\;  \underbrace{ 5\rho^2{d}\log^2(nd/\delta)}_{=B}\;.\label{eq:defB}
    \end{aligned}$
    \item $\begin{aligned}
    \E{\Tr{\A\round{\hat\beta_i\hat\beta_i^\top - {\mathbb E}[\hat\beta_i\hat\beta_i^\top\,|\, \Ecal]}}^2\middle|\Ecal} \leq \underbrace{\BigO{k\rho^4}}_{=\nu^2(k)}\label{eq:nu2}\; .
    \end{aligned}$
\end{enumerate}

The mean shift is bounded under $\Ecal$ as: 
\begin{eqnarray}
    \enorm{\E{\hat\beta_i\hat\beta_i^\top} - \E{\hat\beta_i\hat\beta_i^\top\middle| \Ecal} } &\leq&  \BigO{\rho^2\sqrt{\delta}} \;,
    \label{eq:meanshift}
\end{eqnarray}
\end{proposition}
The proof is deferred to the end of the section.


Recall that $\M \coloneqq \E{\hat\beta_i\hat\beta_i^\top}$, and let us define $\M' \coloneqq \E{\hat\beta_i\hat\beta_i^\top \middle| \Ecal }$ be the mean conditioned on $\Ecal$. 
With Proposition~\ref{prop:pre-subspace}, we can apply Proposition~\ref{prop:rpca} to obtain a nuclear norm guarantee on our our subspace estimation algorithm
\begin{proposition}\label{propo:rpca}
Given that
\[
n = \tilde\Omega(dk^2/\alpha)\;,
\]
with probability $1-2\delta$, then Algorithm ~\ref{alg:filtering} returns a rank-$k$ semi-orthogonal matrix $\hat\U\in{\mathbb R}^{d\times k}$ satisfying
\begin{align*}
\norm{*}{\hat\U\hat\U^\top\round{\sumj{1}{k}p_j\w_j\w_j^\top}\hat\U\hat\U^\top - \sumj{1}{k}p_j\w_j\w_j^\top}
= \BigO{\rho^2k\sqrt{\alpha}}
\end{align*}
for $\delta\in\round{0,0.5}$.
\end{proposition}
\begin{proof}[Proof of Proposition~\ref{propo:rpca}]
We apply Proposition~\ref{prop:rpca} to $\x_i=\hat\beta_i$ conditioned on event $\Ecal$ and get
\begin{align}
    \Tr{\hat\U\hat\U^\top\M'\hat\U\hat\U^\top  } &\ge (1-\BigO{\alpha})\Tr{\Pcal_k(\M')} - \BigO{\rho^2k\sqrt{\alpha}}\label{eqn:Uhat-Mp-Uhat}
\end{align}
with probability $1-2\delta$, when
\[
n = \tilde\Omega\round{\round{dk^2 + \frac{\rho^2d}{\sqrt{k}\rho^2}\cdot\sqrt{k\alpha}}/\alpha}\, = \tilde\Omega(dk^2/\alpha).
\]
WLOG, for the remaining analysis we will assume $\delta \le \nicefrac{1}{nd}$. The nuclear norm term in the proposition statement can be bounded as 
\begin{align}
&\norm{*}{\hat\U\hat\U^\top(\M-\bar\rho^2\I)\hat\U\hat\U^\top-(\M-\bar\rho^2\I)}\nonumber\\
& = \Tr{\M-\bar\rho^2\I} - \Tr{\hat\U\hat\U^\top(\M-\bar\rho^2\I)\hat\U\hat\U^\top}\nonumber\\
& = \Tr{\M-\bar\rho^2\I} - \Tr{\hat\U^\top(\M'-\bar\rho^2\I)\hat\U} - \Tr{\hat\U^\top(\M-\M')\hat\U}\nonumber\\
&\le \Tr{\M-\bar\rho^2\I} - \Tr{\hat\U^\top(\M'-\bar\rho^2\I)\hat\U} +\rho^2k\sqrt{\delta}\; \text{ (Using Equation~\eqref{eq:meanshift}) }\nonumber\\
&= \Tr{\M-\bar\rho^2\I} - \Tr{\hat\U^\top\M'\hat\U} + k\bar\rho^2 +\rho^2k\sqrt{\delta}\nonumber\\
&\le \Tr{\M-\bar\rho^2\I} -(1-\BigO{\alpha})\Tr{\Pcal_k(\M')} + \rho^2k\sqrt{\alpha} + k\bar\rho^2 +\rho^2k\sqrt{\delta}\label{eqn:before-Mp}\\
&\omit\hfill\text{ (Using Equation~\eqref{eqn:Uhat-Mp-Uhat}) }\nonumber.
\end{align}
We need the following bound on $\Tr{\Pcal_k(\M')}$ before proceeding:
\begin{align}
\Tr{\Pcal_k(\M')} &= \Tr{\Pcal_k(\M'-\bar\rho^2\I)} + k\bar\rho^2\nonumber\\
&\ge \Tr{\Pcal_k(\M-\bar\rho^2\I)}-k\rho^2\sqrt{\delta}+k\bar\rho^2\label{eqn:Mp-Mk}\\
 &= \Tr{\M-\bar\rho^2\I}-k\rho^2\sqrt{\delta}+k\bar\rho^2\nonumber,
\end{align}
where Equation~\eqref{eqn:Mp-Mk} holds by the following matrix perturbation bound:
\begin{align*}
    \abs{\Tr{\Pcal_k(\M-\bar\rho^2\I)}-\Tr{\Pcal_k(\M'-\bar\rho^2\I)}} &\le \sum_{i=1}^k \abs{\lambda_i(\M'-\bar\rho^2\I)-\lambda_i(\M-\bar\rho^2\I)}\\
    &\le k\rho^2\sqrt{\delta}\qquad\text{(Using Equation~\eqref{eq:meanshift})}.
\end{align*}
Plugging Equation~\eqref{eqn:Mp-Mk} back in Equation~\eqref{eqn:before-Mp}, we have Equation~\eqref{eqn:before-Mp} bounded by
\begin{align*}
&\le \BigO{\alpha \Tr{\M-\bar\rho^2\I} +\alpha k\bar\rho^2 + \rho^2k\sqrt{\alpha} +\rho^2k\sqrt{\delta}}\\
&\le \BigO{\alpha \Tr{\M-\bar\rho^2\I} + \rho^2k\sqrt{\alpha}  +\rho^2k\sqrt{\delta}}\; \text{ (Using $\delta\le 1/nd\leq \alpha$)}\\
&\le \BigO{\alpha \Tr{\M-\bar\rho^2\I} + \rho^2k\sqrt{\alpha}}
\end{align*}

Thus, we have obtained that
\begin{align*}
&\norm{*}{\hat\U\hat\U^\top\round{\sumj{1}{k}p_j\w_j\w_j^\top}\hat\U\hat\U^\top - \sumj{1}{k}p_j\w_j\w_j^\top}\\
=& \BigO{\alpha \Tr{\sumj{1}{k}p_j\w_j\w_j^\top} + \rho^2k\sqrt{\alpha}}\\
=& \BigO{\rho^2k\sqrt{\alpha}}\qedhere.
\end{align*}
\end{proof}
The following lemma connects this nuclear norm bound to a subspace bound that we want. 
\begin{lemma}[Gap-free spectral bound]\label{lem:gapfree}
Given $k$ vectors $\x_1, \x_2, \cdots, \x_k \in \Rd{d}$, we define $\X_i = \x_i\x_i^\top$ for each $i \in [k]$. 
For  any $\gamma \ge 0$, $\sigma\in{\mathbb R}_+$, and any rank-$k$ PSD matrix $\hat\M\in \Rd{d\times d}$ such that 
\begin{align}\label{eq:lem_assump}
\norm{*}{\hat\M - {\cal P}_k\round{ \sigma^2 \I  + \sum_{i=1}^k \X_i } } \;\; \le \;\;\gamma,
\end{align}
where ${\cal P}_k(\cdot)$ is a best rank-$k$ approximation of a matrix, we have 
\begin{align}
    \sum_{i\in[k]} \enorm{\x_i^\top \round{\I - \hat\U\hat\U^\top}}^2  &\;\;\le\;\; \min\curly{\, 
    {\gamma^2 \sigma_{\rm max} }/\sigma_{\rm min}^2  \,,\, 2 \gamma^{2/3} \sigma_{\rm max}^{1/3}k^{2/3}  
    }\;, 
    \label{eq:subspacebound1}
\end{align}
where $\sigma_{\rm min}$  is the smallest non-zero singular value of $\sum_{i\in[k]} \X_i$, and $\sigma_{\rm max}=\enorm{\sum_{i\in[k]} \X_i}$, and 
$\hat\U\in \Rd{d\times k}$ is the matrix consisting of the top-$k$ singular vectors of $\hat\M$. Further, for all $i\in [k]$, we have 
\begin{align}
     \enorm{\x_i^\top \round{\I - \hat\U\hat\U^\top}}^2  &\;\;\le\;\; \min\curly{\, 
    {\gamma^2 \|\x_i\|_2^2 }/\sigma_{\rm min}^2  \,,\, 2 \gamma^{2/3} \|\x_i\|_2^{2/3}  
    }\;.
    \label{eq:subspacebound2}
\end{align}
\end{lemma}
We provide a proof in Section~\ref{sec:proof_gapfree}.
Using this lemma with $\sigma=0$, 
\[
\hat\M = \hat\U\hat\U^\top\round{\sum_{j=1}^{k}p_j\w_j\w_j^\top}\hat\U\hat\U^\top,
\]
$\X_i = p_i\w_i\w_i^\top$ for all $i\in\squarebrack{k}$, and the nuclear norm bound in Proposition~\ref{propo:rpca}, we get
\begin{align}
    p_i\esqnorm{\w_i^\top\round{\I-\hat\U\hat\U^\top}} &\leq \min\set{\rho^2\frac{\gamma^2}{\sigma_{\rm min}^2}p_i,2\rho^{2/3}\gamma^{2/3}p_i^{1/3}}\nonumber\\
    \implies \esqnorm{\w_i^\top\round{\I-\hat\U\hat\U^\top}} &\leq \min\set{\rho^2\frac{\gamma^2}{\sigma_{\rm min}^2},2\frac{\rho^{2/3}\gamma^{2/3}}{p_{i}^{2/3}}}\nonumber\\
    &\lesssim \min\set{\rho^6k^2\alpha/\sigma_{\rm min}^2, \rho^2k^{2/3}\alpha^{1/3}/p_{i}^{2/3}}.
\end{align}

Since we are aiming for $\enorm{\w_i^\top\round{\I-\hat\U\hat\U^\top}}=\eps\rho$ error, we need
\begin{align*}
    \alpha = \BigO{\max\set{\frac{\eps^2\sigma_{\rm min}^2}{\rho^4 k^2}, \frac{\eps^6p_{\rm min}^2}{k^2}}}
\end{align*}
and the sample complexity is
\[
   n \;=\; \tilde\Omega\Big(\,{dk^2} \min\set{\frac{\rho^4 k^2}{\eps^2\sigma_{\rm min}^2}, \frac{k^2}{\eps^6p_{\rm min}^2}} \,\Big).
\]

In the analysis above, we assume that each example $\beta_i$ is independently drawn. While this is true when $t_{L_1} =1 $, it is no longer the case when $t_{L_1}> 1$ where we have to break up the examples from each task into $t_{L_1}$ different estimators. Recall that $\hat\p$ is the vector of the empirical fractions of the examples that correspond to each linear regressor, and $\p$ is the population version of it. For given a pair of parameters $n_{L_1}, t_{L_1}$ let us define 
\[
\hat\p\sim \frac{1}{n_{L_1}t_{L_1}}\mathrm{multinomial} (n_{L_1}, \p)\cdot t_{L_1},
\] 
and independently
\[
\hat\p^*\sim \frac{1}{n_{L_1}t_{L_1}}\mathrm{multinomial} (n_{L_1}\cdot t_{L_1}, \p).
\]
Notice that $\hat\p$ corresponds to the setting where there are $n_{L_1}$ tasks with each task having $t_{L_1}$ examples, and $\hat\p^*$ corresponds to the setting where there are $n_{L_1}\cdot t_{L_1}$ tasks, each with $1$ example.
By Proposition~\ref{prop:mult-con-l1}, we know that when 
\[
n_{L_1}\ge \frac{2k\log(2/\delta)}{\alpha^2},
\]
$\norm{1}{\hat\p-\hat\p^*}\le \alpha$ with probability $1-\delta$. Hence if we denote the set $G=\{\beta_i\}_{i=1}^{n_{L_1}\cdot t_{L_1}}$ to be the set of the data coming from the model where each task has $t_{L_1}$ examples. There exists a distribution of set $L$, $E$ such that $G = (G'\setminus L')\cup E'$ with $|L|=|E|$, $G'$ has data from $n_{L_1}\cdot t_{L_1}$ independent tasks with $1$ example per task, and $|L| = |E|\le \alpha$ with probability $1-\delta$. Thus we have obtain a reduction from $t_{L_1}$ examples per task setting to the $1$
 example per task setting, in which case our algorithm receives a dataset with less than $2\alpha$ fraction of corruption with probability $\delta$. Since our previous proof applies to this setting, this concludes the proof of Lemma~\ref{lem:projection} and the final sample complexity is
    \begin{eqnarray}
        n_{L1} 
    \;=\; \tilde\Omega\Big(\,\frac{dk^2}{t_{L1}} \min\set{\frac{\rho^4 k^2}{\eps^2\sigma_{\rm min}^2}, \frac{k^2}{\eps^6p_{\rm min}^2}} \,+ k\min\set{\frac{\rho^8 k^4}{\eps^4\sigma_{\rm min}^4}, \frac{k^4}{\eps^{12}p_{\rm min}^4}}\Big)\;, 
    \label{eq:rpca_sample} 
    \end{eqnarray}

We are left to show  
$B=\Ocal(\rho^2d\log^2(nd/\delta))$, $\nu(k)=\Ocal(\rho^2 k)$, and the mean shift bound in Eq.~\eqref{eq:meanshift}.
\begin{proof}[Proof of Proposition~\ref{prop:pre-subspace}]
We first show $B=\Ocal((\rho^2d)\log^2(nd/\delta))$.
From~\cite[Proposition~A.1]{2020arXiv200208936K}, we have $\|\hat\beta_i\|_2^2 \leq (\rho^2 d)\log^2(nd/\delta)$ (i.e. event $\Ecal$ happens) with probability at least $1-\delta$. Using this, we have
\begin{align} 
    \enorm{\hat\beta_i\hat\beta_i^\top - \M} &\leq \esqnorm{\hat\beta_i} + \enorm{\M}\nonumber\\
    &\leq \rho^2{d}\log^2(nd/\delta) + 3\rho^2\nonumber\\
    &\leq 4\rho^2{d}\log^2(nd/\delta)\;,\label{eq:16}
\end{align}
for all $i\in[n]$ 
with probability at least $1-\delta$. 

Second, we bound the mean shift conditioned on event $\Ecal$
\begin{align}
    \enorm{\E{\hat\beta_i\hat\beta_i\middle| \Ecal} - \M} &= \max\limits_{\enorm{\vvec}=1}\abs{\E{z_\vvec \middle| \Ecal}},
\end{align}
where $z_\vvec \coloneqq \round{\vvec^\top\hat\beta_i}^2 - \vvec^\top\M\vvec$. 
The random variable $z_\vvec$ is centered with variance
\begin{align}
    \E{\round{\round{\vvec^\top\hat\beta_i}^2 - \vvec^\top\M\vvec}^2} 
    &\leq \E{\round{\vvec^\top\hat\beta_i}^4} - \round{\vvec^\top\M\vvec}^2\nonumber\\
    &= \BigO{\rho^4}\quad \text{ ($\vvec^\top\hat\beta_i$ is sub-exponential r.v.) },
\end{align}
Recall that $\M' \coloneqq \E{\hat\beta_i\hat\beta_i\middle| \Ecal}$, then using~\ref{fact:eps-tail}, we have
\begin{align}
    \enorm{\M' - \M} &= \max\limits_{\enorm{\vvec}=1}\abs{\E{z_\vvec \middle| \Ecal}}\nonumber\\
    &\leq \max\limits_{\enorm{\vvec}=1} \frac{{\mathbb E}[z_\vvec]+\sqrt{(1-\P(\Ecal))\cdot {\rm Var}(z_\vvec) } }{\Prob{\Ecal}}\nonumber\\
    &\leq \BigO{\rho^2\frac{\sqrt{\delta}}{\Prob{\Ecal}}}\nonumber\\
    &\leq \BigO{\rho^2\sqrt{\delta}}.\label{eq:20_}
\end{align}

Finally, we show $\nu^2(k)=\Ocal(k\rho^2)$.
For any symmetric real matrix $\A$ with $\rank(\A)=k$, and $\normF{\A}\le 1$, 
\begin{align}
    \E{\Tr{\A\round{\hat\beta_i\hat\beta_i^\top - \M'}}^2\middle|\Ecal} &= \E{\Tr{\A\round{\hat\beta_i\hat\beta_i^\top - \M + (\M-\M')}}^2\middle|\Ecal}\nonumber\\
    &= \E{\Tr{\A\round{\hat\beta_i\hat\beta_i^\top - \M}}^2\middle|\Ecal} + \Tr{\A\round{\M-\M'}}^2\nonumber\\
    &\leq \E{\Tr{\A\round{\hat\beta_i\hat\beta_i^\top - \M}}^2\middle|\Ecal} + \BigO{\rho^4}\label{eq:nu2-bound}
\end{align}
where the last inequality is obtained using Equation~\eqref{eq:20_}. Considering the first term in Equation~\eqref{eq:nu2-bound}, we get
\begin{align}
    &\E{\Tr{\A\round{\hat\beta_i\hat\beta_i^\top - \M}}^2\middle|\Ecal}\nonumber\\
    =& \E{\round{\sum_{j=1}^k\lambda_j\vvec_j\vvec_j^\top \round{\hat\beta_i\hat\beta_i^\top - \M}}^2 \middle|\Ecal}\nonumber\\
    =& \E{{\sum_{j,j'=1}^k\lambda_j\lambda_{j'} \round{\round{\hat\beta_i^\top\vvec_j}^2\round{\hat\beta_i^\top\vvec_{j'}}^2 - \vvec_j^\top\M\vvec_j\vvec_{j'}^\top\M\vvec_{j'}}} \middle|\Ecal}\nonumber\\
    \le& {{\sum_{j,j'=1}^k\lambda_j\lambda_{j'} \round{\sqrt{\E{\round{\hat\beta_i^\top\vvec_j}^4\middle|\Ecal}\E{\round{\hat\beta_i^\top\vvec_{j'}}^4\middle|\Ecal}} - \vvec_j^\top\M\vvec_j\vvec_{j'}^\top\M\vvec_{j'}}} } \text{ (Cauchy-Schwarz) }\nonumber\\
    \le& \BigO{\round{\sum_{j=1}^k \lambda_j}^2\rho^4}\qquad\text{ ($4$-th moment bound)}\nonumber\\
    \leq& \BigO{k\rho^4}.\label{eq:nu2-bound-wo-shift}
\end{align}
Using Equation~\eqref{eq:nu2-bound-wo-shift} in Equation~\eqref{eq:nu2-bound}, we get
\begin{align}
    \E{\Tr{\A\round{\hat\beta_i\hat\beta_i^\top - \M'}}^2\middle|\Ecal} &\leq \underbrace{\BigO{k\rho^4}}_{=\nu^2(k)}.\qedhere
\end{align}
\end{proof}

\subsection{Proof of Proposition~\ref{prop:rpca}}
\label{sec:proof-of-rpca}
The following main technical lemma 
guarantees that for any distribution $\X\sim \Pcal$ with a bounded support and a bounded second moment, the filtering algorithm we introduce in Algorithm~\ref{alg:filtering}. 

\begin{lemma}[Main Lemma for Algorithm~\ref{alg:filtering}]\label{lemma:main-filtering}
Let $\Pcal$ be a distribution over ${d\times d}$ PSD matrices with the property that, 
\[
    \Exp{\X\sim \Pcal}{\X} = \M\;,\quad  \|\X-\M\|_2 \le B\;,\quad\text{and}\quad\max_{\normF{\A}\le 1, \rank(\A)\le k}\Exp{\X\sim \Pcal}{\Tr{\A(\X-\M)}^2} \le \nu(k)^2.
\]
Let a set of $n$ random matrices $G = \set{\X_i\in \Rd{d\times d}}_{i\in [n]}$ where each $\X_i$ is independently drawn from $\Pcal$, and the at most $\alpha$ fraction is corrupted by an adversary such that 
the input dataset $S = (G\setminus L)\cup E$ with $|E|=|L|\le \alpha n$, $L\subset G$.  
There exists a numerical constant $c>0$ such that for any $0< \alpha <c$,
if $n =  \Omega((dk^2+(B/\nu)\sqrt{k\alpha})\log(d/(\delta\alpha))/\alpha)$, 
Algorithm~\ref{alg:filtering} 
outputs a dataset $S'\subseteq S$  satisfying the following for 
$\hat\M=\frac{1}{\abs{{S}'}}\sum_{\X_i\in {S}'} \X_i$: 
\begin{enumerate}
\item for the top-$k$ singular vectors $\hat{\U}\in \Rd{d \times k}$ of $\hat{\M}$, 
\[
    \Tr{\hat{\U}^\top \round{\hat\M -\M}\hat{\U}}\;\; \leq \;\;  48\alpha\Tr{{\hat\U}^\top {\M}{\hat\U}} +102\nu\sqrt{k\alpha }\;. 
\]
\item for all rank-$k$ semi-orthogonal matrices $\V\in \Rd{d\times k}$, we have
\[
    \Tr{\V^\top\round{\hat\M-\M}\V}\;\;\ge
     \;\; -10\alpha\Tr{\V^\top{\M}\V}-8\nu\sqrt{k\alpha}\;. 
\]
\end{enumerate}
\end{lemma}
We provide a proof in Section~\ref{sec:proof_main-filering}. 

The proof of Proposition~\ref{prop:rpca} is straightforward given Lemma~\ref{lemma:main-filtering}. For the first claim, note that,
\begin{align*}
\Tr{\hat\U^\top \SIGMA \hat\U} &\ge \Tr{\hat\U^\top \hat\SIGMA \hat\U}- 48\alpha\Tr{{\cal P}_k(\SIGMA)} - 102\nu\sqrt{k\alpha}\qquad\text{ (Using Lemma~\ref{lemma:main-filtering}, part 1)}\\
&\ge  \Tr{\U^\top \hat\SIGMA \U} - 48\alpha\Tr{{\cal P}_k(\SIGMA)} - 102\nu\sqrt{k\alpha}\qquad\text{ (Property of SVD)}\\
&\ge \Tr{{\cal P}_k(\SIGMA)} - 58\alpha\Tr{{\cal P}_k(\SIGMA)} - 110\nu\sqrt{k\alpha}\qquad \text{ (Using Lemma~\ref{lemma:main-filtering}, part 2)}.
\end{align*}
For the second claim, since $\SIGMA \succeq \hat\U\hat\U^\top \SIGMA  \hat\U\hat \U^\top$, we have
\begin{align*}
 \norm{*}{  \SIGMA -  \hat\U\hat\U^\top \SIGMA  \hat\U\hat \U^\top  } &= \Tr{\SIGMA -  \hat\U\hat\U^\top \SIGMA  \hat\U\hat \U^\top  }\\
 &\le \Tr{\SIGMA} - (1-58\alpha)\Tr{{\cal P}_k(\SIGMA)} + 110\nu\sqrt{k\alpha}\qquad\text{(From the first claim).}\\
 & = \Tr{\SIGMA-{\cal P}_k(\SIGMA)} + 58\alpha\Tr{{\cal P}_k(\SIGMA)} + 110\nu\sqrt{k\alpha}\\
& =  \norm{*}{\SIGMA-{\cal P}_k(\SIGMA)} + 58\alpha\norm{*}{{\cal P}_k(\SIGMA)} + 110\nu\sqrt{k\alpha}.
\end{align*}
Similarly, we also get
\begin{align}
    \norm{*}{  {\cal P}_k(\SIGMA) -  \hat\U\hat\U^\top \SIGMA  \hat\U\hat \U^\top  } &\leq 58\alpha\norm{*}{{\cal P}_k(\SIGMA)} + 110\nu\sqrt{k\alpha}
\end{align}

\subsection{Proof of the main analysis of Algorithm~\ref{alg:filtering} in Lemma~\ref{lemma:main-filtering}}
\label{sec:proof_main-filering}

The proof of Lemma~\ref{lemma:main-filtering} is divided into two parts, for each statement of the lemma. Both parts are proven under the following good event. 
We provide a proof of this lemma in Section~\ref{sec:proof-of-goodevent}.

\begin{lemma}
\label{lem:goodevent}
Under the hypotheses of Lemma~\ref{lemma:main-filtering}, 
when 
$n = \Omega((dk^2 + \frac{B}{\nu}\sqrt{k\eps})\log(d/(\eps\tilde\delta))/\eps)$ with probability $1-\tilde\delta$, the following events happen for all semi-orthogonal matrices $\V\in \Rd{d\times k} s.t. \V^\top \V = \I_k$, 
\begin{enumerate}
\item  There exists $G_{\V}\subset G$ such that
\begin{enumerate}
    \item $\abs{G_\V} \geq (1-\eps)n$,
    \item $\abs{\frac{1}{|G_\V|}\Tr{\sum_{\X_i\in G_\V}{\round{\V^\top (\X_i-\M)\V}}}} \le 1.01\nu\sqrt{k\eps}$, and
    \item $\frac{1}{|G_\V|}{\Tr{\sum_{\X_i\in G_\V}{\round{\V^\top (\X_i-\M)\V}}}^2\leq 6.01 k\nu^2}$,
\end{enumerate}
\item  $\begin{aligned}
    \abs{\frac{1}{n}\sum_{\X_i\in G}\Tr{\V^\top(\X_i-\M)\V}} &\le \nu\sqrt{k\eps},
\end{aligned}$
\item All subset $T \subset G$ such that $|T|\le \eps n$ satisfies
\begin{equation*}
    \sum\limits_{\X_i\in T} \Tr{\V^\top(\X_i-\M)\V}\le 7n\nu\sqrt{k\eps}+n\eps\Tr{\V^\top\M\V}.
\end{equation*}
\end{enumerate}
\end{lemma} 
We provide the proof in Section~\ref{sec:proof-of-goodevent}.



\subsubsection{Part 1 of Lemma~\ref{lemma:main-filtering}}

\begin{proposition}[Correctness of Algorithm~\ref{alg:basic_filter}]\label{prop:basic_filter} 
For a set $G$ of $n$ uncorrupted matrices defined as in the hypotheses of Lemma~\ref{lem:goodevent} and  
for some $\eps\geq\alpha>0$, 
suppose the that set $S$ input to Algorithm~\ref{alg:basic_filter} satisfies 
the following: 
$S=(G \setminus L) \cup E$ with $L\subset G$, 
$|E|\leq \alpha|G|$, 
$|L|\leq 9\alpha|G|$, 
$|S|\leq|G|$, and $E\cap G=\emptyset$. 
If the events in Lemma~\ref{lem:goodevent} hold, then  
a single call of  Algorithm~\ref{alg:basic_filter} outputs a set $S'\subseteq S$ achieving one of the following two guarantees. 
\begin{enumerate}
\item If Algorithm~\ref{alg:basic_filter} returns $S' = S$ (unchanged), then
\begin{equation}
    \Tr{\hat{\U}^\top \round{\frac{1}{\abs{{S}'}}\sum_{\X_i\in {S}'} \X_i-\M}\hat{\U}}\;\; \leq \;\;  48\alpha\Tr{\hat{\U}^\top {\M}\hat{\U}} +102\nu\sqrt{k\alpha}
\end{equation}
\item If Algorithm~\ref{alg:basic_filter} returns $S'\subset S$, then there exist two sets $L' \supseteq L$ and $E'\subseteq E$ such that $S'=(G\setminus L')\cup E'$ and  $\E{2\abs{L'}+\abs{E'}}\le 2|L|+|E|$. 
\end{enumerate}
where $\hat\U$ is the top-$k$ singular matrix of $\frac{1}{\abs{{S}'}}\sum_{\X_i\in {S}'}\X_i$.
\end{proposition}
We provide the proof in Section~\ref{sec:proof_basic_filter}.
We are left to show that 
when $n$ is sufficiently large, then Algorithm~\ref{alg:filtering} terminates before removing too many points, with probability at least $1-\delta$. 
To get the logarithmic dependence on $\delta$,we divide the meta-training dataset $\Dcal_{L1}$ into $\log(1/\delta)$ partitions of an equal size, and apply the same routine of  Algorithm~\ref{alg:filtering}. 

We show that each run of Algorithm~\ref{alg:filtering} succeeds with a strictly positive probability, and hence 
one of them is guaranteed to succeed with probability at least $1-\delta$. Further, we have a simple way of choosing the successful run; we select the one that outputs the largest set $S'$. 
Precisely, we call a routine successful if $|S'|\geq (1-8\alpha)|G|$.

First, we need to show that the conditions of Proposition~\ref{prop:basic_filter} hold, throughout the   iterations. However, the condition that $|L|\leq 9\alpha|G|$ might be violated by chance, as we only have guarantee in expectation. We thus bound the probability that there exists a sub-routine of 1-D filtering that results in $|L'|\geq 8\alpha |G|$. 
The proof is similar to the one in~\cite{diakonikolas2017beingrobust}. Let $L_i$ denote the removed set of uncorrupted points in $S_i=(G\setminus L_i) \cup E_i$.
Notice that this event implies $|L_T|\geq 8\alpha |G|$ as $L_i$'s are a monotonically increasing sequence of sets. From Markov's inequality, we have  
$\P(|L_T|\geq 8\alpha |G|) \leq \E{|L_T|}/(8\alpha |G|) \leq 1/6$, 
where in the last inequality we used the fact that $\E{|L_T|}\leq  (1/2)\E{2|L_T|+|E_T|}\leq (1/2)(2|L_0|+|E_0|) \leq (3/2)\alpha|G|$. Hence, with probability at least $5/6$ one run succeeds (taking a union bound with Lemma~\ref{lem:goodevent} for the good events with a choice of $\tilde\delta=1/6$). 
Out of $\log_{6}(2/\delta)$ runs, one succeeds with probability at least $1-\delta/2$. 

\subsubsection{Part 2 of  Lemma~\ref{lemma:main-filtering}}

There exists a set $T\subset G$ such that $S' \supseteq T$, and $\abs{T}\geq (1-9\alpha)\abs{G}$.  Since $S'$ contains a good fraction of points from $G$, then for any semi-orthogonal matrix $\V\in\Rd{d\times k}$, we have
\begin{align}
    &\sum_{\X_i\in S'}\Tr{\V^\top\X_i\V}\nonumber\\
    &\geq \sum_{\X_i\in G}\Tr{\V^\top\X_i\V} - \sum_{\X_i\in G\setminus T}\Tr{\V^\top\X_i\V}\nonumber\\
    &\geq \sum_{\X_i\in G}\Tr{\V^\top\X_i\V} - \round{\sum_{\X_i\in G\setminus T}\Tr{\V^\top\M\V} + n\alpha\Tr{\V^\top\M\V} + 7n\nu\sqrt{k\alpha}}\nonumber\\
    &\omit\hfill\text{(Using Lemma~\ref{lem:goodevent}, part 3)}\nonumber\\
    &\geq \sum_{\X_i\in G}\Tr{\V^\top\M\V} - n\nu\sqrt{k\alpha} - \round{\sum_{\X_i\in G\setminus T}\Tr{\V^\top\M\V} + n\alpha\Tr{\V^\top\M\V} + 7n\nu\sqrt{k\alpha}}\nonumber\\
    &\omit\hfill\text{(Using Lemma~\ref{lem:goodevent}, part 2)}\nonumber\\ 
    &= \sum_{\X_i\in T}\Tr{\V^\top\M\V} - 8n\nu\sqrt{k\alpha} - n\alpha\Tr{\V^\top\M\V}\nonumber\\
    &\geq n\round{1-10\alpha}\Tr{\V^\top\M\V} - 8n\nu\sqrt{k\alpha} \qquad\text{($\because \abs{T}\geq (1-9\alpha)n$)}\nonumber
\end{align}
\begin{align}
    \implies n\Tr{\V^\top\hat\M\V} &\geq \abs{S'}\Tr{\V^\top\hat\M\V} \geq n\round{1-10\alpha}\Tr{\V^\top\M\V} - 8n\nu\sqrt{k\alpha}\nonumber\\
    \text{or, }\Tr{\V^\top\round{\hat\M - \M}\V} &\geq -10\alpha\Tr{\V^\top\M\V} - 8\nu\sqrt{k\alpha}\nonumber
\end{align}
when the good events of Lemma~\ref{lem:goodevent} hold, which happens if $n = \Omega\round{\inv{\eps}\round{k+\frac{B}{\nu}\sqrt{k\eps}}\log\frac{d}{\delta}}$ with probability at least $1-\delta$.

\subsection{Proof of Proposition~\ref{prop:basic_filter}}\label{sec:proof_basic_filter}

\begin{enumerate}
    \item To show the first part of Proposition~\ref{prop:basic_filter}, notice that by Lemma~\ref{lem:goodevent}, part 1, there exists a subset $G_\V\subset G$ such that 
    \begin{align*}
    \abs{G_{\hat{\U}}}&\ge (1-\alpha)n\\
    \abs{\frac{1}{\abs{G_{\hat{\U}}}}\sum_{\X_i\in G_{\hat{\U}}}\Tr{\hat{\U}^\top(\X_i-\M)\hat{\U}}}&\le 1.01\nu\sqrt{k\alpha}\\
    \frac{1}{\abs{G_{\hat\U}}}\Tr{\sum_{\X_i\in G_\V}{\round{\V^\top (\X_i-\M)\V}}}^2&\leq 6.01 k\nu^2.
    \end{align*}
    For the input to the First-Filter algorithm, $S_0 = (G\setminus L)\cup E$, we can reclassify the sample in $G\setminus G_{\hat\U}$ to be the error and have $S_0 = (G_{\hat\U}\setminus L)\cup E'$ where $|L|\le 9\alpha n$ and $|E'|\le |E|+\alpha n\le 2\alpha n$. Hence the output of the First-Filter algorithm satisfies
    \begin{align}
        \abs{\inv{\abs{S_G}}\sum\limits_{\X_i\in S_G}\Tr{\hat\U^\top\X_i\hat\U} - \Tr{\hat\U^\top\M\hat\U} } &\leq 54\nu\sqrt{k\alpha},
    \end{align}
    using Proposition~\ref{prop:diakanikolas-high-prob}, and from \textbf{if} condition of Algorithm~\ref{alg:basic_filter}, we have
    \begin{align}
        \inv{n}\sumi{1}{n}\Tr{\hat\U^\top\X_i\hat\U} - \inv{\abs{S_G}}\sum\limits_{\X_i\in S_G}\Tr{\hat\U^\top\X_i\hat\U} &\leq 48\round{\nu\sqrt{k\alpha}+\alpha\Tr{\hat\U^\top\M\hat\U}}.
    \end{align}
    Combining the above two inequalities, we get
    \begin{align}
        &\inv{n}\sumi{1}{n}\Tr{\hat\U^\top\X_i\hat\U} - \Tr{\hat\U^\top\M\hat\U}\nonumber\\
        =& \inv{n}\sumi{1}{n}\Tr{\hat\U^\top\X_i\hat\U} - \inv{\abs{S_G}}\sum\limits_{\X_i\in S_G}\Tr{\hat\U^\top\X_i\hat\U}\nonumber\\
        &\qquad\qquad+ \inv{\abs{S_G}}\sum\limits_{\X_i\in S_G}\Tr{\hat\U^\top\X_i\hat\U} - \Tr{\hat\U^\top\M\hat\U}\nonumber\\
        \leq& \inv{n}\sumi{1}{n}\Tr{\hat\U^\top\X_i\hat\U} - \inv{\abs{S_G}}\sum\limits_{\X_i\in S_G}\Tr{\hat\U^\top\X_i\hat\U}\nonumber\\
        &\qquad\qquad+ \abs{\inv{\abs{S_G}}\sum\limits_{\X_i\in S_G}\Tr{\hat\U^\top\X_i\hat\U} - \Tr{\hat\U^\top\M\hat\U}}\nonumber\\
        \leq& 102\nu\sqrt{k\alpha} + 48\alpha\Tr{\hat\U^\top\M\hat\U}.
    \end{align}
    \item We use $x_i$ to denote $\Tr{\U^\top\X_i\U}$ and $\mu^P$ to denote $\Tr{\U^\top \M\U}$.
    Our goal is to show that our Algorithm removes much less good points, from the set $G\cap(S\setminus S_G)$.
    Notice that, 
    \begin{align}
        \sum_{\substack{x_i\in S\setminus S_G,\\x_i\ge \mu^{S_G} }}\round{x_i-\mu^{S_G}} &= \sum_{x_i\in S\setminus S_G}\round{x_i-\mu^{S_G}} - \sum_{\substack{x_i\in S\setminus S_G,\\x_i<\mu^{S_G}}}\round{x_i-\mu^{S_G}}\nonumber\\
        &\geq \sum_{x_i\in S\setminus S_G}\round{x_i-\mu^{S_G}}\nonumber\\
        &= \sum_{x_i\in S}\round{x_i-\mu^{S_G}} - \sum_{x_i\in S_G}\round{x_i-\mu^{S_G}}\nonumber\\
        &= \sum_{x_i\in S}\round{x_i-\mu^{S_G}}\nonumber\\
        &\ge 48 n \round{\alpha\mu^{S_G} + \nu\sqrt{k\alpha}}\qquad\text{(from the if clause in Algorithm~\ref{alg:basic_filter})}\nonumber\\
        &\ge 48 n \round{\alpha\mu^P + \round{1-18\alpha}\nu\sqrt{k\alpha}}\nonumber\\
        &\geq 48 n \round{\alpha\mu^P + \inv{2}\nu\sqrt{k\alpha}}
    \end{align}
    where the last inequality is using $\alpha\leq \nicefrac{1}{36}$, and the second last from the guarantee from the First-Filter algorithm.
    (by the if clause in the algorithm).

    On the other hand, note that,
    \begin{align}
        \abs{\set{x_i\in G\cap(S\setminus S_G)\given x_i\ge \mu^{S_G}}}&\le \abs{(S\setminus S_G)} \leq n\alpha.\label{eq:30}
    \end{align}
    Therefore,
    \begin{align}
        \sum_{\substack{x_i\in G\cap(S\setminus S_G),\\x_i\ge \mu^{S_G}}}\round{x_i-\mu^{S_G}} &\leq \sum_{\substack{x_i\in G\cap(S\setminus S_G),\\x_i\ge \mu^{S_G}}}\round{x_i-\mu^P} + \abs{S\setminus S_G}\abs{\mu^P - \mu^{S_G}} \nonumber\\
        &\leq \sum_{\substack{x_i\in G\cap(S\setminus S_G),\\x_i\ge \mu^{S_G}}}\round{x_i-\mu^{P}} + n\alpha\abs{\mu^P-\mu^{S_G}}\nonumber\\
        &\leq n\round{7\nu\sqrt{k\alpha} + \alpha\mu^P + 18\nu\alpha\sqrt{k\alpha}}\nonumber\\
        &\omit\hfill\text{(Lemma~\ref{lem:goodevent}, part 3 and Proposition~\ref{prop:diakanikolas-high-prob})}\\
        &\leq n\round{(7+18\alpha)\nu\sqrt{k\alpha} + \alpha\mu^P}\nonumber\\
        &\leq n\round{8\nu\sqrt{k\alpha} + \alpha\mu^P}\nonumber\qquad\text{($\because \alpha\leq \nicefrac{1}{36}\leq \nicefrac{1}{18}$.)}\\
        &\leq n\round{8\nu\sqrt{k\alpha} + 16\alpha\mu^P}\nonumber\\
        &= 16n\round{\alpha\mu^P + \inv{2}\nu\sqrt{k\alpha}}
    \end{align}
    where the last inequality follows by using Lemma~\ref{lem:goodevent}, part 3, in conjunction with Equation~\eqref{eq:30}, and using the guarantee provided in Proposition~\ref{prop:diakanikolas-high-prob}. Hence we have shown that
    \begin{align}
    \sum_{\substack{x_i\in S\setminus S_G,\\x_i\ge \mu^{S_G}}}\round{x_i-\mu^{S_G}} \ge 3\sum_{\substack{x_i\in G\cap(S\setminus S_G),\\x_i\ge \mu^{S_G}}}\round{x_i-\mu^{S_G}}.\label{eq:32}
    \end{align}
    Applying Fact~\ref{fact:mean-filter}, we get
    \begin{align}
        \E{\abs{S\setminus S'}} &= \frac{\sum\limits_{\substack{i\in S\setminus S_G,\\x_i\geq \mu^{S_G}}}\round{x_i-\mu^{S_G}} }{\max\set{x_i-\mu^{S_G}}_{i\in S\setminus S_G}}\\
        \text{and, }\E{\abs{G\cap(S\setminus S')}} &= \frac{\sum\limits_{\substack{i\in G\cap(S\setminus S_G),\\x_i\geq \mu^{S_G}}}\round{x_i-\mu^{S_G}} }{\max\set{x_i-\mu^{S_G}}_{i\in S\setminus S_G}}.
    \end{align}
    This combined with Equation~\eqref{eq:32} gives
    \begin{align}
        3\E{\abs{G\cap\round{S\setminus S'}}} &\leq \E{\abs{S\setminus S'}}\nonumber\\
        &= \E{\abs{\round{L'\setminus L}\cup\round{E\setminus E'}}}\nonumber\\
        &\leq \E{\abs{L'\setminus L}} + \E{\abs{E\setminus E'}}.
    \end{align}
    Since $\abs{L'\setminus L} = \abs{G\cap(S\setminus S')}$, we finally have
    \begin{align}
        \E{\abs{G\cap(S\setminus S')}} &\leq \inv{2}\E{\abs{E\setminus E'}},\nonumber\\
        \implies \E{2\abs{L'} + \abs{E'}} &\leq 2\abs{L} + \abs{E}.
    \end{align}
    Hence, we can run the filter on $\set{\X_i\in S\setminus S_G}$, and guarantee that on every application of Algorithm~\ref{alg:basic_filter}, we remove at-least $2$ times more corrupted points in expectation than good points.
\end{enumerate}
This completes the proof.

\begin{fact}[Filter based on mean]\label{fact:mean-filter}
Assuming that $a\le x_1\le\ldots\le x_n\le b$, $t\sim \Ucal\squarebrack{a,b}$, then
\begin{equation*}
    \E{\sum\limits_{i=1}^n \indicator{}{x_i>t}} = \sum\limits_{i=1}^n (x_i - a)/(b-a)
\end{equation*}
\end{fact}

\subsection{Proof of Lemma~\ref{lem:goodevent}}\label{sec:proof-of-goodevent}
\subsubsection{Proof of Lemma~\ref{lem:goodevent}, part 1}\label{lem:all-good-set}
The basic idea of the proof is the following. First we construct a $\eps$-net on semi-orthogonal matrices $\V\in \Rd{d\times k}$, and argue that for each semi-orthogonal matrices $\V$ on the net, the set $\set{\X_i\in G \given \Tr{\V^\top (\X_i-\M)\V}\le \Theta(\nu\sqrt{k/\eps})}$ satisfies the three conditions in the lemma. Then,
for each matrix $\V'$ that is not on the set, we argue there exists a $G_\V$ with $\V$ on the $\eps$-net which satisfies the three conditions under $V'$ in the lemma. We show the three conditions for a matrix on the net as follows.
\begin{lemma}\label{lem:filtering_step1}
Let $G = \set{x_i \sim \Pcal}_{i=1}^n$ where $\mu^P$ is the mean, and $\sigma^2_P$ is the variance of a real distribution $\Pcal$. For any real number $0<\eps\le 1/18$, define set $T \coloneqq \set{x_i \in G \given \abs{x_i - \mu^P}\le 3\sigma_P/\sqrt{\eps}}$, and $\mu^T \coloneqq \frac{1}{\abs{T}}\sum_{x_i\in T}{x_i}$. Then with probability at least $1-3\exp\round{-\Theta\round{n\eps}}$,
\begin{enumerate}
    \item $\abs{T} \geq (1-\eps)n$,
    \item $\abs{\mu^T-\mu^P}\le \sigma_P\sqrt{\eps}$, and
    \item $\frac{1}{|T|}{{\sum_{x_i\in T}{\round{x_i-\mu^P}^2}}\leq 2\sigma_P^2}$.
\end{enumerate}
\end{lemma}
We provide a proof in Section~\ref{sec:filtering_step1_proof}. 
\begin{proposition}[Covering Number for Low-Rank Matrices, Lemma 3.1 in~\cite{candes2011tight}]\label{prop:low-rank-eps-net}
Let $S_r \coloneqq \set{\X\in \Rd{n_1\times n_2}: \rank(\X) \le r, \normF{\X}=1}$. Then there exists an $\eps$-net $\bar{S}_{r,\eps} \subset S_r$ with respect to Frobenius norm obeying
\[
    \abs{\bar{S}_{r,\eps}}\le (9/\eps)^{(n_1+n_2+1)r}.
\]
\end{proposition}
With a union bound over all the elements of the net, we show that for each matrix $\V$ in the net, set $\set{\X_i\in G \given \Tr{\V^\top (\X_i-\M)\V}\le \Theta(\nu\sqrt{k/\eps})}$ satisfies the three conditions in the lemma. The following technical proposition will helps us connect an arbitrary semi-orthogonal matrix $V'$ to the net.
\begin{proposition}\label{prop:f2-norm-bound}
With probability $1-\delta$, all subset $T \subset G$ such that $|T|\ge (1-\eps)n$ satisfies $\frac{1}{|T|}\sum_{\X_i\in T} \normF{\X_i-\M}^2\le \nu^2d^2/\delta(1-\eps)$
\end{proposition}
\begin{proof}
For each $i, j\in [d]$, define matrix $\mathbf{E}^{i,j}\in \Rd{d \times d}$ such that 
\[
E^{i,j}_{i',j'}=
\begin{cases}
1 & \text{if } i'=i \text{ and } j'=j,\\
0& \text{otherwise}.
\end{cases}
\]
Then,
\begin{align}
    \Exp{\X\sim \Pcal}{\normF{\X-\M}^2} &= \sum_{i,j\in [d]}\Exp{\X\sim \Pcal}{\Tr{\mathbf{E}_{i,j}(\X-\M)}^2}\nonumber\\
    &\le \nu^2d^2\nonumber,
\end{align}
where the last inequality follows from the assumption in Lemma~\ref{lem:goodevent}. By Markov's inequality,
\begin{align}
\Prob{\frac{1}{n}\sum_{i=1}^n \normF{\X_i-\M}^2\ge \nu^2d^2/\delta} &\le \frac{\delta}{\nu^2d^2}\E{\frac{1}{n}\sum_{i=1}^n \normF{\X_i-\M}^2}\nonumber\\
&\le \delta.    
\end{align}
Then for any subset $T\subset G$, and $|T|\ge (1-\eps)n$, 
\begin{align*}
\frac{1}{|T|}\sum_{\X_i\in T} \normF{\X_i-\M}^2 \le \frac{1}{|T|}\sum_{\X_i\in G} \normF{\X_i-\M}^2 &\le \nu^2d^2/\delta(1-\eps).\qedhere
\end{align*}
\end{proof}

We show Lemma~\ref{lem:goodevent}, part 1 as follows.
By Proposition~\ref{prop:low-rank-eps-net}, there exists an $\delta'$-net of size $(9\sqrt{k}/\delta')^{(k+d+1)k}$ over the set $S_k = \{\V\in \Rd{d\times k}: \rank(\V) \le k, \normF{\V}=\sqrt{k}\}$. Since the set of rank $k$ projection matrices $P_k \coloneqq \{\V\in \Rd{d\times k}, \V^\top \V = \I_k\}\subset S_k$, there exists a $\delta'$-net $\bar{P}_k$ of size $\round{18\sqrt{k}/\delta'}^{(k+d+1)k} = \exp\round{\Theta\round{dk\log\round{k/\delta'}}}$ of the set $P_k$. Fixing a projection matrix $\V$, the distribution of $\V^\top \X_i\V$ satisfies $\E{\Tr{\V^\top \X_i\V}} = \Tr{\V^\top \M\V}$ and
\[
\E{\Tr{{\V^\top (\X_i-\M)\V}}^2} = k^2\E{\Tr{\frac{1}{k}\V\V^\top(\X_i-\M)}^2}\le k\nu^2.
\]
Thus, we can apply Lemma~\ref{lem:filtering_step1} with a union bound to show that, given that 
\begin{align*}
n = \Omega(dk^2\log(k/\delta\delta')/\eps),
\end{align*}
with probability $1-\delta$, for each projection matrix $\V\in \bar{P}_k$, there exist a subset of random matrices $G_{\V}\subset G$ such that $G_\V  \coloneqq \set{X_i \in G \given \abs{\Tr{\V^\top\round{\X_i - \M}\V}}\leq 3\nu\sqrt{k/\eps}}$, and it satisfies
\begin{enumerate}
    \item $\abs{G_\V} \geq (1-\eps)n$,
    \item $\abs{\inv{\abs{G_\V}}\sum_{\X_i\in G_\V}{\Tr{\V^\top (\X_i-\M)\V}}} \le \nu\sqrt{k\eps}$, and
    \item ${{\inv{\abs{G_\V}}\sum_{\X_i\in G_\V}{\Tr{\V^\top (\X_i-\M)\V}^{2}}}\leq 2k\nu^2}$.
\end{enumerate}
From now on, let us set 
\begin{align*}
\delta' = \sqrt{\eps\delta}/100d.
\end{align*} 
The remaining argument conditions on that the event in Proposition~\ref{prop:f2-norm-bound}, which happens with probability $1-\delta$, namely, for all $G\subset S$ such that $\abs{G}\ge (1-\eps)n$, 
\begin{align}
\frac{1}{\abs{G}}\sum_{\X_i\in G} \normF{\X_i-\M}^2\le \nu^2d^2/\delta(1-\eps)\label{eqn:f2-norm-bound}.
\end{align}
The following proposition deal with the projections that are not in $\bar{P}_k$.

\begin{proposition}
For any arbitrary projection matrix $\V'$, which may not be an element of $\bar{P}_k$, define projection matrix $\V\in \bar{P}_k$ such that $\normF{\V-\V'}\le \delta$. Then $G_\V$ is still an $\eps$-good set for projection $\V'$, namely
\begin{enumerate}
\item $\abs{G_\V} \geq (1-\eps)n$,
\item $\abs{\frac{1}{\abs{G_\V}}\sum_{\X_i\in G_\V}{\Tr{{\V'}^\top (\X_i-\M){\V'}}}} \le 1.01\nu\sqrt{k\eps}$, and
\item ${{\frac{1}{\abs{G_\V}}\sum_{\X_i\in G_\V}{\Tr{{\V'}^\top (\X_i-\M){\V'}}^{2}}}\leq 6.01k\nu^2}$.
\end{enumerate}
\end{proposition}
\begin{proof}\leavevmode
\begin{enumerate}
\item $\abs{G_\V} \geq (1-\eps)n$ holds trivially.
\item Observe that
\begin{align*}
&\abs{\frac{1}{\abs{G_\V}}\sum_{\X_i\in G_\V}{\Tr{{\V'}^\top (\X_i-\M){\V'}}}}\\
&\le \frac{1}{\abs{G_\V}}\vast(\Tr{\sum_{\X_i\in G_\V}{ \round{{\V'}^\top (\X_i-\M){\round{\V'-\V}}}}} \\
&\qquad + \Tr{\sum_{\X_i\in G_\V}{\round{{\round{\V'-\V}}^\top (\X_i-\M){\V}}}} + \Tr{\sum_{\X_i\in G_\V}{\round{\V^\top (\X_i-\M)\V}}}\vast)\\
&\le \frac{2}{|G_\V|}\sum_{\X_i\in G_\V}\normF{\X_i-\M}\delta' +\nu\sqrt{k\eps}\\
&\le 2\delta' \sqrt{\frac{1}{|G_\V|}\sum_{\X_i\in G_\V}\normF{\X_i-\M}^2}  +\nu\sqrt{k\eps}\\
&\le 1.01\nu\sqrt{k\eps}
\end{align*}
\item
Notice that 
\begin{align}
&\frac{1}{|G_\V|}\sum_{\X_i\in G_\V}\Tr{\round{{\V'}^\top (\X_i-\M){\V'}}}^2\nonumber\\
&= \frac{1}{|G_\V|}\sum_{\X_i\in G_\V}\Bigg(\Tr{ \round{\V'^\top (\X_i-\M){\round{\V'-\V}}}} + \Tr{\round{{\round{\V'-\V}}^\top (\X_i-\M){\V}}}\nonumber\\
&\qquad\qquad\qquad\qquad\qquad+ \Tr{\round{\V^\top (\X_i-\M)\V}}\Bigg)^2\nonumber\\
&\le \frac{3}{|G_\V|}\sum_{\X_i\in G_\V}\Bigg(\Tr{ \round{{\V'}^\top (\X_i-\M)\round{\V'-\V}}}^2 \nonumber\\
&\qquad\qquad\qquad\qquad+ \Tr{\round{{\round{\V'-\V}}^\top (\X_i-\M){\V}}}^2 + \Tr{\round{{\V}^\top (\X_i-\M){\V}}}^2\Bigg)\label{eqn:eps-good-cov-bound}
\end{align}
Notice that by Cauchy-Schwarz, 
\begin{align*}
\Tr{\round{{\V'}^\top (\X_i-\M){\round{\V'-\V}}}}^2\le \normF{\round{\V'-\V}^\top \V'}^2\normF{\X_i-\M}^2\le \delta'^2\normF{\X_i-\M}^2,
\end{align*}
and likewise
\begin{align*}
\Tr{\round{{\round{\V'-\V}}^\top (\X_i-\M)\V}}^2\le \delta'^2\normF{\X_i-\M}^2.
\end{align*}
By Equation~\eqref{eqn:f2-norm-bound}, Equation~\eqref{eqn:eps-good-cov-bound} is bounded by
\begin{align*}
    6\nu^2k+ \nu^2\eps/10000 &\le  6.01k\nu^2.\qedhere
\end{align*}
\end{enumerate}
\end{proof}
This completes all the proofs.

\subsubsection{Proof of Lemma~\ref{lem:goodevent}, part 2}\label{fact:trace_concentration}

\begin{proof}[Proof of Lemma~\ref{lem:goodevent}, part 2]
Since $\enorm{\X-\M}\leq B$, and $\max_{\|\A\|_*\leq 1}\E{\round{\Tr{\A\round{\X_i-\M}}}^2}\leq \nu^2$ for all $i\in G$,  therefore, from Bernstein inequality, we have
\begin{align}
    \Prob{\enorm{\inv{n}\sum_{i\in G}\X_i - \M} \geq \nu\sqrt{\eps/k}} &\leq 2d\exp\squarebrack{\frac{-n^2\nu^2\eps/(2k)}{ nd\nu^2 + Bn\nu\sqrt{\eps/k}/3 }}\leq \delta
\end{align}
Define matrix $\A^{(1,\vvec)}, \A^{(2,\vvec)},\ldots, \A^{(d,\vvec)}$ such that $\A^{(i,\vvec)}_{i,*} = \vvec$, and the other rows of $\A^{(i)}$ are $0$.
\begin{align*}
\enorm{\E{\sumi{1}{n}(\X_i-\M)^2}} &= \max_{\|\vvec\|_2=1}\sumi{1}{n}\E{\vvec^\top(\X_i-\M)(\X_i-\M)\vvec}\\
&= \max_{\|\vvec\|_2=1}\sumi{1}{n}\sum_{j=1}^d\E{\Tr{\A^{(j, \vvec)}(\X_i-\M)}^2}\\
&\le nd\nu^2
\end{align*}

if $n\geq
\frac{2}{\eps}\round{dk + \frac{B}{3\nu}\sqrt{\eps k}}\log\frac{2d}{\delta}$. Therefore for any semi-orthogonal matrix $\U\in\Rd{d\times k}$,
\begin{align}
    \abs{\frac{1}{n}\sum_{\X_i\in G}\Tr{\V^\top(\X_i-\M)\V}} &\leq \abs{\frac{1}{n}\sum_{\X_i\in G}\Tr{\Pcal_k(\X_i-\M)}}\nonumber\\
    &\leq k\enorm{\inv{n}\sum_{i\in G}\X_i - \M}\nonumber\\
    &\leq \nu\sqrt{k\eps},
\end{align}
if $n = \Omega\round{\round{\inv{\eps}\round{dk+\frac{B}{\nu}\sqrt{k\eps}}}\log\frac{d}{\delta}}$.
\end{proof}

\subsubsection{Proof of Lemma~\ref{lem:goodevent}, part 3}\label{sec:proof-resillence}
Here we show that for any semi-orthogonal matrix $\V\in\Rd{d\times k}$ on the net $\bar{P}_k$, and any subset $T$, $\sum_{\X_i\in T}\Tr{\V^\top (\X_i-\M)\V} \le n\nu\sqrt{k/\eps}$ condition on the event defined in Lemma~\ref{lem:goodevent}, part 1 happens.

Denote 
\begin{align*}
G_L &= \set{ \X_i\in G \given \Tr{\V^\top(\X_i-\M)\V}<-3\nu\sqrt{k/\eps} },\\
G_M &= \set{ \X_i\in G\given -\sqrt{k/\eps}\le \Tr{\V^\top(\X_i-\M)\V}\le 3\nu\sqrt{k/\eps} },\\
G_H &= \set{ \X_i\in G\given \Tr{\V^\top(\X_i-\M)\V}>3\nu\sqrt{k/\eps} }.
\end{align*} 
Given a subset $T$, with $\abs{T}\le \eps n$ let $T_{L} = \set{ \X_i\in T\given \Tr{\V^\top(\X_i-\M)\V}<3\nu\sqrt{k/\eps} }$ and $T_{H} = T \setminus T_{L}$. By definition, $\sum_{\X_i\in T_{L}} \Tr{\V^\top(\X_i-\M)\V}\le 3\nu\sqrt{k/\eps}\abs{T_{L}}$.

\begin{align*}
&\sum_{\X_i\in T_{H}}\Tr{\V^\top(\X_i-\M)\V}\nonumber\\
\le& \sum_{\X_i\in G_{H}}\Tr{\V^\top(\X_i-\M)\V}\qquad\text{($\because T_H\subseteq G_H$)}\\
\le& \sum_{\X_i\in G_{H}}\Tr{\V^\top(\X_i-\M)\V}+\sum_{\X_i\in G_{L}}\Tr{\V^\top(\X_i-\M)\V}+|G_{L}|\Tr{\V^\top \M \V}\\
=& \sum_{\X_i\in G}\Tr{\V^\top(\X_i-\M)\V} - \sum_{\X_i\in G_M}\Tr{\V^\top(\X_i-\M)\V} +|G_{L}|\Tr{\V^\top \M \V}\\
\le& \sum_{\X_i\in G}\Tr{\V^\top(\X_i-\M)\V} - \sum_{\X_i\in G_M}\Tr{\V^\top(\X_i-\M)\V} +n\eps\Tr{\V^\top \M \V}\\
&\qquad\qquad\qquad\qquad\qquad\qquad\qquad\qquad\qquad\text{(Using Lemma~\ref{lem:goodevent}, part 1(a))}\\
\le& 3n\nu\sqrt{k\eps} +n\eps\Tr{\V^\top \M \V},
\end{align*}
where the last inequality holds since $\sum_{\X_i\in G}\Tr{\V^\top(\X_i-\M)\V}\le n\nu\sqrt{k\eps}$ by Lemma~\ref{lem:goodevent}, part 2, and $\sum_{\X_i\in G_M}\Tr{\V^\top(\X_i-\M)\V}\le 2n\nu\sqrt{k\eps}$ by Lemma~\ref{lem:goodevent}, part 1(b) with $G_M = G_\V$ when $\V$ is the net $\bar{P}_k$. Combining the bound on $T_L$ and $T_H$ yields that $\sum_{\X_i\in T} \Tr{\V^\top(\X_i-\M)\V}\le 6n\nu\sqrt{k\eps} + n\eps\Tr{\V^\top\M\V}$.

For a projection matrix $\V'$ not on the $\delta'$-net $\bar{P}_k$, 

\begin{align}
&\sum_{\X_i\in T}\Tr{{\V'}^\top (\X_i-\M){\V'}} \nonumber\\
&= \vast(\Tr{\sum_{\X_i\in T}{ \round{{\V'}^\top (\X_i-\M){\round{\V'-\V}}}}} + \Tr{\sum_{\X_i\in T}{\round{{\round{\V'-\V}}^\top (\X_i-\M){\V}}}}\nonumber\\
&\qquad\qquad\qquad + \Tr{\sum_{\X_i\in T}{\round{\V^\top (\X_i-\M)\V}}}\vast)\nonumber\\
&\le 2\sum_{\X_i\in T}\|\X_i-\M\|_F\delta' + 6n\sqrt{k\eps} + n\eps\Tr{\V^\top\M\V}\nonumber\\
&\le 2\delta' \sqrt{|T|}\sqrt{\sum_{\X_i\in T}\|\X_i-\M\|_F^2}  + 6n\sqrt{k\eps} + n\eps\Tr{\V^\top\M\V}\nonumber\\
&\le  \eps\nu n + 6\nu n\sqrt{k\eps} + n\eps\Tr{\V^\top\M\V}\nonumber\\
&\le  7\nu n\sqrt{k\eps} + n\eps\Tr{\V^\top\M\V}
\end{align}

\subsection{Proof of Lemma~\ref{lem:filtering_step1}}

\label{sec:filtering_step1_proof}

\begin{enumerate}
    \item 
    Using Markov's inequality we get
    \begin{align}
        \Probability{x\sim \Pcal}{\abs{x-\mu^P}\geq z} &\leq \frac{\Exp{x\sim \Pcal}{({x-\mu^P})^2}}{z^2}\nonumber\\
        &\leq \frac{\sigma_P^2}{z^2}\nonumber\\
        \implies \Probability{x\sim \Pcal}{\abs{x-\mu^P}\geq \sigma_P/\sqrt{{\eps}}} &\leq \eps.\nonumber
    \end{align}
    Let us define the indicator random variable $Z_i \coloneqq \indicator{}{\abs{x_i-\mu^P}\geq \sigma_P/\sqrt{\eps}}$, and let $p \coloneqq \E{Z_i}$. Then from the Chernoff bound we get,
    \begin{align}
    \Probability{G\sim \Pcal^n}{\inv{n}\sumi{1}{n}Z_i \geq \round{1+z}p} &\leq \exp\squarebrack{-\frac{z^2pn}{3}}.\nonumber
    \end{align}
    Set $(1+z)p = 2\eps$, we get 
    \begin{align}
    \Probability{G\sim \Pcal^n}{\inv{n}\sumi{1}{n}Z_i \geq 2\eps} &\leq \exp\squarebrack{-\frac{(2\eps/p-1)^2pn}{3}}\nonumber\\
    &\le \exp\squarebrack{-\frac{\eps n}{3}},\nonumber
    \end{align}
    where the last inequality holds since $p\le \eps$. This implies
    \begin{align*}
        \Prob{\abs{T} \leq (1-2\eps)n} &\leq \exp\squarebrack{-\eps n/3}
    \end{align*}
    \item Define event $\Ecal  \coloneqq \set{x:\abs{x-\mu^P}\leq \sigma_P/\sqrt{\eps}}$. In order to show that $\mu^T = \frac{1}{|T|}\sum_{x_i\in T}x_i$ concentrates to $\mu^P$, we will (1) apply Bernstein inequality to argue that $\mu^T$ concentrate around $\mu^{P'}$ which is the mean of $\Pcal'$, the distribution of $x$ conditioned on the event $\Ecal$, and (2) argue that the mean $\mu^{P'}$ is close to $\mu^P$, which, by triangle inequality, concludes the proof. 

First, we prove a bound on $\abs{\mu^{P'}-\mu^P}$, thus finishing part (2) of the proof.
\begin{proposition}\label{prop:mean-shift}
\[
\abs{\mu^{P'}-\mu^P} \le 2\sqrt{\eps}\sigma_P.
\]
\end{proposition}
\begin{proof}

Notice that
\begin{align}
\abs{\mu^{P'} - \mu^P} &= \abs{\Exp{X\sim \Pcal}{X\mid \Ecal} - \mu^P}\nonumber\\
&= \abs{\Exp{X\sim \Pcal}{X - \mu^P\mid \Ecal}}\nonumber\\
&= \frac{1}{\Prob{\Ecal}}\abs{\Exp{X\sim \Pcal} {\round{X-\mu^P}\indicator{}{\abs{X-\mu^P}\leq \sigma_P/\sqrt{\eps}}}}\nonumber\\
&= \frac{1}{\Prob{\Ecal}}\abs{\Exp{X\sim \Pcal} {\round{X-\mu^P}\indicator{}{\abs{X-\mu^P}\geq \sigma_P/\sqrt{\eps}}}}\nonumber\\
&\le \inv{\Prob{\Ecal}}\sqrt{\Exp{X\sim \Pcal}{\round{X-\mu^P}^2} \Exp{X\sim \Pcal}{\indicator{}{\abs{X-\mu^P}\geq \sigma_P/\sqrt{\eps} }}}\label{eqn:mean-shift:Cauchy}\\
&\le \frac{\sqrt{1-\Prob{\Ecal}}}{\Prob{\Ecal}}\sigma_P\label{eqn:mean-shift:var-bound}\\
&\le 2\sqrt{\eps}\sigma_P,\nonumber
\end{align}
where Equation~\eqref{eqn:mean-shift:Cauchy} holds by Cauchy-Schwarz, and Equation~\eqref{eqn:mean-shift:var-bound} holds since
\[
\Exp{X\sim \Pcal}{\round{X-\mu^P}^2}\le \sigma_P^2.\qedhere
\]
\end{proof}

To show part (1), let us first show the following simple fact due to Bernstein inequality

\begin{proposition}
Given $n$ iid samples $X_1,\ldots,X_n$ from distribution $\Pcal'$, it holds that 
\[
\Probability{X_i\sim \Pcal'}{\abs{\frac{1}{n}\sum_{i=1}^n X_i-\mu^{P'}}\ge \sigma_P\sqrt{\eps} }\le 2\exp\round{-n\eps/13}
\]
\end{proposition}
\begin{proof}

First we bound the variance of $\Pcal'$,

\begin{align*}
\Exp{X\sim \Pcal'}{\round{X-\mu^{P'}}^2} &= \frac{1}{\Prob{\Ecal}}\Exp{X\sim \Pcal}{\round{X-\mu^{P'}}^2\indicator{}{\abs{X-\mu^P}\leq \sigma/\sqrt{\eps}}}\\
&\le \frac{1}{\Prob{\Ecal}}\Exp{X\sim \Pcal}{\round{X-\mu^{P'}}^2}\\
&=\frac{1}{\Prob{\Ecal}}\Exp{X\sim \Pcal}{\round{X-\mu^P}^2 + \round{\mu^P - \mu^{P'}}^2}\\
&\le \frac{1}{\Prob{\Ecal}}\round{\sigma_P^2+4\eps\sigma_P^2}\\
&\le 6\sigma_P^2.
\end{align*}
Hence, we can apply Bernstein inequality (Proposition~\ref{prop:matrix-bernstein}) and get
\begin{align*}
\Probability{X_i\sim \Pcal'}{\abs{\frac{1}{n}\sum_{i=1}^n X_i-\mu^{P'}}\ge \sigma_P\sqrt{\eps}} &\le 2\exp\squarebrack{\frac{-n^2\sigma_P^2\eps/2}{6n\sigma_P^2+n\sigma_P^2/3}}\\
&\le 2\exp(-n\eps/13).\qedhere
\end{align*}

\end{proof}

Notice that condition on the size of set $T$, each $X_i\in T$ follows from distribution $\Pcal'$ independently. Hence, we have that condition on the size of $T$, with probability at least $1-2\exp\round{-\abs{T}\eps/13}$, it holds that $\abs{\frac{1}{\abs{T}}\sum_{X_i\in T} X_i-\mu^{P'}}\le \sigma_P\sqrt{\eps}$. Since we have shown that $\abs{T}\ge n(1-2\eps)$ with probability $1-\exp(-n\eps /3)$, by a union bound, we conclude that with probability at-least $1-3\exp(-n\eps/13)$,
\[
\abs{\frac{1}{\abs{T}}\sum_{X_i\in T} X_i-\mu^{P'}}\le \sigma_P\sqrt{\eps}.
\]
Combining Proposition~\ref{prop:mean-shift} with triangle inequality yields that with probability $1-3\exp(-n\eps/13)$,
\[
\abs{\mu^G-\mu^P} = \abs{\frac{1}{\abs{T}}\sum_{X_i\in T} X_i-\mu^{P}}\le 3\sigma_P\sqrt{\eps}.
\]

    \item 

    Let us define the function $y(x)\coloneqq \round{x-\mu^P}\cdot\indicator{}{\abs{x-\mu^P}\leq \sigma_P/\sqrt{\eps}}$. This implies $\abs{y(X)}\leq \sigma_P/\sqrt{\eps}$. Then,
    \begin{align}
        \Exp{X\sim \Ucal\round{G}}{y(X)^2} &\leq \abs{\Exp{X\sim \Ucal\round{G}}{y(X)^2} - \Exp{X\sim \Pcal}{y(X)^2}} + \Exp{X\sim \Pcal}{y(X)^2}\label{eq:19}
    \end{align}
    Looking at the above two terms individually, with the second term
    \begin{align}
        \Exp{x\sim \Pcal}{y(x)^2} &= \Exp{x\sim \Pcal}{\round{x-\mu^P}^2\cdot \indicator{}{\abs{x-\mu^P}\leq \sigma_P/\sqrt{\eps}}}\nonumber\\
        &\leq \Exp{X\sim \Pcal}{\round{x-\mu^P}^2}\nonumber\\
        &\leq \sigma_P^2.\label{eq:20}
    \end{align}
    For the first term we apply~\cite[Lemma~5.44.]{vershynin2010introduction} on the random variable $y(x)$ for $x\sim \Pcal$, and get that with probability $1-\exp(-C \eps n)$ for a constant $C>0$,
    \begin{align}
        \abs{\Exp{X\sim \Ucal\round{G}}{y(x)^2} - \Exp{x\sim \Pcal}{y(x)^2}} &\leq \frac{\sigma_P^2}{2}\label{eq:21}.
    \end{align} 
    Applying the fact that $|T|\ge (1-2\eps)n$ holds with probability $1-\exp(-\eps n/3)$, and plugging Equation~\eqref{eq:20} and~\eqref{eq:21} in Equation~\eqref{eq:19}, we get that with probability $1-\exp(-\Theta(\eps n))$
    \begin{align}
        \inv{\abs{T}}\sum_{X_i\in T}{\round{X_i-\mu^P}^2} =  \frac{n}{|T|}{\Exp{X\sim \Ucal\round{G}}{y(X)^2}\leq 2\sigma_P^2}.
    \end{align}
    
\end{enumerate}
Taking a union bound of the probability of the three conditions and replacing $\eps$ by $\eps/9$ yield the statement of the lemma.

\subsection{Proof of Lemma~\ref{lem:gapfree}}
\label{sec:proof_gapfree}

We claim that 
\begin{align}
    \norm{*}{ \round{\I-\hat\U\hat\U^\top} \V_j }  &\;\; \leq \;\; \gamma / \sigma_j \;,
    \label{eq:wedin}
\end{align}
where $\V_j=[\vvec_1\,\ldots\,\vvec_j]$ is the matrix consisting of the $j$ singular vectors of $\M = \sigma^2\I + \sum_{i'\in[k]}\X_{i'}$ corresponding to the top $j$  singular values, and $\sigma_j$ is the $j$-th singular value. 
This follows from the proof of the gap-free Wedin's theorem in~\cite[Lemma~B.3]{AL16}, which proves a similar bound on the spectral norm. 
Concretely, let $\hat\U_\perp\in{\mathbb R}^{d\times (d-k)}$ denote an  orthogonal matrix spanning the null space of $\hat\U^\top$. We can write the singular value decomposition as 
\begin{eqnarray}
    \hat\M\;=\;\hat\U \hat \D \hat\U^\top \;,\;\;\;\; \B = \V_j\D\V_j^\top +
    \V'_j\D'{\V'}_j^\top\;,
\end{eqnarray}
where $\V_j'$ spans the subspace orthogonal to $\V_j$, and $\B=  \sum_{i=1}^k \X_i+ \sigma^2\U\U^\top$. Let  ${\bf R}=\hat\M- \B  $, and we get \begin{eqnarray*}
 \hat\U_\perp^\top \hat\M \V_j
    &=& \hat\U_\perp^\top{\bf R}\V_j+\hat\U_\perp^\top\B\V_j \\ 
    &=& \hat\U_\perp^\top{\bf R}\V_j+\hat\U_\perp^\top\V_j\D \;.
\end{eqnarray*}
Since $\hat\U_\perp^\top \hat\M=\zero$,  taking nuclear norm and applying the triangular inequality, 
\begin{eqnarray*}
       \norm{*}{\hat\U_\perp^\top\V_j} &=&  \norm{*}{\hat\U_\perp^\top{\bf R}\V_j\D^{-1}}  \\
       &\leq& \gamma/\sigma_j \;.
\end{eqnarray*}

To get the first term on the upper bound \eqref{eq:subspacebound1}, 
we follow the analysis of \cite[Lemma~5]{li2018learning}.  Notice that $\x_i$ lie on the subspace spanned by $\V_j$ where $j$ is the rank of $\sum_{i'\in[k]}\X_{i'}$. 
It follows from $\normF{\round{\I-\hat\U\hat\U^\top}\V_j}\leq\norm{*}{\round{\I-\hat\U\hat\U^\top}\V_j} \leq \gamma/\sigma_j$ with a choice of $j=k$ that
\begin{eqnarray*}
\sum_{i\in[k]} \enorm{\round{\I-\hat\U\hat\U^\top} \V_j \V_j^\top \x_i}^2 &=& \Tr{\round{\I-\hat\U\hat\U^\top} \V_j \V_j^\top\round{ \sum_{i\in[k]}\x_i\x_i^\top} \V_j \V_j^\top \round{\I-\hat\U\hat\U^\top} } \\ 
& \leq & \enorm{\sum_{i\in[k]}\x_i\x_i^\top} \norm{*}{\V_j \V_j^\top \round{\I-\hat\U\hat\U^\top} \round{\I-\hat\U\hat\U^\top}\V_j \V_j^\top} \\
&\leq & \sigma_{\rm max}\gamma^2/\sigma_{\rm min}^2 \;.
\end{eqnarray*}

Next, we optimize over this choice of $j$ to get the tightest bound that does not depend on the singular values. Applying a similar bound as the above series of inequalities, we get 
\begin{align*}
    \sum_{i\in[k]}\esqnorm{\round{\I-\hat\U\hat\U^\top} \x_i} 
    &= \sum_{i\in[k]}\esqnorm{\round{\I-\hat\U\hat\U^\top} \V_j\V_j^\top \x_i} + \sum_{i\in[k]}\esqnorm{\round{\I-\hat\U\hat\U^\top} \round{\I-\V_j\V_j^\top} \x_i}\\
    &\leq \round{ \sigma_{\rm max}\gamma^2 / \sigma_j^2 }  + (k-j)\sigma_{j+1} \;, 
\end{align*} 
where we used the fact that 
\begin{eqnarray*}
\sum_{i\in[k]}\esqnorm{\round{\I-\hat\U\hat\U^\top} \round{\I-\V_j\V_j^\top} \x_i} &\leq&  \Tr{\round{\I-\V_j\V_j^\top} \round{ \sum_{i\in[k]} \x_i\x_i^\top}\round{\I-\V_j\V_j^\top} }\\ &\leq& (k-j)\sigma_{j+1}\;.
\end{eqnarray*} 
 A good choice of $j$ to approximately minimizes the upper bound is for the two terms to be of similar orders. Precisely, we choose $j$ to be the largest index such that $\sigma_j\geq \gamma^{2/3}\sigma_{\rm max}^{1/3}(k-j)^{-1/3}$ (we take  $j=0$ if $\sigma_1\leq \gamma^{2/3}\sigma_{\rm max}^{1/3}k^{-1/3}$). 
 This gives an upper bound of $ 2\gamma^{2/3}\sigma_{\rm max}^{1/3}k^{2/3}$. 
 
 The second upper bound in \eqref{eq:subspacebound1} follows from a similar argument, and is a  direct corollary of \cite[Lemma A.11]{2020arXiv200208936K}.
\section{Proof sketch of the adaption to exponential tail setting}\label{sec:adaption}
In this setting, we give a sketch of how to adapt the proof of Algorithm~\ref{alg:filtering} to the setting where the distribution has exponential like tail.

 {\bf Closing the gap in Outlier-Robust PCA (ORPCA).} 
\cite{xu2012outlier, feng2012robust,yang2015unified,cherapanamjeri2017thresholding} study robust PCA under the assumption that each sample $\z_i = \A\x_i+\mathbf{v}_i$ where $\x_i, \mathbf{v}_i$ are drawn from isotropic Gaussian distribution, and the goal is to learn the top-$k$ eigenspace of $\A \A^\top$. 
When $n$ samples are observed, $\alpha$ fraction of which are corrupted by an adversary, 
\cite{xu2012outlier} 
introduces a filtering algorithm to find a subspace $\hat\U$ achieving: 
\begin{eqnarray*}
    \norm{*}{\hat\U^\top \A \A^\top\hat\U} \geq  (1-c\sqrt{\alpha}\log(1/\alpha))\,\norm{*}{\A\A^\top }\;, 
\end{eqnarray*}
for some  $c>0$ 
when $n/k(\log k)^5\to \infty$. 
For the reasons explained in \S\ref{sec:subspace}, 
this is sub-optimal in $\alpha$ in comparisons to an information theoretic lower bound with a multiplicative factor of  $(1-c\,\alpha)$. 
Applying the proposed  Algorithm~\ref{alg:filtering}, it is possible to generalize Proposition~\ref{prop:rpca} to this Gaussian setting and achieve an optimal upper bound. 
To get some intuition, notice that with our assumption on the second moment of the projected variance $\Tr{\U^\top \X_i \U}$, our current proof proceeds by focusing on that $1-\alpha$ probability mass, which falls in the interval $[-\Theta({\sqrt{1/\alpha}}), \Theta(\sqrt{1/\alpha})]$. This is tight with only the second moment assumption. However, if we assume exponential tail on $\Tr{\U^\top \X_i \U}$, namely $\Prob{\Tr{\U^\top \X_i \U}>t}\le \exp(-t^p)$, for some $p>0$, we can instead focus on the probability mass in $[-\Theta(\log^{1/p}(1/\alpha)), \Theta(\log^{1/p}(1/\alpha))]$ which is also at least $1-\alpha$. The would give an error of $\alpha\log^{1/\gamma}(1/\alpha)$.
 We provide a sketch of how to adapt the proof of our algorithm to the exponential tail setting as follows.
\begin{lemma}[Main Lemma for Algorithm~\ref{alg:filtering}, adaption from Lemma~\ref{lemma:main-filtering}]\label{lemma:main-filtering-adapt}
Let $\Pcal$ be a distribution over ${d\times d}$ PSD matrices with the property that, 
\begin{align*}
    \Exp{\X\sim \Pcal}{\X} = \M\;,\quad  \|\X-\M\|_2 \le B\;,\\
    \text{and}\quad\max_{\normF{\A}\le 1, \rank(\A)\le k}\Probability{\X\sim \Pcal}{\abs{\Tr{\A(\X-\M)}}\ge \nu(k)t} \le \exp(-t^p)\;,
\end{align*}
for some $p> 0$. Let a set of $n$ random matrices $G = \set{\X_i\in \Rd{d\times d}}_{i\in [n]}$ where each $\X_i$ is independently drawn from $\Pcal$, and the at most $\alpha$ fraction is corrupted by an adversary such that 
the input dataset $S = (G\setminus L)\cup E$ with $|E|=|L|\le \alpha n$, $L\subset G$.  
There exists a numerical constant $c>0$ such that for any $0< \alpha <c$,
if $n =  \tilde\Omega((dk^2+(B/\nu)\sqrt{k}\alpha)/\alpha^2)$, 
Algorithm~\ref{alg:filtering} 
outputs a dataset $S'\subseteq S$  satisfying the following for 
$\hat\M=\frac{1}{\abs{{S}'}}\sum_{\X_i\in {S}'} \X_i$: 
\begin{enumerate}
\item for the top-$k$ singular vectors $\hat{\U}\in \Rd{d \times k}$ of $\hat{\M}$, 
\[
    \Tr{\hat{\U}^\top \round{\hat\M -\M}\hat{\U}}\;\; \leq \;\;  \BigO{\alpha\Tr{{\hat\U}^\top {\M}{\hat\U}} +\nu\sqrt{k}\alpha\log(1/\alpha)^{1/p}}\;. 
\]
\item for all rank-$k$ semi-orthogonal matrices $\V\in \Rd{d\times k}$, we have
\[
    \Tr{\V^\top\round{\hat\M-\M}\V}\;\;\ge
     \;\; -\BigO{\alpha\Tr{\V^\top{\M}\V}+\nu\sqrt{k}\alpha\log(1/\alpha)^{1/p}}\;. 
\]
\end{enumerate}
\end{lemma}

Notice that the probability mass beyond $\abs{\Tr{\sum_{\X_i\in G_\V}{\round{\V^\top (\X_i-\M)\V}}}}\geq \nu\sqrt{k}\log(1/\eps)^{1/p}$ is less than $\eps$ by the exponential tail bound. Hence similar to Lemma~\ref{lem:goodevent}, by letting $G_{\V}\subset G$ to contain all the points in $G$ such that $\abs{\Tr{\sum_{\X_i\in G_\V}{\round{\V^\top (\X_i-\M)\V}}}}\leq \nu\sqrt{k}\log(1/\eps)^{1/p}$, we have part 1 of the following lemma. Part 2 of the following lemma follows from matrix Bernstein inequality, and part 3 can be shown with the same argument of Lemma~\ref{lem:goodevent}. 
\begin{lemma}[Adaption from Lemma~\ref{lem:goodevent}]
Under the hypotheses of Lemma~\ref{lemma:main-filtering-adapt}, 
when $n = \tilde\Omega((dk^2 + \frac{B}{\nu}\sqrt{k}\eps)/\eps^2)$ with probability $1-\tilde\delta$, the following events happen for all semi-orthogonal matrices $\V\in \Rd{d\times k} s.t. \V^\top \V = \I_k$, 
\begin{enumerate}
\item  There exists $G_{\V}\subset G$ such that
\begin{enumerate}
    \item $\abs{G_\V} \geq (1-\eps)n$,
    \item $\abs{\frac{1}{|G_\V|}\Tr{\sum_{\X_i\in G_\V}{\round{\V^\top (\X_i-\M)\V}}}} \le \nu\sqrt{k}\eps\log(1/\eps)^{1/p}$, and
    \item $\abs{\Tr{\sum_{\X_i\in G_\V}{\round{\V^\top (\X_i-\M)\V}}}}\leq \nu\sqrt{k}\log(1/\eps)^{1/p}$,
\end{enumerate}
\item  $\begin{aligned}
    \abs{\frac{1}{n}\sum_{\X_i\in G}\Tr{\V^\top(\X_i-\M)\V}} &\le \nu\sqrt{k}\eps\log(1/\eps)^{1/p},
\end{aligned}$
\item All subset $T \subset G$ such that $|T|\le \eps n$ satisfies
\begin{equation*}
    \sum\limits_{\X_i\in T} \Tr{\V^\top(\X_i-\M)\V}\le 7n\nu\sqrt{k}\eps\log(1/\eps)^{1/p}+n\eps\Tr{\V^\top\M\V}.
\end{equation*}
\end{enumerate}
\end{lemma} 
Thus by changing line 5 of Algorithm~\ref{alg:basic_filter} to $48(\alpha\mu^{S_G}+\nu\sqrt{k}\alpha\log(1/\alpha)^{1/p})$, the statement in Proposition~\ref{prop:basic_filter}
can be changed to that either the output of Algorithm~\ref{alg:basic_filter} satisfies
\begin{equation}
    \Tr{\hat{\U}^\top \round{\frac{1}{\abs{{S}'}}\sum_{\X_i\in {S}'} \X_i-\M}\hat{\U}}\;\; \leq \;\;  \BigO{\alpha\Tr{\hat{\U}^\top {\M}\hat{\U}} +\nu\sqrt{k}\alpha\log(1/\alpha)^{1/p}}
\end{equation}
or the algorithm removes more corrupted data points than uncorrupted data points in expectation. Finally, similar to Proposition~\ref{prop:rpca}, we get 
\begin{align*}
\Tr{{\cal P}_k(\SIGMA)} - \Tr{\hat\U^\top \SIGMA \hat\U} &=  \BigO{\alpha\Tr{{\cal P}_k(\SIGMA)} + \nu\sqrt{k}\alpha\log(1/\alpha)^{1/p}}\;,\\ 
\text{and }\quad\norm{*}{  \SIGMA -  \hat\U\hat\U^\top \SIGMA  \hat\U\hat \U^\top  } &\le \norm{*}{  \SIGMA -  {\cal P}_k \round{\SIGMA}  }+\BigO{\alpha \norm{*}{ {\cal P}_k \round{\SIGMA} }+\nu \sqrt{k}\alpha\log(1/\alpha)^{1/p}}\;. 
\end{align*}
\section{Lower bound for robust PCA, proof of Proposition~\ref{prop:rpca_lb}}\label{sec:rpca-lb}
In this section we show that under the setting of Proposition~\ref{prop:rpca}, it is information theoretically impossible to learn subspace $\hat{\U}$ such that $\|\SIGMA-\hat{\U}\hat{\U}^\top\SIGMA\hat{\U}\hat{\U}^\top\|=o\round{\nu(k)\sqrt{k\alpha}}$.

\begin{definition}
 Given a subset $I\subset [d], \abs{I}=k$, define distribution $\Pcal_I$ as follows: Suppose random variable $\x\sim \Pcal_I$, and each coordinate $x_i, i\in I$ is sampled independently such that 
 \begin{align*}
 x_i = 
 \begin{cases} 
 \sqrt{\nu(k)}& \text{ with probability } (1-\alpha/k)/2\\
 -\sqrt{\nu(k)} & \text{ with probability } (1-\alpha/k)/2\\
 (\nu(k)^2k/\alpha)^{1/4} & \text{ with probability } \alpha/2k\\
 -(\nu(k)^2k/\alpha)^{1/4} & \text{ with probability } \alpha/2k
 \end{cases}.
 \end{align*}
 The other coordinates $x_i, i\not\in I$ is sampled independently such that 
 \begin{align*}
 x_i = 
 \begin{cases} 
 \sqrt{\nu(k)}& \text{ with probability } 1/2\\
 -\sqrt{\nu(k)} & \text{ with probability } 1/2
 \end{cases}.
 \end{align*}
\end{definition}
 The second moment matrix $\SIGMA^{I}:=\Exp{\x\sim \Pcal_I}{\x\x^\top}$ of $\Pcal_{I}$ satisfies
 \begin{align*}
 \Sigma^I_{i,j} = 
 \begin{cases} 
 \nu(k)(1-\alpha/k+\sqrt{\alpha/k}) & i = j, i\in I \\
 \nu(k) & i = j, i\not\in I\\
 0 & i\ne j
 \end{cases}.
 \end{align*}
It is clear that with probability $1$, $\enorm{\x\x^\top-\SIGMA^I}\le \|\x\|_2^2+\enorm{\SIGMA^I}\le \nu(k)\round{d+2k^{3/2}/\sqrt{\alpha}}\le 2\nu(k)d=B$.
 Next we verify the fourth moment condition in Proposition~\ref{prop:rpca}. 
 WLOG, let us assume $A_{i,j}=0$ for any $j<i$, hence $\A$ is a upper triangular.
\begin{align*}
    \E{\Tr{\A\round{\x\x^\top-\SIGMA^I}}^2} &= \E{\round{\sum_{i,j}A_{i,j}\round{x_ix_j-\Sigma^I_{i,j}}}^2}\\
    &= \E{\sum_{i,j}\sum_{i',j'}A_{i,j}\round{x_ix_j-\Sigma^I_{i,j}}A_{i',j'}\round{x_{i'}x_{j'}-\Sigma^I_{i',j'}}}
\end{align*}
Based on the number of distinct values $i,j,i',j'$ take, the terms inside the summation can be classified into $4$ difference cases
 \begin{enumerate}
 \item $i,j,i',j'$ takes $4$ difference values. In this case, \[\E{A_{i,j}\round{x_ix_j-\Sigma^I_{i,j}}A_{i',j'}\round{x_{i'}x_{j'}-\Sigma^I_{i',j'}}}=0\]
 \item $i,j,i',j'$ takes $3$ difference values. In this case, \[\E{A_{i,j}\round{x_ix_j-\Sigma^I_{i,j}}A_{i',j'}\round{x_{i'}x_{j'}-\Sigma^I_{i',j'}}}=0\]
 \item $i,j,i',j'$ takes $2$ difference values. In this case, if $i=j$, $i
' =j'$ and $i\ne i'$,
\[\E{A_{i,i}\round{x_i^2-\Sigma^I_{i,i}}A_{i',i'}\round{x_{i'}^2-\Sigma^I_{i',i'}}}=0.\] If $i=i'$, $j = j'$ and $i\ne j$, \[\E{A_{i,j}^2\round{x_ix_j-\Sigma^I_{i,j}}^2}=\E{A_{i,j}^2(x_ix_j)^2} = \E{A_{i,j}^2\Sigma^I_{i,i}\Sigma^I_{j,j}}.\] If $i = j'$, $j = i'$ and $i\ne j$, $A_{i,j}$ or $A_{j,i}$ must be $0$, hence the expectation is $0$.
 \item $i,j,i',j'$ takes $1$ value. In this case \[\E{ A_{i,i}^2\round{x_i^2-\Sigma^I_{i,i}}^2}=\begin{cases}0 & i\not\in I\\
 A_{i,i}^2\nu(k)^2(2-\alpha/k)& i\in I
 \end{cases}
 \]
 \end{enumerate}
 Taking summation over the above cases yields the following bound
 \[
 \sum_{i<j} A_{i,j}^2\Sigma^I_{i,i}\Sigma^I_{j,j} +  \sum_{i\in I}A_{i,i}^2(2-\alpha/k)\le 2\normF{\A}^2\nu(k)^2 \le 2\nu(k)^2.
 \]
 
Here we have shown that each distribution $\Pcal_I$ satisfies 
 \[ \max_{\normF{\A}\leq 1} \Exp{\x\sim \Pcal_I}{\round{\left\langle \A,\x\x^\top-\E{\x\x^\top}\right\rangle}^2}
\leq 2\nu(k)^2.
\]
 Then we define the base case distribution $\Pcal_\emptyset$ as follows:
 \begin{definition}
 Suppose random variable $\x\sim \Pcal_\emptyset$, and each coordinate $x_i, i\in I$ is sampled independently 
 \begin{align*}
 x_i = 
 \begin{cases} 
 \sqrt{\nu(k)}& \text{ with probability } 1/2\\
 -\sqrt{\nu(k)} & \text{ with probability } 1/2
 \end{cases}.
 \end{align*}
 \end{definition}
 It is clear that $D_{\rm TV}(\Pcal_{I}, \Pcal_\emptyset) = \alpha$. Thus we have shown that each pair $(\Pcal_\emptyset, \Pcal_I) \in \Theta_{\nu(k), \alpha}$. Now let us fix an estimator $\hat\U$, and let $\hat\U_\emptyset = \hat\U(\{\x_i\}_{i=1}^n), \{\x_i\}_{i=1}^n\sim \Pcal_\emptyset^n$ denote the random subspace when the datapoints are drawn from $\Pcal_\emptyset$. WLOG, let us assume $d$ is a multiple of $k$, and let $I_1,\ldots, I_{d/k}$ be a partition of $[n]$. 
 Notice that
 \begin{align*}
 \E{\sum_{i=1}^{d/k}\Tr{\hat\U_{\emptyset}^\top\Sigma_{I_i}\hat\U_{\emptyset}}} &= \nu(k)\cdot k\cdot \round{d/k+1-\alpha/k+\sqrt{\alpha/k}},
 \end{align*}
 and hence there exists a $i^*$ such that
 \begin{align*}
 \E{\Tr{\hat\U_{\emptyset}^\top\Sigma_{I_{i^*}}\hat\U_{\emptyset}}} &\le \nu(k)\cdot k\cdot \round{d/k+1-\alpha/k+\sqrt{\alpha/k}}\cdot (d/k)^{-1}\\
 &\le \nu(k)\cdot k \cdot \round{1+\frac{1+\sqrt{\alpha/k}}{d/k}}.
 \end{align*}
 The sub-optimality can be expressed as
 \begin{align*}
 &\E{\norm{*}{\SIGMA_{I_{i^*}}-\hat\U_\emptyset\hat\U^\top_\emptyset\SIGMA_{I_{i^*}}\hat\U_\emptyset\hat\U^\top_\emptyset}-\norm{*}{\SIGMA_{I_{i^*}}-\Pcal_k\round{\SIGMA_{I_{i^*}}}}}\\
 =& 
 \nu(k) k\round{1-\alpha/k+\sqrt{\alpha/k}} - \E{\Tr{\hat\U_{\emptyset}^\top\Sigma_{I_{i^*}}\hat\U_{\emptyset}}}\\
 \ge & \nu(k)\round{\sqrt{\alpha k}-\alpha-\frac{k^2(1+\sqrt{\alpha/k})}{d}}\text{ (Using $k\ge 16$, $d\ge k^2/\alpha$) }\\ 
\ge &\frac{1}{16}\nu(k) \sqrt{\alpha k}
 \end{align*}
 This implies that for any subspace estimator $\hat\U$, we can find distribution $\Dcal = \Pcal_{\emptyset}, \Dcal' = \Pcal_{I_{i^*}}$ such that 
\begin{enumerate}
 \item $\begin{aligned}
 \Exp{\{\x_i\}_{i=1}^n\sim \Dcal^n}{\norm{*}{\SIGMA_{I_{i^*}} - \hat\U\hat\U^\top\SIGMA_{I_{i^*}}\hat\U\hat\U^\top}- \norm{*}{  \SIGMA_{I_{i^*}} -  {\cal P}_k \round{\SIGMA_{I_{i^*}}}  }}\ge \frac{1}{16}\nu(k)\sqrt{k\alpha},
 \end{aligned}
 $
 \item
 $\begin{aligned}
 D_{\rm TV}(\Dcal, \Dcal') \le \alpha,
 \end{aligned}
 $
 \item
 $\begin{aligned}
 \max_{\normF{\A}\leq 1} \Exp{\x\sim \Dcal'}{\left(\left\langle \A,\x\x^\top-\E{\x\x^\top}\right\rangle\right)^2}
\leq 2\nu(k)^2.
 \end{aligned}
$
\end{enumerate}
The proof is complete.
\section{Robust clustering algorithm}
\label{sec:cluster_algorithm}

\begin{definition}
Pseudo-distributions are generalizations of probability distributions except for the fact that they need not be non-negative. A level-$2m$ pseudo-distribution $\xi$, for $m\in\Natural\cup\curly{\infty}$, is a measurable function that must satisfy
\begin{align}
    \int_{\Rd{d}}q(\x)^2d\xi(\x) &\geq 0\qquad\text{for all polynomials $q$ of degree at-most $m$, and}\\
    \int_{\Rd{d}}d\xi(\x) &= 1.
\end{align}
A straightforward polynomial interpolation argument shows that every level-$\infty$ pseudo-distribution $\xi$ is non-negative, and thus are actual probability distributions.
\end{definition}
\begin{definition}
A pseudo-expectation $\pseudoExp{\xi}{f(\x)}$ of a function $f$ on $\Rd{d}$ with respect to a pseudo-expectation $\xi$, just like the usual expectation, is denoted as 
\[
\pseudoExp{\xi}{f(\x)} = \int_{\Rd{d}}f(\x)d\xi(\x).
\]
\end{definition}
\begin{definition}
The SOS ordering $\preceq_{\text{SOS}}$ between two finite dimensional tensors $\Tcal_1$, and $\Tcal_2$, i.e., $\Tcal_1\preceq_{\text{SOS}} \Tcal_2$, means $\inner{\Tcal_1,\vvec^{\otimes 2m}} \preceq_{2m} \inner{\Tcal_2,\vvec^{\otimes 2m}}$ as polynomial in $\vvec$.
\end{definition}

Given $\set{\hat\beta_i}_{i=1}^n$, we want to find $\curly{\gamma_i}_{i=1}^n$ such that $\inv{n}\sumi{1}{n}\pseudoExp{\xi}{\inner{\hat\beta_i-\gamma_i,\vvec}^{2m}}$ is small for all pseudo-distributions $\xi$ of $\vvec$ over the sphere. Since $\gamma_i = \hat\beta_i\ \forall\ i\in\squarebrack{n}$ would be an over-fit, therefore to avoid it, it turns out that the natural way is to introduce the term $\sumi{1}{n}\inner{\gamma_i-\beta_i,\gamma_i}^{2m}$ which must be small at the same time. On the other hand, we know that if $\sum_{i\in\Gcal}\pseudoExp{\xi}{\inner{\hat\beta_i-\gamma_i,\vvec}^{2m}}$, and $\sum_{i\in\Gcal}\pseudoExp{\xi}{\inner{\hat\beta_i-\beta_i,\vvec}^{2m}}$ (from the SOS proof) are small, then from the Minkowski's inequality, $\sum_{i\in\Gcal}\pseudoExp{\xi}{\inner{\gamma_i-\beta_i,\vvec}^{2m}}$ will also be small. To make this hold, it is sufficient to impose that whenever $\curly{\z_i}_{i=1}^n$ are such that $\sumi{1}{n}\pseudoExp{\xi}{\inner{\z_i,\vvec}^{2m}}\leq 1$ for all pseudo-distributions $\xi$ over the unit sphere, then $\sumi{1}{n}\inner{\z_i,\gamma_i}^{2m}$ is also small. This, however, is not efficiently imposable, but there is a standard SOS way of relaxing this, which is to require $\sumi{1}{n}\pseudoExp{\zeta_i}{\inner{\z_i,\gamma_i}^{2m}}$ to be small whenever $\sumi{1}{n}\pseudoExp{\zeta_i}{\z_i^{\otimes 2m}} \preceq_{\text{SOS}} \Ical$ for all pseudo-distributions $\curly{\zeta_i(\z_i)}_{i=1}^n$, where $\Ical$ is the identity tensor of the appropriate dimension.

Therefore, we need to find $\curly{\gamma_i}_{i=1}^n$ such that $\sumi{1}{n}\pseudoExp{\xi}{\inner{\hat\beta_i-\gamma_i,\vvec}^{2m}}$, and $\sumi{1}{n}\pseudoExp{\zeta_i}{\inner{\z_i,\gamma_i}^{2m}}$ are small whenever $\sumi{1}{m}\pseudoExp{\zeta_i}{\z_i^{\otimes 2m}} \preceq_{\text{SOS}} \Ical$ for all pseudo-distributions $\xi$ over the unit sphere, and $\curly{\zeta_i}_{i=1}^n$.

\begin{algorithm}[ht]
   \caption{Basic clustering relaxation~\cite[Adaptation of Algorithm~1]{2017arXiv171107465K}}
   \label{alg:basis_clustering_relaxation}
\begin{algorithmic}[1]
   \STATE {\bf Input:} $\{\hat\beta_i\}_{i\in[n]}$, $\curly{c_i\in\squarebrack{0,1}}_{i=1}^n$, $m\in\Natural$, multiplier $\lambda\geq 0$, threshold $\Gamma\geq 0$.
   \STATE Define $\tau_i(\gamma_i,\xi,\zeta_i) \coloneqq \pseudoExp{\vvec\sim\xi}{\inner{\hat\beta_i-\gamma_i,\vvec}^{2m}} + \lambda \pseudoExp{\z_i\sim \zeta_i}{\inner{\gamma_i,\z_i}^{2m}}$.
   \STATE Find $\curly{\gamma_i^*}_{i=1}^n$ such that $\sumi{1}{n}c_i\tau_i(\gamma_i^*,\xi,\zeta_i) \leq 2\Gamma$, for all $\xi$ over unit sphere, and \\ \hspace{1cm} for all $\curly{\zeta_i}_{i=1}^n$ satisfying $ \sumi{1}{n}\pseudoExp{\zeta_i}{\z_i^{\otimes 2m}}\preceq_{\text{SOS}} \Ical$.
   \STATE Or else, find $\xi^*$, and $\curly{\zeta_i^*}_{i=1}^n$ such that $\forall\ \curly{\gamma_i}_{i=1}^n$, $\sumi{1}{n}c_i\tau_i(\gamma_i,\xi^*,\zeta_i^*) \geq \Gamma\ $.
   \STATE {\bfseries Output:}  $\curly{\gamma_i^*}_{i=1}^n$, or $\curly{\xi^*,\curly{\zeta_i^*}_{i=1}^n}$.
\end{algorithmic}
\end{algorithm}

If there is no solution that makes the desired quantities small, then from duality there must exist pseudo-distributions $\xi^*$, and $\curly{\zeta_i^*}_{i=1}^n$ such that the objective cannot be small for any choice of $\curly{\gamma_i}_{i=1}^n$. Since elements in $\curly{\gamma_i}_{i=1}^n$ are independent of each other in the objective for a fixed $\xi$ and $\curly{\zeta_i}_{i=1}^n$, and the objective can made small on the good set $\Gcal$, therefore we can look at $\min_{\gamma}\pseudoExp{\xi^*}{\inner{\hat\beta_i-\gamma,\vvec}^{2m}}$ or $\min_\gamma\pseudoExp{\zeta_i^*}{\inner{\z_i,\gamma}^{2m}}$ if they are large for any $i\in\squarebrack{n}$. Such tasks can be removed and the process can be repeated. It is shown in~\cite{2017arXiv171107465K} that the procedure after a finite number of iterations can remove all the outliers and eventually the sum of the desired quantities can be made small.

\begin{algorithm}[ht]
   \caption{Outlier Removal Algorithm~\cite[Adaptation of Algorithm~2]{2017arXiv171107465K}}
   \label{alg:outlier_removal}
\begin{algorithmic}[1]
   \STATE {\bf Input:} $\set{\hat\beta_i }_{i\in[n]}$, $B\geq 0$, $m\in\Natural$, $p_{\rm min}$, and $\rho\geq 0$.
   \STATE Initialize $\cvec = \one_n$, and set $\lambda = p_{\rm min} n \round{B/\rho}^{2m}$.
   \WHILE{true}
   \STATE Run Algorithm~\ref{alg:basis_clustering_relaxation} with $\{\hat\beta_i\}_{i\in\squarebrack{n}}$, $\cvec$, $\lambda$, and threshold $\Gamma = 4\round{nB^{2m} + \lambda \rho^{2m}/p_{\rm min}}$ to obtain $\curly{\gamma_i^*}_{i=1}^n$, or $\curly{\xi^*,\curly{\zeta_i^*}_{i=1}^n}$.
   \IF{$\curly{\gamma_i^*}_{i=1}^n$ are obtained}
        \STATE \hspace{0.3cm}{\bfseries Output:} $\curly{\gamma_i^*}_{i=1}^n$, $\cvec$
   \ELSIF{$\curly{\xi^*,\curly{\zeta_i^*}_{i=1}^n}$ is obtained}
        \STATE $\tau_i^* \gets \min_\gamma \tau_i(\gamma,\xi^*,\zeta_i^*)\quad \forall\ i\in\squarebrack{n}$, as defined in Algorithm~\ref{alg:basis_clustering_relaxation}
        \STATE $c_i \gets c_i\round{1-\tau_i^*/\max_{j \in\squarebrack{n}}\tau_j^*}\quad \forall\ i\in\squarebrack{n}$
   \ENDIF
   \ENDWHILE
\end{algorithmic}
\end{algorithm}

Algorithm~\ref{alg:outlier_removal} uses a down-weighting way of reducing the weight of possible outlier tasks, and~\cite{2017arXiv171107465K} show that the re-weighting step down-weights the outlier tasks more than the tasks in the $\ell$-th set. They also show that the returned $\curly{\gamma_i^*}_{i=1}^n$ constitute a good clustering such that one of the clusters is centered close to some true mean $\w_j$ for some $j\in\squarebrack{k}$.

\begin{algorithm}[ht]
   \caption{Algorithm for re-clustering $\curly{\gamma_i}_{i=1}^n$~\cite[Adaptation of Algorithm~3]{2017arXiv171107465K}}
   \label{alg:recluster}
\begin{algorithmic}[1]
   \STATE {\bf Input:} 
   $\Dcal_H = \{ \{(\x_{i,j},y_{i,j}) \}_{j\in[t_{H}]} \}_{i\in[n_H]}$, $\hat\U$, 
   $B\geq 0$, $m,k\in\Natural$, $p_{\rm min}$, and $\rho\geq 0$.
   \STATE Initialize $R = \rho$,  set $W = \curly{\zero}$, and $\hat\beta_i\leftarrow(1/t_H)\sum_{\ell=1}^{t_H} \hat\U^\top\x_{i,j}y_{i,j} $ for all $i\in[n]$
   \STATE $\rho_{\rm final} \gets \Theta\round{Bp_{\rm min}^{-1/m}}$, $M = \ceil{4/p_{\rm min}}$
   \WHILE{$R\geq \rho_{\rm final}$}
   \STATE $b \gets 1$
   \FOR{$\w'\in W$}
        \STATE Let $\{\gamma_i^{(b)}\}_{i=1}^n, \cvec^{(b)}$ be the output of Algorithm~\ref{alg:outlier_removal} with $\{\hat\beta_i-\w'\}_{i=1}^n$, $B$, $p_{\rm min}$, $R$ as input.
        \STATE $b\gets b+1$
   \ENDFOR
   \STATE Let $\{\Ccal_j\}_{j=1}^M$ be the maximal covering derived from $\{\gamma_i^{(j)}\}_{i=1,j=1}^{n,M}, \set{\cvec^{(j)}}_{j=1}^M$
   \STATE $W\gets \curly{\w'_j}_{j=1}^M$ where $\w'_j$ is the mean of points in $\Ccal_j\ \forall\ j\in\squarebrack{M}$
   \STATE $R\gets C'\round{\sqrt{RBp_{\rm min}^{-1/m}} + Bp_{\rm min}^{-1/m}}$
   \ENDWHILE
    \STATE {\bfseries Output:} $W,\set{\Ccal_\ell}_{\ell=1}^M$
\end{algorithmic}
\end{algorithm}

Algorithm~\ref{alg:recluster} repeatedly uses Algorithm~\ref{alg:outlier_removal} on re-centered data to find the individual clusters. The set $S_j$ obtained in Algorithm~\ref{alg:recluster} is almost entirely a subset of one of the true good sets. After obtaining $M$ centers from Algorithm~\ref{alg:recluster} we can re-consolidate the sets into $k$ new sets $\set{\Ccal_\ell}_{\ell=1}^k$ (by merging together all $S_j$ whose means are within distance $B\tilde{p}_{\rm min}^{-1/m}/4$.), and can shown to obey the desired guarantee using~\cite[Theorem~5.4.]{2017arXiv171107465K}.

\begin{algorithm}[ht]
   \caption{Estimating $\set{r_\ell^2}_{\ell=1}^k$}
   \label{alg:estimate_rl2}
\begin{algorithmic}[1]
    \STATE {\bf Input:} $\Dcal_H = \{ \{(\x_{i,j},y_{i,j}) \}_{j\in[t_{H}]} \}_{i\in[n_H]}$, $\set{\tilde\w_\ell}_{\ell=1}^k$, $\alpha\geq 0$, $\delta\in(0,1)$
    \FOR{$\ell\in\squarebrack{k}$}
        \STATE $r_{\ell,i}^2\gets t_{H}^{-1} \sum_{j\in[t_{H}]} \round{y_{i,j}-\x_{i,j}^\top\tilde\w_\ell}^2$ for all $i\in\Ccal_\ell$
        \IF{$\alpha>0$}
            \STATE $\tilde{r}_\ell^2 \gets $Univariate\_Mean\_Estimator$\round{\set{ r_{\ell,i}^2}_{i \in\Ccal_\ell},\alpha,\delta}$\hfill[~\cite{lugosi2019robust}]
        \ELSE
            \STATE $\tilde{r}_\ell^2 \gets \inv{\abs{\Ccal_\ell}}\sum\limits_{i\in\Ccal_\ell}r_{\ell,i}^2$
        \ENDIF
    \ENDFOR
    \STATE {\bfseries Output:} $\set{\tilde{r}_\ell^2}_{\ell=1}^k$
\end{algorithmic}
\end{algorithm}

\section{Proof of  Lemma~\ref{lem:cluster} for robust clustering}
\label{sec:cluster_proof}

We use  \cite[Theorem~1.2]{2017arXiv171107465K}, 
to analyze Algorithm~\ref{alg:recluster} with input 
$\set{\hat\U^\top\hat\beta_i = \frac{1}{t}\sum_{j=1}^{t} y_{i,j} \hat\U^\top\x_{i,j}}_{i\in[n]}$ for $t=t_H$ and $n=n_H$.  

\begin{theorem}[{\cite[Theorem~1.2]{2017arXiv171107465K}}]
    \label{thm:SoSproof}
    Suppose $\set{\hat\U^\top \hat\beta_i \in \R^k }_{i\in[n]}$ can be  partitioned into sets $\set{\Gcal_\ell}_{\ell=1}^k\cup \Hcal$, where $\Hcal$ is the set of outliers, of size $\alpha n$. 
    Suppose $\abs{\Gcal_\ell}=n\tilde{p}_\ell$, and has mean $\w_\ell$, that its $2m$-th moment $M_{2m}\round{\Gcal_\ell}\preceq_{2m} B$. 
    Also suppose that $\alpha\leq \tilde{p}_{\rm min}/8$. 
    Finally suppose the separation $\Delta\geq C_{\rm sep}B/\tilde{p}_{\rm min}^{1/m}$ 
    with $C_{\rm sep}\geq C_0$ for a universal constant $C_0$. 
    Then Algorithm~\ref{alg:recluster} runs in time $\Ocal((nk)^{\Ocal(m)})$ and  outputs 
    estimates $\set{\tilde{\w}'_\ell}_{\ell=1}^k$ such that $\enorm{\tilde{\w}'_\ell - \hat\U^\top \w_\ell}\leq \BigO{B\cdot\round{\alpha/\tilde{p}_{\rm min} + C_{\rm sep}^{-2m}}^{1-1/2m}}$.
\end{theorem}

It shows how the accuracy depends on the choice of $m$ and the SOS proof of an upper  bound $B$. We will show that 
if $B=\rho\sqrt{2mC/t}$ then the SOS proof holds, in which case the condition $\Delta\geq C_{\rm sep}B/\tilde p_{\rm min}^{1/m} $ in Theorem~\ref{thm:SoSproof} translates into
\begin{align}
    t &\geq \frac{2mC\rho^2}{\Delta^2}\cdot \frac{C_{\rm sep}^2}{\tilde{p}_{\rm min}^{2/m}}.
\end{align}
Further, to get $\enorm{\tilde\w_j-\hat\U^\top\w_j}= \BigO{\Delta}$ we need
\begin{align}
    t &\gtrsim \frac{2mC\rho^2}{\Delta^2}\cdot \max\curly{\inv{C_{\rm sep}^{4m-2}}, \round{\frac{\alpha}{\tilde{p}_{\rm min}}}^{2-\inv{m}}}.
\end{align}
Combining the two conditions, we finally get
\begin{align}
    t &\gtrsim \frac{2mC\rho^2}{\Delta^2}\cdot \max\curly{\frac{C_{\rm sep}^2}{\tilde{p}_{\rm min}^{2/m}}, \inv{C_{\rm sep}^{4m-2}}, \round{\frac{\alpha}{\tilde{p}_{\rm min}}}^{2-\inv{m}}}.
\end{align}

We are left to show that SOS proof exists for the choice of $B=\rho\sqrt{2mC/t}$.
The following lemma, whose proof is in \S\ref{sec:emp-sos-proof}, gives 
\begin{align}
    \frac{1}{\abs{\Gcal_\ell}}\sum_{i\in\Gcal_\ell}\inner{\hat\U^\top\round{\hat\beta_i - \beta_i}, \vvec}^{2m} \preceq_{2m} \rho^{2m}\|\vvec\|_2^{2m}(2m)^m\frac{C^m}{t^m} \leq (2m)^m\frac{C^m}{t^m}
\end{align}
for all $\enorm{\vvec}\leq 1$ with probability at least $7/8$ for all $\ell\in\squarebrack{k}$.

\begin{lemma}[SOS proof exists with high probability]\label{lem:emp-sos-proof}
Given $t\ge 2m$, for $m,t\in\Natural$, there exists a constant $C>0$ such that
for any $\ell\in\squarebrack{k}$, and $np_\ell \ge (km)^{\Theta(m)}/\delta$, with probability at least $1-\delta$, it holds that
\begin{align*}
\frac{1}{n\tilde{p}_\ell}\sum_{i\in\Gcal_\ell}\inner{\hat\U^\top(\hat\beta_i-\beta_i) , \vvec}^{2m} \preceq_{2m} \rho^{2m}\|\vvec\|_2^{2m}(2m)^m\frac{C^m}{t^m},
\end{align*}
for 
all $\vvec\in\R^k$.
\end{lemma}
From~\cite[Proposition~D.7.]{2020arXiv200208936K} we have that if $n = \Omega\round{\frac{\log k}{p_{\rm min}}}$, then $\tilde{p}_{\rm min} \geq p_{\rm min}/2$ with probability at least $7/8$. To further simplify the conditions, note that $\alpha<\tilde{p}_{\rm min}/8$, and fix $C_{\rm sep} = \Theta(1)$, then we simply need
\begin{align}
    t &\gtrsim \frac{m\rho^2}{p_{\rm min}^{2/m}\Delta^2}.
\end{align}
Using a median of means algorithm from~\cite[Proposition~1]{lugosi2019mean} and~\cite{hopkins2018mean,cherapanamjeri2019fast}, by repeatedly and independently estimating $M = \Omega\round{\log\inv\delta}$ number of estimates, $\set{\set{\tilde\w_j'^{(\ell)}}_{j=1}^k}_{\ell=1}^M$ we can compute the improved estimates $\set{\tilde\w_j\in\Rd{d}}_{j=1}^k$ by applying back $\hat\U$ that satisfy
\begin{align}
    \enorm{  \tilde\w_j-\hat\U\hat\U^\top\w_j} &\lesssim \Delta\quad\forall\ j\in\squarebrack{k}\label{eq:tilde-w-bound}
\end{align}

With the assumption that $$
\enorm{\hat\U\hat\U^\top\w_j - \w_j}\lesssim \Delta
$$and Equation~\eqref{eq:tilde-w-bound}, we have
\begin{align}
    \enorm{\tilde\w_j - \w_j} &\leq \enorm{\tilde\w_j - \hat\U\hat\U^\top\w_j} + \enorm{\hat\U\hat\U^\top\w_j - \w_j}\lesssim \Delta.\qedhere
\end{align}

We next show a bound on the error in estimating $r_\ell$.
\begin{proposition}[Estimating $r_\ell$]\label{prop:estimating-r2}
If $n_H = \tilde\Omega\round{\frac{\rho^4}{\Delta^4t_Hp_{\rm min}}}$, we can estimate $r_\ell^2$ as $\tilde{r}_\ell^2$ satisfying
\begin{align}
    \abs{\tilde{r}^2_\ell - r^2_\ell} &\leq r_\ell^2\frac{\Delta^2}{50\rho^2}
\end{align}
with probability at least $1-\delta$, for all $\ell\in\squarebrack{k}$ using Algorithm~\ref{alg:estimate_rl2}, where $r_\ell^2\coloneqq \esqnorm{\tilde\w_\ell-\w_\ell} + s_\ell^2$, if $\alpha_H = \BigO{\frac{\Delta^2\sqrt{t_H}p_{\rm min}}{\rho^2\log \round{\frac{\rho^2}{\Delta^2t_H}}}}$.
\end{proposition}

\begin{proof}
Define $r_{\ell,i}^2 = t_H^{-1}\sum_{j=1}^{t_H}\round{y_{i,j}-\tilde\w_\ell^\top\x_{i,j}}^2\sim \frac{r_\ell^2}{t_H}\chi^2(t_H)$ for all $i\in\Ccal_\ell, \ell\in\squarebrack{k}$. Since we compute $r_{\ell}^2$ for each cluster $\Ccal_\ell$ independently, the maximum corruption in the $\ell$-th cluster is bounded by $\alpha_H/p_\ell$. Using Corollary~\ref{cor:chi-var}, we can compute an estimator using Algorithm~\ref{alg:estimate_rl2}, that given $\tilde\alpha_\ell = \alpha_H/p_\ell$ corrupted samples, returns $\tilde{r}_\ell^2$ for all $\ell\in\squarebrack{k}$ satisfying
\begin{align}
    \abs{r_\ell^2-\tilde{r}_\ell^2} &= \BigO{r_\ell^2\cdot \tilde\alpha_\ell\cdot\max\set{\frac{\log(1/\tilde\alpha_\ell)}{t_H},\sqrt{\frac{\log(1/\tilde\alpha_\ell)}{t_H}}}}\nonumber\\
    &= \BigO{r_\ell^2\cdot\frac{\alpha_H}{p_\ell}\cdot\max\set{\frac{\log(p_\ell/\alpha_H)}{t_H},\sqrt{\frac{\log(p_\ell/\alpha_H)}{t_H}}}}\nonumber\\
    &= \BigO{r_\ell^2\cdot\frac{\alpha_H}{p_\ell}\cdot\frac{\log(p_\ell/\alpha_H)}{\sqrt{t_H}}}\nonumber
\end{align}
when $n_Hp_\ell = \tilde\Omega\round{\inv{\tilde\alpha_\ell^2}} = \tilde\Omega\round{\frac{p_\ell^2}{\alpha_H^2}}$ for all $\ell\in\squarebrack{k}$. Using the fact that
\[
\frac{\beta}{\log\inv\beta}\geq \frac{e}{e-1}\alpha \implies \alpha\log\inv\alpha\leq \beta
\]
for $\alpha,\beta\in\round{0,1}$, we have that for $\alpha_H \leq C{\frac{\Delta^2\sqrt{t_H}p_{\rm min}}{\rho^2\log \round{\frac{\rho^2}{\Delta^2t_H}}}}$ for some $C>0$,
\begin{align}
    \abs{r_\ell^2-\tilde{r}_\ell^2} &= r_\ell^2\cdot \frac{\Delta^2}{50\rho^2}
\end{align}
with probability at-least $1-\delta$, when $n_H = \tilde\Omega\round{\frac{\rho^4}{\Delta^4t_Hp_{\rm min}}}$.
\end{proof}

\subsection{Proof of  Lemma~\ref{lem:emp-sos-proof} for 
Sum-of-Squares proof  \texorpdfstring{\(y\x\)}{yx}}
\label{sec:emp-sos-proof}

We combine the following 
 Proposition~\ref{prop:sos-deviation} and Lemma~\ref{lem:sos-proof} to yield the desired SOS proof.

\begin{proposition}[see the proof of Lemma 4.1 in~\cite{hopkins2018mixture}]\label{prop:sos-deviation}

Let $Z_i = \frac{1}{t}\sum_{j=1}^t y_{i,j}\x_{i,j}-\beta_i$ for $i\in \Gcal_\ell$. If $np_\ell\ge (km)^{\Theta(m)}/\delta\ \forall\ \ell\in\squarebrack{k}$, then with probability at least $1-\delta$,
\begin{align*}
\sum_{i=1}^n\inner{Z_i, \vvec}^{2m} - \E{\inner{Z_i, \vvec}^{2m}}\preceq_{2m} \frac{1}{4}\|\vvec\|_2^{2m}\frac{C^m}{t^m}
\end{align*}
\end{proposition}
\begin{proof}
We show in Lemma~\ref{lem:sos-proof} that the distribution of $\frac{\sqrt{t}}{\sqrt{C}}Z_i$ is $2m$-explicitly bounded with variance proxy $1$. In the proof of Lemma 4.1 in~\cite{hopkins2018mixture}, it is shown that for a $2m$-explicitly bounded distribution, given $n\ge (km)^{\Theta(m)}/\delta$ samples, with probability at least $1-\delta$ (Fact 7.6 in~\cite{hopkins2018mixture}), 
\begin{align*}
\frac{t^m}{C^m}\left(\sum_{i=1}^n\inner{Z_i, \vvec}^{2m} - \E{\inner{Z_i, \vvec}^{2m}}\right)\preceq_{2m} \frac{1}{4}\|\vvec\|_2^{2m},
\end{align*}
which implies the propositions.
\end{proof}

\begin{fact}[Claim A.9. in~\cite{diakonikolas2019efficient}] \label{claim:dec-chisq}
Let $\sigma_y^2 = \|\beta\|^2_2+\sigma^2$.For any $v \in \R^d$, 
we have that 
\[\round{\vvec^\top \x}y= \round{\frac{\vvec^\top \beta+\|\vvec\|_2\sigma_y}{2}} Z_1^2+\round{\frac{\vvec^\top \beta - \|\vvec\|_2\sigma_y}{2}} Z_2^2 \;,\]
where $Z_1,Z_2 \sim \Ncal(0,1)$ and $Z_1, Z_2$ are independent. 
\end{fact}
We say that polynomial $p(v) \succeq q(v)$ if $p(x)-q(x)$ can be written as a sum of squares of polynomials. We write this as $p(v) \succeq_{2m} q(v)$ if we want to
emphasize that the proof only involves polynomials of degree at most $2m$.

\begin{fact}[Basic facts about SOS proofs]\hfill
\begin{itemize}
    \item $\round{\vvec^\top\beta}^2\preceq_{2}\|\vvec\|_2^2\|\beta\|_2^2$ (Cauchy-Schwarz),
    \item $p_1\succeq_{m_1} p_2\succeq_{m_1} 0 \;, \text{ and }\; q_1\succeq_{m_2} q_2 \succeq_{m2} 0 \implies p_1p_2\succeq_{m_1+m_2} q_1q_2$.
\end{itemize}
\label{fact:basic}
\end{fact}

\begin{definition}
A distribution $\Dcal$ over $\Rd{d}$ with mean $\mu\in\Rd{d}$ is $2m$-explicitly bounded for $m\in\Natural$, if $\ \forall\ i\leq 2m$, we have
\[
    \Exp{X\sim\Dcal}{\inner{X-\mu,\vvec}^i} \preceq_i (\sigma i)^{i/2}\enorm{\vvec}^{i}
\]
for variance proxy $\sigma^2\in\R_+$.
\end{definition}

\begin{lemma}\label{lem:sos-proof}
Given that $t\ge 2m$, for $m,t\in\Natural$, there exists a constant $C>0$ such that
\begin{align*}
\E{\inner{\frac{1}{t}\sum_{i=1}^t y_i\x_i-\beta, \vvec}^{2m}}\preceq_{2m} \rho^{2m}\|\vvec\|_2^{2m}(2m)^m\frac{C^m}{t^m}.
\end{align*}
\end{lemma}
\begin{proof}
\[
\begin{array}{l}
\E{\inner{\frac{1}{t}\sum_{i=1}^t y_i\x_i-\beta, \vvec}^{2m}}\\
= \E{\round{\round{\frac{\vvec^\top \beta+\|\vvec\|_2\sigma_y}{2} } (S_1-1)+\round{\frac{\vvec^\top \beta - \|\vvec\|_2\sigma_y}{2}} (S_2-1)}^{2m}}\\
= \sum_{i=0}^{2m}  \binom{2m}{i}\left(\frac{\vvec^\top \beta+\|\vvec\|_2\sigma_y}{2} \right)^{i}\left(\frac{\vvec^\top \beta - \|\vvec\|_2\sigma_y}{2}\right)^{2m-i}M_{i}M_{2m-i}\\
= \sum_{i=0}^{m}  \binom{2m}{i}\left(\left(\frac{\vvec^\top \beta+\|\vvec\|_2\sigma_y}{2} \right)^{2m-2i}+\left(\frac{\vvec^\top \beta-\|\vvec\|_2\sigma_y}{2} \right)^{2m-2i}\right)\left(\frac{(\vvec^\top \beta)^2 - \|\vvec\|_2^2\sigma_y^2}{4}\right)^{i}M_{i}M_{2m-i}
\end{array}
\]
where $S_1, S_2\sim \frac{\chi^2(t)}{t}$, and $M_i = \Exp{Z_j\sim \Ncal(0, 1)}{{\round{\frac{1}{t}\sum_{j=1}^t (Z_j^2-1)}^i}}$.

First, notice that 
\begin{align}
&\left(\frac{\vvec^\top \beta+\|\vvec\|_2\sigma_y}{2} \right)^{2m-2i}+\left(\frac{\vvec^\top \beta-\|\vvec\|_2\sigma_y}{2} \right)^{2m-2i}\nonumber\\
&= 2^{2i-2m}\cdot 2\cdot\sum_{j = 0}^{m-i} \binom{2m-2i}{2j}  \round{\vvec^\top\beta}^{2j}\round{\|\vvec\|_2\sigma_y}^{2m-2i-2j}\nonumber\\
&\preceq_{2m-2i} 2^{2i-2m}2^{2m-2i}\|\vvec\|_2^{2m-2i}\rho^{2m-2i}\nonumber\\
&= \|\vvec\|_2^{2m-2i}\rho^{2m-2i}\label{eqn:sos-1}.
\end{align}
The SOS order hold by the fact that $\round{\vvec^\top\beta}^2 \preceq_{2}\|\beta\|_2^2\|\vvec\|_2^2\preceq_{2}\sigma_y^2\|\vvec\|_2^2$ and Fact~\ref{fact:basic}. 
Secondly, notice that 
\begin{align}
\round{\frac{\round{\vvec^\top \beta}^2 - \|\vvec\|_2^2\sigma_y^2}{4}}^{i} &= 2^{-2i}\sum_{j=0}^i\binom{i}{j}(-1)^{i-j}(\vvec^\top\beta)^{2j}\round{\|\vvec\|_2\sigma_y}^{2i-2j}\nonumber\\
&\preceq_{2i} 2^{-i}\|\vvec\|_2^{2i}\rho^{2i}\nonumber\\
&\preceq_{2i}\|\vvec\|_2^{2i}\rho^{2i}\label{eqn:sos-2}
\end{align}
Combining Equation~\eqref{eqn:sos-1} and Equation~\eqref{eqn:sos-2} yields 
\begin{align*}
\E{\inner{ \frac{1}{t}\sum_{i=1}^t y_i\x_i-\beta, \vvec}^{2m}}\preceq_{2m} \rho^{2m}\|\vvec\|^{2m}_2\sum_{i=0}^m\binom{2m}{i}M_iM_{2m-i}
\end{align*}

\begin{fact}\label{fact:moment-bound}
Given that $t\ge 2m$, it holds that 
\[
M_i\leq 2e^{2i}i^{i/2}/t^{i/2} \;,
\] 
\end{fact}
\begin{proof}
By Bernstein inequality, $\sum_{j=1}^t (Z_j^2-1)$ is a combination of sub-Gaussian with norm $\Theta\round{\sqrt{t}}$ and sub-exponential with norm $\Theta\round{1}$ with disjoint supports. Hence
\begin{align}
    \inv{t^i}\E{\left(\sum_{j=1}^t (Z_j^2-1)\right)^i} &\lesssim \frac{e^i(ti)^{i/2}+e^{2i}i^i}{t^i}\nonumber\\
    &= e^i\frac{i^{i/2}}{t^i}(t^{i/2}+e^ii^{i/2})\;,\nonumber\\
    &= e^i\frac{i^{i/2}}{t^{i/2}}\round{1+e^i\frac{i^{i/2}}{t^{i/2}}}\nonumber\\
    &\leq 2e^{2i}\frac{i^{i/2}}{t^{i/2}}.\qedhere
\end{align}
\end{proof}
Applying Fact~\ref{fact:moment-bound} gives
\begin{align*}
    \E{\inner{\frac{1}{t}\sum_{i=1}^t y_i\x_i-\beta, \vvec}^{2m}}\preceq_{2m} \rho^{2m}\|\vvec\|_2^{2m}(2m)^m\frac{C^m}{t^m}
\end{align*}
for some $0<C<e^6$. This concludes the proof.
\end{proof}
\subsection{The distribution of \texorpdfstring{\(y\x\)}{yx} is \texorpdfstring{\(\Omega(d)\)}{(d)}-Poincar\'e}
We show that even for the simplest parameter setting where the regression vector $\beta=0$, and the noise has variance $1$, the distribution of $y\x$ is Poincar\'e with parameter $\Omega(d)$, and thus applying the result on Poincar\'e distribution from~\cite{2017arXiv171107465K} naively would only yield trivial guarantee. 
\begin{remark}
Suppose $\x\sim \Ncal(0, \I)$, $y \sim \Ncal(0,1)$, and function $f(\z) = \|\z\|_2^2$. Then $\Var{f(\z)} = (2d+6)\E{\|\nabla f(\z)\|_2^2}$.
\label{rem:poincare}
\end{remark}
\begin{proof}
\begin{align*}
\Var{f(\z)} &= \E{y^4\|\x\|_2^4} - \E{y^2\|\x\|_2^2}^2\\
& = 3d(d+2) - d^2\\
&= 2d^2+6d.
\end{align*}
Since $\nabla f(\z) = 2\z$, we have
\begin{align*}
\E{\|\nabla f(\z)\|_2^2} = \E{4y^2\|\x\|_2^2} = 4d.
\end{align*}
Hence
\[
\Var{f(\z)} = \inv{2}(d+3)\E{\|\nabla f(\z)\|_2^2}.\qedhere
\]
\end{proof}
\begin{remark}\label{rem:choice-m}
    The choice of $m$ can be made from $t$ appropriately by considering the analytical inverse map. For any $c>0$, if $y = \frac{t\Delta^2}{2c\rho^2\log (1/p_{\rm min})}$, and $x = \frac{2\log (1/p_{\rm min})}{m}$, then it is clear that the condition $t \geq c\cdot \frac{m}{p_{\rm min}^{2/m}}\cdot\frac{\rho^2}{\Delta^2}$, can be written as $y\geq e^x/x$, or $-1/y\geq (-x)e^{-x}$.
    Therefore, this condition is satisfied when
    \begin{align}
        \frac{2\log(1/p_{\rm min})}{-W_{-1}\round{-\frac{2c\rho^2\log(1/p_{\rm min})}{t\Delta^2}}} &\leq m \leq \frac{2\log(1/p_{\rm min})}{-W_{0}\round{-\frac{2c\rho^2\log(1/p_{\rm min})}{t\Delta^2}}}
    \end{align}
    if $t \geq 2ec\cdot \frac{\rho^2}{\Delta^2}\log\inv{p_{\rm min}}$, where $W_0$ and $W_{-1}$ are the Lambert $W$ functions.
\end{remark}
\section{Proof of Lemma~\ref{lem:refined-est}, Classification and robust estimation}\label{sec:proof-classification}

\begin{algorithm}[ht]
   \caption{Classification and robust estimation}
   \label{alg:classification}
\begin{algorithmic}[1]
   \STATE {\bf Input:} data $\Dcal_{L2} =  \curly{(\x_{i,j},y_{i,j})}_{i\in[n_{L2}],j\in[t_{L2}]}$, $\curly{\Ccal_\ell,\; \tilde\w_\ell,\; \tilde{r}^2_\ell }_{\ell\in[k]}$, $\alpha>0$, $\delta\in\round{0,1}$.
   \STATE {\bf compute} for all $i\in[n_{L2}]$ 
    \[h_i \gets \argmin_{\ell \in\squarebrack{k}} 
    \inv{2\tilde{r}_\ell^2}\sum_{j\in\squarebrack{t_{L2}}} \round{y_{i,j}-\x_{i,j}^\top\tilde\w_\ell}^2 + t_{L2} \log\tilde{r}_\ell\]
        \STATE \hspace{0.3cm}  $\Ccal_{h_i}\gets \Ccal_{h_i}\cup\curly{(\x_{i,j},y_{i,j})}_{j=1}^{t_{L2}}$
    \STATE {\bf compute} for all $\ell\in[k]$, 
        \STATE \hspace{0.3cm} $\hat\w_\ell \gets$ Robust\_Least\_Squares$(\Ccal_\ell)$\hfill[~\cite[Algorithm~2]{diakonikolas2019efficient}]
        \STATE \hspace{0.3cm} $r_{\ell,i}^2\gets t_{L2}^{-1}
        \sum_{j\in[t_{L2}]}
        \round{y_{i,j}-\x_{i,j}^\top\hat\w_\ell}^2$ for all $i\in\Ccal_\ell$
        \STATE \hspace{0.3cm} $\hat{s}_\ell^2 \gets $Univariate\_Mean\_Estimator$\round{\set{ r_{\ell,i}^2}_{i \in\Ccal_\ell},\alpha,\delta}$\hfill[~\cite{lugosi2019robust}]
        \STATE \hspace{0.3cm} $\hat{p}_\ell \gets \abs{\Ccal_\ell}/n_{L2}$
    \STATE {\bfseries Output:} $\curly{\Ccal_\ell,\; \hat\w_\ell,\; \hat{s}^2_\ell,\; \hat{p}_\ell}_{\ell=1}^k$
\end{algorithmic}
\end{algorithm}

\begin{lemma}[Lemma A.15 in~\cite{2020arXiv200208936K}]\label{lem:growing-label}
Given estimated parameters satisfying $\enorm{\tilde{\w}_i-\w_i} \le \Delta/10$,  $(1-\Delta^2/50)\tilde{r}_i^2\le s_i^2+\esqnorm{\tilde{\w}_i-\w_i}\le (1+\Delta^2/50)\tilde{r}_i^2$ for all $i\in\squarebrack{k}$, and a new task with $t_{\out}\ge \Theta\round{\log (k/\delta)/\Delta^4}$ samples whose true regression vector is $\beta = \w_h$, Algorithm~\ref{alg:classification} predicts $h$ correctly with probability $1-\delta$.
\end{lemma}

Since the set $G$ contains $n_{L2}$ i.i.d. samples from our data generation model, by the assumption that $n_{L2}= \tilde\Omega\round{\frac{d}{p_{\rm min}\eps^2t_{L2}}} =\Omega\round{\frac{\log(k/\delta)}{p_{\min}}}$ and from Proposition~\ref{prop:mult-con-infty}, it holds that the number of tasks such that $\beta = \w_i$ is $n_{L2}\hat{p}_i\ge \frac{1}{2}n_{L2}p_i$ with probability at least $1-\delta$. Hence, with this probability, there exists at least $n_{L2}p_i/2$ i.i.d. examples in $G$ for estimating $\w_i$ and $s_i^2$. Lemma~\ref{lem:refined-est} guarantee that our algorithm correctly classified all the tasks in $G$, which implies that there are at least $(p_i/2-\alpha_{L2})n_{L2}$ uncorrupted tasks, and at most $\alpha_{L2}n_{L2}$ corrupted tasks, and hence the corruption level is at most $\frac{\alpha_{L2}n}{(p_i/2-\alpha_{L2})n} \le \frac{4\alpha_{L2}}{p_i}$ since $\alpha_{L2}\leq p_{\rm min}/4$. We can apply robust linear regression algorithm to each cluster separately, and the error of the algorithm is bounded by the following lemma.
\begin{lemma}[Robust Linear Regression, Lemma 1.3 in~\cite{diakonikolas2019efficient}] \label{lem:alg-lr-Gaussian}
Let $S'$ be an $\alpha$-corrupted set of labeled samples of size $\Omega((d/\alpha^2) \poly\log({d}/({\alpha \tau})))$.
There exists an efficient algorithm that on input $S'$ and $\alpha>0$, returns a candidate vector $\widehat{\beta}$
such that with probability at least $1-\tau$ it holds $\enorm{\widehat{\beta} - \beta}  = \BigO{\sigma \alpha \log(1/\alpha)}$.
\end{lemma}

For a single regressor $i\in [k]$, Lemma~\ref{lem:alg-lr-Gaussian} implies that for any $\eps\ge \Omega\round{\frac{\alpha_{L2}}{p_i}\log(p_i/\alpha_{L2})}$, given that ${n_{L2}t_{L2}p_i}\ge \tilde\Omega(d/\eps^2)$, it holds that with probability $1-\delta$, our estimation satisfies
\begin{align*}
\enorm{\hat{\w}_i-\w_i} &\leq \eps s_i. 
\end{align*}

Using Corollary~\ref{cor:chi-var}, we have that our robust variance estimator guarantees that
\begin{align*}
\abs{\hat{s}_i^2-\round{s_i^2+\esqnorm{\hat\w_i-\w_i}}} &\le {\frac{\eps}{\sqrt{t_{L2}}} \round{s_i^2+\esqnorm{\hat\w_i-\w_i}}}\\
&\le {\frac{\eps}{\sqrt{t_{L2}}}s_i^2},
\end{align*}
with probability $1-\delta$.
Taking a union bound over $k$ regressors, we have that for any $\eps\ge \Omega\round{\frac{\alpha}{p_{\rm min}}\log(p_{\rm min}/\alpha)}$, given that 
\[
n_{L2}\ge \tilde\Omega\round{\frac{d}{t_{L2} p_{\rm min} \eps^2}},
\]
for all $i\in [k]$, our estimators satisfy
\begin{align*}
\esqnorm{\hat{\w}_i-\w_i} &\leq \eps s_i,\quad \text{and}\\
\abs{\hat{s}_i^2-s_i^2} &\leq {\frac{\eps}{\sqrt{t_{L2}}}s_i^2}. \text{ (Applying Corollary~\ref{cor:chi-var}) }
\end{align*}
By Proposition~\ref{prop:mult-con-infty}, it holds that \begin{align*}
|\hat{p}_i-p_i| &\le \sqrt{\frac{3\log(k/\delta)}{n_{L2}}p_i}+\alpha_{L2}\\
&\le \min\curly{p_{\min}/10,\eps p_i\sqrt{t_{L2}/d}}+\alpha_{L2}.\qedhere
\end{align*}
The condition on $\eps$ can be converted in to a condition on $\alpha_H$ by the fact that $\frac{\beta}{\log\inv\beta}\geq \frac{e}{e-1}\alpha \implies \alpha\log\inv\alpha\leq \beta$ for $\alpha,\beta\in\round{0,1}$.
This completes the proof.
\section{Proofs of technical lemmas and remarks}

\subsection{Auxiliary Lemmas}
\begin{fact}[$\eps$-tail bound for distributions with bounded second moment]\label{fact:eps-tail}
Suppose random variable $z$ with probability density function $p\round{\cdot}$, satisfies $\E{z^2}\le \sigma^2$, then for any event $\Ecal$ with $\Prob{\Ecal}\ge 1-\eps$, it holds that
\[
\abs{\E{z} - \Prob{\Ecal}\E{z\middle|\Ecal}} \le \sqrt{\eps}\sigma .
\]
\end{fact}
\begin{proof}
Notice that 
\begin{align*}
&\abs{\E{z} - \Prob{\Ecal}\E{z\middle|\Ecal}}\\
=&  \abs{\Prob{\bar{\Ecal}}\E{z\middle|\bar{\Ecal}}}\\
=& \abs{\int_{-\infty}^\infty \indicator{}{z\in \bar{\Ecal}}z p(z)dz}\\
\le& \sqrt{\int_{-\infty}^\infty \indicator{}{z\in \bar{\Ecal}}p(z)dz \cdot \int_{-\infty}^\infty z^2p(z)dz}\quad \text{ (Using Cauchy–Schwarz) }\\
\le& \sqrt{\eps}\sigma.\qedhere
\end{align*}
\end{proof}
\begin{proposition}[Matrix Bernstein inequality, Theorem 1.6.2 in \cite{tropp2015introduction}]\label{prop:matrix-bernstein}
Let $\mathbf{S}_1,\ldots,\mathbf{S}_n$ be independent, centered random matrices with common dimension $d_1\times d_2$, and assume that each one is uniformly bounded $\E{\mathbf{S}_k} = \zero$ and $\enorm{\mathbf{S}_k} \le L\ \forall\ k = 1, \ldots, n$.

Introduce the sum
\[
\Z \coloneqq \sum_{k=1}^n \mathbf{S}_k
\]
and let $v(\Z)$ denote the matrix variance statistic of the sum:
\begin{align*}
v(\Z) \coloneqq \max\curly{\enorm{\E{\Z\Z^\top}}, \enorm{\E{\Z^\top\Z}}}
\end{align*}
Then
\begin{align*}
\Prob{\enorm{\Z} \ge t} \le (d_1 +d_2)\exp\curly{\frac{-t^2/2}{v(\Z)+Lt/3}}
\end{align*}
for all $t \ge 0$.
\end{proposition}
\begin{corollary}[Robust mean estimation of chi-square distribution]\label{cor:chi-var}
Let $G = \{x_i\}_{i=1}^n$ where each $x_i$ is drawn independently from a scaled chi-square distribution $\frac{\sigma^2}{t} \chi^2(t)$ for $t\in\Natural$. Suppose $S = (G\setminus L)\cup 
E$ where $|L| \le \eps n$ and $|E|\le \eps n$. Assuming that $n = \tilde\Omega(\frac{1}{\eps^2})$, the trimmed mean estimator define in~\cite{lugosi2019robust} takes $S$ as input and output an estimate $\hat\sigma^2$, with probability $1-\delta$, that satisfies
\[
\abs{\hat\sigma^2 - \sigma^2} = \BigO{\sigma^2\eps\max\set{\frac{\log(1/\eps)}{t},\sqrt{\frac{\log(1/\eps)}{t}}}}.
\]
\end{corollary}
\begin{proof}
We show the corollary by applying Theorem 1 in~\cite{lugosi2019robust} to the chi-square setting. First we bound the $\Ecal(4\eps, X)$ term in~\cite[Theorem~1]{lugosi2019robust}. 
For a zero mean random variable $X$, define quantile 
\[
Q_p(X) = \sup\{ M \in \Rd{} : \Prob{X \ge M}\ge 1-p\},
\]
and $\Ecal(\eps, X)$ is defined as 
\[
\Ecal(\eps, X) \coloneqq \max\set{{\E{|X|\indicator{}{X\le Q_{\eps/2}}}, \E{|X|\indicator{}{X\ge Q_{1-\eps/2}}}}}.
\]
where we denote $Q_p(X)$ by $Q_p$ simply. Let $X = \frac{x_i}{\sigma^2}-1$. Note that under chi-square distribution for $x_i$, for small $\eps$, we can assume $Q_{\eps/2}\le 0$ and $Q_{1-\eps/2}\ge 0$. By Bernstein's inequality, we have that for any $u\ge 0$,
\[
\Prob{\abs{\frac{x_i}{\sigma^2}-1}\ge u}\le 2\exp\round{-ct\min\set{u^2,u}}.
\]
If $\log(2/\eps)\le ct$, let $u_\eps =  \sqrt{\log(2/\eps)/ct}$ such that $2\exp\round{-ct\min\set{u_\eps^2,u_\eps}} = \eps$. We have
\begin{align*}
\E{|X|\indicator{}{X\ge Q_{1-\eps/2}}} &= \int_{Q_{1-\eps}}^{\infty} \Prob{X\ge u} du\\
&= \int_{Q_{1-\eps}}^{u_\eps} \Prob{X\ge u} du+ \int_{u_\eps}^{\infty} \Prob{X\ge z} dz\\
&\le \eps\cdot u_\eps + \frac{\eps}{ct}\sqrt{\log\frac{2}{\eps}}\\
&= \BigO{\frac{\eps\sqrt{\log(1/\eps)}}{\sqrt{t}}}.
\end{align*}
Otherwise, let $u_\eps =  \log(2/\eps)/ct$ such that $2\exp(-ct\min(u_\eps^2,u_\eps)) = \eps$. Then,
\begin{align*}
\E{|X|\indicator{}{X\ge Q_{1-\eps}}} &= \int_{Q_{1-\eps}}^{\infty} \Prob{X\ge u} du\\
&= \int_{Q_{1-\eps}}^{u_\eps} \Prob{X\ge u} du + \int_{u_\eps}^{\infty} \Prob{X\ge u} du\\
&\le \eps\cdot u_\eps + \frac{2}{ct}\exp(-ct u_\eps)\\
&=\BigO{\frac{\eps\log(1/\eps)}{t}}.
\end{align*}
Combing the two term we get 
\[
Q_{1-\eps/2}(X)=\BigO{\eps\max\set{\frac{\log(1/\eps)}{t},\sqrt{\frac{\log(1/\eps)}{t}}}},
\]
which implies
\[
\Ecal(4\eps, X) = Q_{1-\eps/2}(X)=\BigO{\eps\max\set{\frac{\log(1/\eps)}{t},\sqrt{\frac{\log(1/\eps)}{t}}}}.
\]
The variance of $X$ is bounded as
\[
\sigma_X^2 = \Var{\frac{x_i}{\sigma^2}-1} = \BigO{\frac{1}{t}}.
\]
Hence, with the assumption that $n = \tilde\Omega(1/\eps^2)$, Theorem 1 in~\cite{lugosi2019robust} guarantee to estimate $\sigma^2$ with error 
\[
\frac{\abs{\hat{\sigma}^2-\sigma^2}}{\sigma^2}  = \BigO{\eps\max\set{\frac{\log(1/\eps)}{t},\sqrt{\frac{\log(1/\eps)}{t}}}}.\qedhere
\]
\end{proof}
\begin{proposition}[Trimmed mean estimator for distributions with bounded variances (see, e.g. Proposition 2.2 in~\cite{steinhardtlecture})]\label{prop:diakanikolas-high-prob}
Suppose a multi-set $S = \set{x_i}_{i=1}^{n}$, $0 <\eps\le 1/8$, satisfies $S = (G\setminus L)\cup E$, $|E|\le \eps|G|$, $|L|\le \eps|G|$, and set $G$ satisfies
\begin{align*}
\abs{\frac{1}{|G|}\sum_{x_i\in G} x_i - \mu} &\le \sqrt{\eps}\\
\frac{1}{|G|}\sum_{x_i\in G} (x_i-\mu)^2 &\le 1
\end{align*}
Let $R$ be the set containing the lower and upper $2\eps$ quantiles from $S$, then set $S' = S\setminus R = (G\setminus L')\cup E'$ satisfies
\[
\abs{\frac{1}{|S'|}\sum_{x_i\in S'}x_i -\mu}\le 18\sqrt{\eps},
\]
\end{proposition}
\begin{proof}
The result is well-known. We provide a proof here for completeness. First note that all the datapoints in $E$ that exceed the $\eps$-quantile of $G$ must lie in $R$. By Chebyshev's inequality, $\frac{1}{|G|}\sum_{x_i\in G }\indicator{}{|x_i-\mu|\ge \sqrt{\frac{1}{\eps}}+\sqrt{\eps}}\le \eps$.
Therefore $|x_i-\mu|\le \sqrt{\frac{1}{\eps}}+\sqrt{\eps}$ for all $x_i\in E'$. 
Second, since $|L'|\le 5\eps$, by Fact~\ref{fact:eps-tail}, the mean of $G\setminus L'$ lies within $\sqrt{\frac{10\eps}{1-5\eps}}$ of $\mu$.
Finally, the difference between $\mu$ and the mean of $(G\setminus L')\cup E'$ is bounded by
\begin{align}
    \abs{\frac{1}{|S'|}\sum_{x_i\in S'}x_i -\mu} &\leq \frac{1}{|S'|}\round{|E'|(\sqrt{\frac{1}{\eps}}+\sqrt{\eps})+|G\setminus L'|\sqrt{\frac{10\eps}{1-5\eps}}}\nonumber\\
    &\le 18\sqrt{\eps} \; \text{ (Assuming that $\eps\le 1/8$) }.\qedhere
\end{align}
\end{proof}

\begin{proposition}[$\ell_1$ deviation bound of multinomial distributions~\cite{weissman2003inequalities}]\label{prop:mult-con-l1}
Let $\p = \{p_1,\ldots,p_k\}$ be a vector of probabilities (i.e. $p_i\ge 0$ for all $i\in [k]$ and $\sum_{i=1}^kp_i = 1$). Let $\x \sim {\rm multinomial} (n, \p)$ follow a multinomial distribution with $n$ trials and probability $\p$. Then given $n \ge 2k\log(2/\delta)/\alpha^2$ with probability $1-\delta$, 
\[
\norm{1}{\frac{1}{n}\x-\p} \le \alpha,
\]
\end{proposition}

 \begin{proposition}[$\ell_\infty$ deviation bound of multinomial distributions, Proposition D.7 in ~\cite{2020arXiv200208936K}]\label{prop:mult-con-infty}
Let $\p = \{p_1,\ldots,p_k\}$ be a vector of probabilities (i.e. $p_i\ge 0$ for all $i\in [k]$ and $\sum_{i=1}^kp_i = 1$). Let $\x \sim {\rm multinomial} (n, \p)$ follow a multinomial distribution with $n$ trials and probability $\p$. Then with probability $1-\delta$, for all $i\in [k]$,
\[
\abs{\frac{1}{n}x_i-p_i} \le \sqrt{\frac{3\log(k/\delta)}{n}p_i},
\]
which implies
\[
\norm{\infty}{\frac{1}{n}\x-\p} \le \sqrt{\frac{3\log(k/\delta)}{n}}.
\]
for all $i\in\squarebrack{k}$.
\end{proposition}
\begin{fact}[Gaussian 4-th moment conditions]\label{fact:gauss-4}
Let $\vvec$ and $\uvec$ denote two fixed vectors, we have
\begin{enumerate}
    \item $\begin{aligned}
    \Exp{ \x \sim {\cal N}(\zero,\I) }{ \round{ \vvec^\top \x }^2\round{ \uvec^\top \x }^2 } = \| \uvec \|_2^2 \cdot \| \vvec \|_2^2 + 2 \langle \uvec , \vvec \rangle^2
    \end{aligned}$
    \item $\begin{aligned}
    \Exp{ \x \sim {\cal N}(\zero,\I) }{ \round{ \vvec^\top \x }^3\round{ \uvec^\top \x } } = 3\| \vvec \|_2^2\cdot\langle \vvec , \uvec \rangle
    \end{aligned}$
\end{enumerate}
\begin{align*}
\end{align*}
\end{fact}

\section{Outlier robust principal component analysis}
\label{sec:ORPCA}

We provide comparisons of Algorithm~\ref{alg:filtering} to state-of-the-art baselines, in both theory and numerical experiments. 

\subsection{Theoretical comparisons}

{\bf Comparisons with  \cite{xu2012outlier}.}
 Outlier-Robust Principal Component Analysis (ORPCA) \cite{xu2012outlier, feng2012robust,yang2015unified} studies a similar problem  
under a Gaussian model. 
For comparison, we can modify the best known ORPCA estimator from \cite{xu2012outlier} to our setting in  Proposition~\ref{prop:rpca}, to get a semi-orthogonal $\hat\U$ achieving
\begin{align}
\norm{*}{  \SIGMA -  \hat\U\hat\U^\top \SIGMA  \hat\U \hat\U^\top  } = \norm{*}{  \SIGMA -  {\cal P}_k \round{\SIGMA}  }+\Ocal \big(\,\alpha^{1/2} \norm{*}{ {\cal P}_k \round{\SIGMA} }+ \nu {k}\alpha^{1/4} \,\big)\;. 
\label{eq:rpca_subopt_appendix}
\end{align}
We significantly improve in the dominant third term (see Eq.~\eqref{eq:rpca_opt}).


{\bf Comparisons with filters based on the second moment of $z_i$'s.} 
Popular recent results on robust  estimation are based on filters that rely on the second moment~\cite{diakonikolas2017beingrobust, steinhardt2018resilience}. One might wonder if it is possible to apply these filters to $z_i$'s to remove the corrupted samples. Such approaches fail when $n=\Ocal(d)$, even when $\x_i$ are standard Gaussian; this is immediate from the fact that the empirical second moment of $z_i$'s does not concentrate until we have $n=\Ocal(d^2)$ uncorrupted samples.

\subsubsection{Proof of robust mean estimation for the second moment  in Remark~\ref{lem:robust_M}} 
\label{sec:robust_M_proof}

From the definition of $\hat\M$,
\begin{align}
    \hat\M &= \round{2n}^{-1}\sum_{i=1}^{n} \round{\hat\beta_i^{(1)}\hat\beta_i^{(2)\top} + \hat\beta_i^{(2)}\hat\beta_i^{(1)\top}},
\end{align}
which is the empirical mean of the matrices $\hat\beta_i^{(1)}\hat\beta_i^{(2)\top} + \hat\beta_i^{(2)}\hat\beta_i^{(1)\top}$.
Let us consider the $d^2$-length vectors $\hat\m_i$, and $\m_i$ constructed by unrolling the matrix $\hat\beta_i^{(1)}\hat\beta_i^{(2)\top}$, and $\beta_i\beta_i^\top$ respectively. Then $(2n)^{-1}\B\sumi{1}{n}\hat\m_i$ will be the unrolled vector corresponding to the matrix $\hat\M$ where $\B \coloneqq \I_d+\Pmat_d$, and $\Pmat_d\in\Rd{d^2\times d^2}$ is the permutation matrix corresponding to the transposition of $d\times d$ matrices. The covariance matrix of the samples $\B\hat\m$ is therefore
\begin{align}
    \B\E{\hat\m\hat\m^\top}\B - \B\m\m^\top\B,
\end{align}
and its norm is bounded by
\begin{align}
    \enorm{\E{\hat\m\hat\m^\top} - \m\m^\top}.
\end{align}
Returning back to the matrix notation, we therefore essentially need to bound the operator norm of the covariance tensor of the samples $\hat\beta_i^{(1)}\hat\beta_i^{(2)\top}$, where $\beta_i = \w_j$ with probability $p_j$. This can be bounded as
\begin{align}
    &\sup_{\normF{\A}=1}\Exp{i\sim\p,\x,y}{\Tr{\A\round{\hat\beta_i^{(1)}\hat\beta_i^{(2)\top}-\M}}^2}\nonumber\\
    =& \sup_{\normF{\A}=1}\Exp{i\sim\p,\x,y}{\round{\hat\beta_i^{(2)\top}\A\hat\beta_i^{(1)}-\Tr{\A\M}}^2}\nonumber\\
    =& \sup_{\normF{\A}=1}\Exp{i\sim\p,\x,y}{\round{\hat\beta_i^{(2)\top}\A\hat\beta_i^{(1)}}^2-\round{\Tr{\A\M}}^2}\nonumber\\
    =& \sup_{\normF{\A}=1}\Exp{i\sim\p,\x,y}{\round{\hat\beta_i^{(2)\top}\A\hat\beta_i^{(1)}}^2-\round{\Tr{\A\beta_i\beta_i^\top}}^2+ \round{\Tr{\A\beta_i\beta_i^\top}}^2 - \round{\Tr{\A\M}}^2}\nonumber\\
    \leq& \sup_{\normF{\A}=1}\Exp{i\sim\p,\x,y}{\round{\hat\beta_i^{(2)\top}\A\hat\beta_i^{(1)}}^2-\round{\beta_i^\top\A\beta_i}^2} + \sup_{\normF{\A}=1}\Exp{i\sim\p}{\round{\Tr{\A\beta_i\beta_i^\top}}^2 - \round{\Tr{\A\M}}^2}\nonumber\\
    \leq& \Exp{i\sim\p}{\sup_{\normF{\A}=1}\Exp{\x,y}{\round{\hat\beta_i^{(2)\top}\A\hat\beta_i^{(1)}}^2-\round{\beta_i^\top\A\beta_i}^2\Big|\ i\ }}\nonumber\\
    &\qquad\qquad\qquad\qquad+ \sup_{\normF{\A}=1}\Exp{i\sim\p}{\round{\Tr{\A\beta_i\beta_i^\top}}^2 - \round{\Tr{\A\M}}^2}\label{eq:6}
\end{align}
Considering the first term of Equation~\eqref{eq:6}, where for a fixed $i$, we compute the inner expectation
\begin{align}
    &\Exp{\x,y}{\round{\hat\beta_i^{(2)\top}\A\hat\beta_i^{(1)}}^2-\round{\beta_i^\top\A\beta_i}^2}\nonumber\\
    =& \Exp{\x,y}{\hat\beta_i^{(2)\top}\A\hat\beta_i^{(1)}\hat\beta_i^{(1)\top}\A^\top\hat\beta_i^{(2)}-\beta_i^\top\A\beta_i\beta_i^\top\A^\top\beta_i}\nonumber\\
    =& \Exp{\hat\beta^{(1)}}{ \Exp{\x,y}{\hat\beta_i^{(2)\top}\A\hat\beta_i^{(1)}\hat\beta_i^{(1)\top}\A^\top\hat\beta_i^{(2)}-\beta_i^\top\A\beta_i\beta_i^\top\A^\top\beta_i\ \Big|\ \hat\beta_i^{(1)}\ }}
\end{align}
Define the PSD matrix $\B\coloneqq \A\hat\beta_i^{(1)}\hat\beta_i^{(1)\top}\A^\top$, and note that its expectation is
\begin{align}
    \E{\A\hat\beta_i^{(1)}\hat\beta_i^{(1)\top}\A^\top} &\leq \A\round{\round{1+\inv{t}}\beta_i\beta_i^\top + \frac{\rho^2}{t}\I}\A^\top\nonumber\\
    &= \round{1+\inv{t}}\A\beta_i\beta_i^\top\A^\top + \frac{\rho^2}{t}\A \A^\top\label{eq:Exp_B}
\end{align}
from which we get
\begin{align}
    &\Exp{\x,y}{\round{\hat\beta_i^{(2)\top}\A\hat\beta_i^{(1)}}^2-\round{\beta_i^\top\A\beta_i}^2}\nonumber\\
    =& \Exp{\hat\beta_i^{(1)}}{ \Exp{\x,y}{\hat\beta_i^{(2)\top}\B\hat\beta_i^{(2)}-\beta_i^\top\A\beta_i\beta_i^\top\A^\top\beta_i\ \Big|\ \hat\beta_i^{(1)}\ }}\nonumber\\
    =& \Exp{\hat\beta_i^{(1)}}{ \Tr{\B\Exp{\x,y}{\hat\beta_i^{(2)}\hat\beta_i^{(2)\top}}}} -\beta_i^\top\A\beta_i\beta_i^\top\A^\top\beta_i\nonumber\\
    \leq& \Exp{\hat\beta_i^{(1)}}{ \Tr{\B\round{1+\inv{t}}\beta_i\beta_i^\top + \frac{\rho^2}{t}\B }} -\beta_i^\top\A\beta_i\beta_i^\top\A^\top\beta_i\nonumber\\
    =& \Exp{\hat\beta_i^{(1)}}{ \round{1+\inv{t}}\beta_i^\top\B\beta_i + \frac{\rho^2}{t}\Tr{\B}} -\beta_i^\top\A\beta_i\beta_i^\top\A^\top\beta_i\nonumber\\
    =& \round{1+\inv{t}}\beta_i^\top\round{\round{1+\inv{t}}\A\beta\beta^\top\A^\top + \frac{\rho^2}{t}\A \A^\top}\beta_i\nonumber\\
    &\qquad+ \frac{\rho^2}{t}\Tr{\round{1+\inv{t}}\A\beta_i\beta_i^\top\A^\top + \frac{\rho^2}{t}\A \A^\top} -\beta_i^\top\A\beta_i\beta_i^\top\A^\top\beta_i\nonumber\\
    =& \squarebrack{\round{1+\inv{t}}^2-1}\round{\beta_i^\top\A\beta_i}^2 + \frac{\rho^2}{t}\round{1+\inv{t}}\beta_i^\top\round{\A\A^\top + \A^\top\A}\beta_i + \frac{\rho^4}{t^2}\normF{\A}^2\nonumber\\
    =& \round{\frac{2}{t}+\inv{t^2}}\round{\Tr{\A\beta_i\beta_i^\top}}^2 + \frac{\rho^2}{t}\round{1+\inv{t}}\round{\Tr{\beta_i^\top\A\A^\top\beta_i} + \Tr{\beta_i^\top\A^\top\A\beta_i}} + \frac{\rho^4}{t^2}\normF{\A}^2\nonumber\\
    \leq& \round{\frac{2}{t}+\inv{t^2}}d\normF{\A}^2\enorm{\beta_i}^4 + \frac{2\rho^2}{t}\round{1+\inv{t}}\normF{\A}^2\esqnorm{\beta_i} + \frac{\rho^4}{t^2}\normF{\A}^2\nonumber\\
    \leq& \frac{3\rho^4}{t} + \frac{4\rho^4}{t} + \frac{\rho^4}{t^2}\nonumber\\
    =& \frac{8\rho^4}{t}\label{eq:9}.
\end{align}
Plugging back Equation~\eqref{eq:9} in Equation~\eqref{eq:6} we get
\begin{align}
    \sup_{\normF{\A}=1}\Exp{i\sim\p,\x,y}{\Tr{A\round{\hat\beta_i^{(1)}\hat\beta_i^{(2)\top}-\M}}^2} &\lesssim \frac{\rho^4}{t} + \sup_{\normF{\A}=1}\Exp{i\sim\p}{\round{\Tr{\A\beta_i\beta_i^\top}}^2 - \round{\Tr{\A\M}}^2}\nonumber\\
    &\lesssim \frac{\rho^4}{t} + \rho^4\nonumber\\
    &\lesssim \rho^4\nonumber\\
    \implies \enorm{\B\E{\hat\m\hat\m^\top}\B - \B\m\m^\top\B} &\lesssim \rho^4.
\end{align}
Using~\cite[Theorem~1.3.]{cheng2019high}, we finally get that the robust mean estimate $\bar\M$ of $\M$ can be computed using $n=\Omega\round{\frac{d^2}{\alpha}\log\frac{d}{\delta}}$ independent tasks in time $\BigO{nd^2/\mathrm{\poly}\round{\alpha}}$, and satisfies
\begin{align}
    \enorm{\bar\M-\M} &\leq \normF{\bar\M-\M} \leq \BigO{\rho^2\sqrt{\alpha}}
\end{align}
with probability at-least $1-\delta$ for $\delta\in\round{0,1}$ using~\cite[Algorithm~2]{cheng2019high}.

\begin{lemma}\label{lem:bound_M}
Let $\beta\in\Rd{d}$ be the true model for a task which gets to observe $t$ samples $\X\in\Rd{t\times d}$, and provides labels $\y\sim\Ncal\round{\X\beta,\sigma^2\I_t}$ where $y_{j}\in\R$ is the label for $\x_{j}\sim\Ncal\round{\zero,\I_d}$. 
The estimator for $\beta$ would be $\hat\beta \coloneqq \inv{t}\X^\top\y$, and will satisfy
\begin{align}
    \E{\hat\beta\hat\beta^\top} - \beta\beta^\top&= \inv{t}\beta\beta^\top + \frac{\esqnorm{\beta}+\sigma^2}{t}\I\;.
\end{align}
\end{lemma}
\begin{proof}
\begin{align}
    \hat\beta &= \inv{t}\X^\top\y\nonumber\\
    &= \inv{t}\X^\top\X\beta + \inv{t}\X^\top\epsilon \qquad(\text{where }\epsilon\sim\Ncal\round{\zero,\sigma^2\I_t})\nonumber\\
    \implies \hat\beta - \beta &= \round{\inv{t}\X^\top\X - \I_d}\beta + \inv{t}\X^\top\epsilon\nonumber
\end{align}
Let $\z \coloneqq \hat\beta-\beta$, then
\begin{align}
    \E{\z} &= \zero, \quad \text{and}\nonumber\\
    \E{\z\z^\top} &= \E{\round{\round{\inv{t}\X^\top\X - \I_d}\beta + \inv{t}\X^\top\epsilon}\round{\round{\inv{t}\X^\top\X - \I_d}\beta + \inv{t}\X^\top\epsilon}^\top}\nonumber\\
    &= \frac{\sigma^2}{t}\I_d + \inv{t^2}\E{\X^\top\X\beta\beta^\top\X^\top\X} - \beta\beta^\top\nonumber\\
    &= \inv{t}\round{\esqnorm{\beta}+\sigma^2}\I_d + \inv{t}\beta\beta^\top\qquad(\text{Using Fact~\ref{fact:gauss-4}})
\end{align}
completing the proof.
\end{proof}

\subsection{Experimental comparisons}

We demonstrate the comparison between Algorithm~\ref{alg:filtering} and HRPCA by considering a distribution class. Consider the distribution of the uncorrupted samples $\x\in\Rd{d}$ to be as follows:
\begin{itemize}
    \item $x_1 \sim \Ncal(0,1.1)$,
    \item $x_2 = z\cdot x_1/\sqrt{1.1}$, where $z$ is an independent Rademacher random variable,
    \item $\x_{3:d} \sim \Ncal(\zero_{d-2},\I_{d-2})$,
\end{itemize}
and $\alpha$ is the corruption level. We sample $n\geq 5d/\alpha$ points from this distribution. Note that $\SIGMA \coloneqq \E{\x\x^\top} = \I_d + 0.1\e_1\e_2^\top$. The adversary then corrupts a point $\x$ with probability $\alpha$ as
\[
    \x' \gets \begin{bmatrix}0 & z'\cdot 2\alpha^{1/4} & \x_{3:d}\end{bmatrix},
\]
where $z'$ is an independent Rademacher random variable.

We run Algorithm~\ref{alg:filtering} and HRPCA, on the above described setup by choosing $d = 10$, $\alpha\in\set{0.005,0.01,\ldots,0.025}$, $k = 1$, and $n = 10^4$. To evaluate the performance of both the algorithms, we compute the variance captured: $\Tr{\hat\U^\top\SIGMA\hat\U}$, where $\hat\U\in\Rd{d\times k}$ is the output of either algorithms. We also compute the best oracle solution which is $\Tr{\U_n^\top\SIGMA\U_n}$ where $\U_n\in\Rd{d\times k}$ is the top $k$ singular vector matrix of $\hat\SIGMA_n \coloneqq \inv{n}\sum_{i=1}^{n}\x_i\x_i^\top$ where $\set{\x_i}_{i=1}^n$ is the original uncorrupted sample set. Notice that this estimator is the optimal subspace estimator in the absence of extra structural assumptions about the subspaces. 

We demonstrate the variance captured, the number of corrupted points left, and the number of uncorrupted points removed by the algorithms in Figure~\ref{fig:exp-error-comparison}. A random guess of the subspace will capture roughly variance $1$. The oracle estimator has variance $\approx 1.0886$. The best rank-$1$ subspace is spanned by $\e_1$ whose captured variance is $1.1$. We show the average performance of Algorithm 2 over $100$ independent trials. The HRPCA algorithm is very slow since one need to pick the best subspace from the $n/2$ iterations while each iteration requires eigen-decomposition once. Thus we only take the average over $10$ trials, which is enough to see the trend of its performance.


\end{document}